\documentclass{article}

\usepackage{microtype}
\usepackage{graphicx}
\usepackage{subcaption}
\usepackage{multirow}
\usepackage{booktabs} % for professional tables

\usepackage{hyperref}

\usepackage[accepted]{icml2026}

\usepackage{amsmath}
\usepackage{amssymb}
\usepackage{mathtools}
\usepackage{amsthm}

\usepackage{diagbox}

\usepackage{titletoc}

\usepackage[nameinlink, noabbrev, capitalize]{cleveref}

\Crefname{equation}{Eq.}{Eq.}
\Crefname{proposition}{Proposition}{Proposition}
\Crefname{assumption}{Assumption}{Assumption}
\Crefname{figure}{Fig.}{Fig.}
\Crefname{definition}{Definition}{Definitions}
\Crefname{section}{Sec.}{Sec.}
\Crefname{table}{Tab.}{Tab}
\Crefname{tabular}{Tab.}{Tab}
\Crefname{appendix}{App.}{App.}

\newcommand{\oursGraph}{\textsc{R-UniGraph}}
\newcommand{\ours}{\textsc{R-UniMesh}}
\newcommand{\pv}{\texttt{PyVista}}

\usepackage[capitalize,noabbrev]{cleveref}

\DeclareMathOperator{\Tr}{Tr}

\DeclareMathOperator{\err}{err}

\newtheorem{theorem}{Theorem}
\newtheorem{proposition}{Proposition}

\newtheorem{corollary}{Corollary}

\theoremstyle{definition}
\newtheorem{definition}{Definition}

\newtheorem{assumption}{Assumption}
\newtheorem{remark}[theorem]{Remark}
\newtheorem{example}{Example}

\newtheorem*{theorem*}{Theorem 1}

\newtheorem*{corollary*}{Corollary 1}

\definecolor{darkblue}{rgb}{0.0,0.0,0.65}
\definecolor{darkred}{rgb}{0.68,0.05,0.0}
\definecolor{darkgreen}{rgb}{0.0,0.29,0.29}
\definecolor{darkpurple}{rgb}{0.47,0.09,0.29}
\hypersetup{
   colorlinks = true,
   citecolor  = darkred,
   linkcolor  = darkred,
   filecolor  = darkred,
   urlcolor   = darkred,
 }

\newcommand{\errreg}[1]{\text{err}_{\mathrm{reg}}(#1)}

\newcommand{\PfX}{P_{f(\mathbf{X})}}
\newcommand{\PX}{P_{\mathbf{X}}}

\newcommand{\kl}[2]{D_{\text{KL}}\left(#1 \parallel #2\right)}

\usepackage[textsize=tiny]{todonotes}

\DeclareMathOperator{\ACC}{ACC}

\DeclareMathOperator{\RMSE}{RMSE}

\crefformat{equation}{#2Eq.~#1#3}
\Crefformat{equation}{#2Eq.~#1#3}

\icmltitlerunning{Smoothness Errors in Dynamics Models and How to Avoid Them}

\begin{document}

\twocolumn[
  \icmltitle{Smoothness Errors in Dynamics Models and How to Avoid Them}



  \icmlsetsymbol{equal}{*}

  \begin{icmlauthorlist}
    \icmlauthor{\href{https://ebrmn.space/}{Edward Berman}}{equal,neu}
    \icmlauthor{\href{https://www.luisali.com/}{Luisa Li}}{equal,neu}
    \icmlauthor{\href{https://jypark0.github.io/}{Jung Yeon Park}}{neu}
    \icmlauthor{\href{https://www.robinwalters.com/}{Robin Walters}}{neu}
  \end{icmlauthorlist}

  \icmlaffiliation{neu}{\href{https://www.robinwalters.com/}{Geometric Learning Lab}, Northeastern University}
  %\icmlaffiliation{comp}{Company Name, Location, Country}

  \icmlcorrespondingauthor{Edward Berman}{eddieberman@g.harvard.edu}

  \icmlkeywords{PDEs, Dynamics Modeling, Oversmoothing}

  \vskip 0.3in
]



%\printAffiliationsAndNotice{}  % no special notice (required even if empty)
 \printAffiliationsAndNotice{\icmlEqualContribution}

\begin{abstract}
Modern neural networks have shown promise for solving partial differential equations over surfaces, often by discretizing the surface as a mesh and learning with a mesh-aware graph neural network. However, graph neural networks suffer from oversmoothing, where a node's features become increasingly similar to those of its neighbors. Unitary graph convolutions, which are mathematically constrained to preserve smoothness, have been proposed to address this issue. Despite this, in many physical systems, such as diffusion processes, smoothness naturally increases and unitarity may be overconstraining. In this paper, we systematically study the smoothing effects of different GNNs for dynamics modeling and prove that unitary convolutions hurt performance for such tasks. We propose relaxed unitary convolutions that balance smoothness preservation with the natural smoothing required for physical systems. We also generalize unitary and relaxed unitary convolutions from graphs to meshes. In experiments on PDEs such as the heat and wave equations over complex meshes and on weather forecasting, we find that our method outperforms several strong baselines, including mesh-aware transformers and equivariant neural networks. Our project page is available at \href{https://www.luisali.com/research/smoothness}{https://www.luisali.com/research/smoothness}.
\end{abstract}

\section{Introduction}

Solving partial differential equations (PDEs) is crucial across many scientific and engineering domains, including acoustics, fluid dynamics, and electrodynamics. Recently, neural networks have been explored as alternatives to analytic and traditional numerical methods for solving PDEs. Neural networks offer faster inference \citep{cui2024transformer}, discretization free solutions \citep{li2021fourier}, better robustness to partial observability \citep{schlaginhaufen2021learning, huang2024diffusionpde, morelpredicting}, and synergy with existing finite element methods \citep{gupta2023generalized}.

\begin{figure}[t]
    \centering
    \includegraphics[width=1\linewidth]{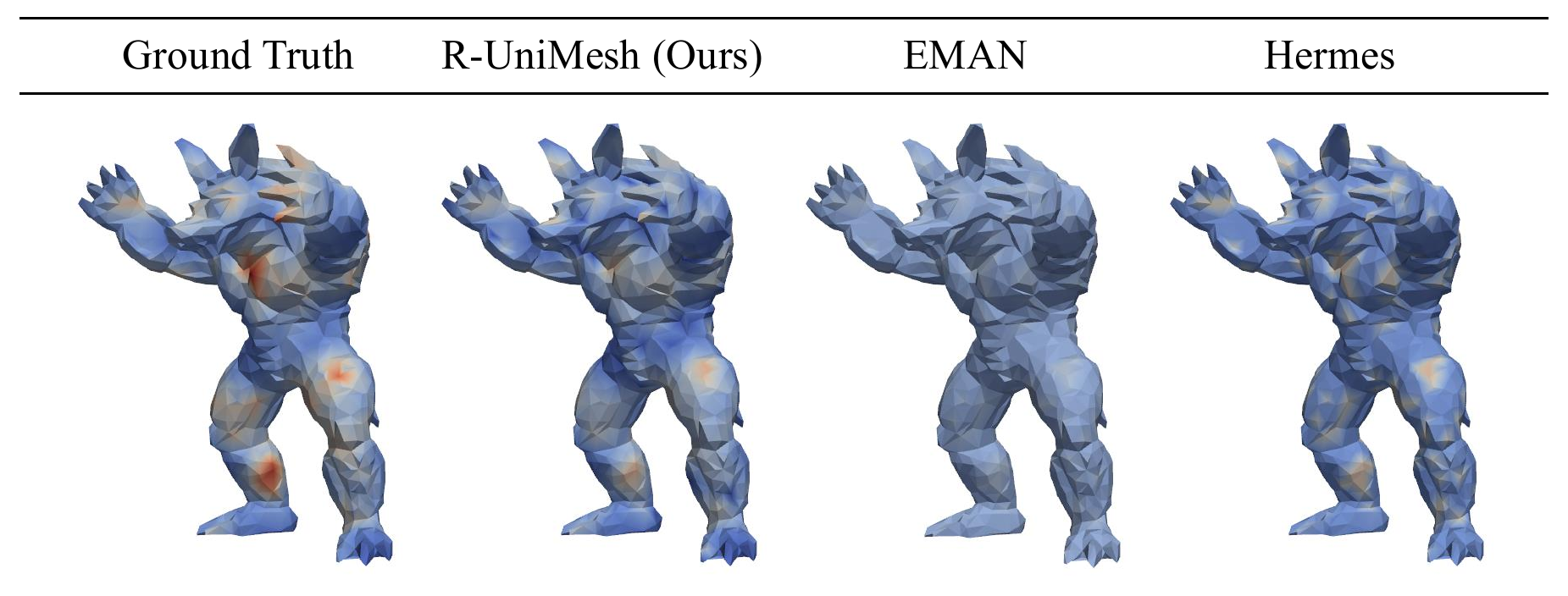}
    \caption{Qualitative comparison of autoregressive model predictions for the heat equation on the armadillo mesh at timestep $T=190$. Our \ours{} model remains faithful to the ground truth during each step of the rollout, whereas the EMAN model over smooths and the Hermes model under smooths. A more complete comparison over several timesteps is in \cref{sec:diagnostics}, \cref{tab:diagnostic}.} 
    \label{fig:salesman}
\end{figure}

However, neural network models often have architectural biases that hurt their ability to model certain dynamics.  In particular, many deep learning methods solve PDEs by discretizing the domain into a grid or mesh and modeling the solution using a graph neural network \citep[e.g.,][]{janny2023eagle, park2023modelingdynamicsmeshesgauge}. Unfortunately, graph neural networks (GNNs) tend to oversmooth \cite{li2018deeper}, where adjacent node features become increasingly similar over successive iterations of message passing. The phenomenon of oversmoothing occurs in a variety of settings \cite{cai2020note, bodnar2022neural, keriven2022not, rusch2023survey, kiani2024unitary, arroyo2025vanishing, suinterplay} and hampers the performance of deep GNNs. 

To address oversmoothing, \citet{kiani2024unitary} propose using unitary graph convolutions, which constrain weight matrices to be unitary. This ensures that the linear transformations preserve norms and remain invertible, improving network stability. They also prove that unitary convolutions prevent oversmoothing by preserving the Rayleigh quotient, a measure of smoothness for signals defined on graphs. However, this poses a new problem: many dynamics problems commonly solved using GNNs require \textit{some} amount of smoothing. For example, heat diffusion on graphs and meshes naturally smooths the input node features. Using unitary graph convolutions in such problems would result in undersmoothing and does not give a complete solution to smoothness errors. 

In this work, we first theoretically characterize the limitations of unitary convolutions for dynamics problems. In particular, we derive a lower bound on the approximation error of unitary functions and show that unitary functions are overconstrained for dynamical systems where the solution's norm has high angular dependence. To address this issue, we propose \emph{relaxed unitary convolutions}, which balance smoothness preservation with modeling fidelity, outperforming existing methods on dynamic systems that require natural smoothing. We also generalize both the Rayleigh quotient and unitary convolution framework from graphs to meshes so that relaxations can be applied in this setting. Finally, we systematically investigate smoothness tendencies of different mesh-GNN architectures and find that our mechanism for approximately preserving smoothness is key to successful modeling, providing equal or greater improvement to other inductive biases such as equivariance.

In summary, our contributions are the following:

\begin{enumerate}
  \item Derive a lower bound on the approximation error of unitary functions, demonstrating that they are overly restrictive when predicting dynamics with high angular dependence in the solution's norm (\cref{sec:theory}).
  \item  Introduce relaxed unitary convolutions that balance accuracy with smoothness preservation, and extend both the Rayleigh quotient and unitary convolution framework to meshes (\cref{sec:relax_uni}).
  \item Empirically analyze the smoothness behavior of various GNN architectures on complex dynamical systems, showing that controlling smoothness can match or outperform strong baselines (\cref{sec:experiments}).
\end{enumerate}

\section{Related Works}

\paragraph{Oversmoothing and undersmoothing in GNNs.} Like \citet{kiani2024unitary}, our work quantifies the effect of neural networks on the Rayleigh quotient \cite{chung1997spectral} of a graph. \citet{kiani2024unitary} prove that unitary functions, and in particular the unitary convolution network, strictly preserve the Rayleigh quotient and therefore the smoothness of input graphs. However, we show both theoretically and empirically how this property can be overconstraining in GNNs. Other approaches to quantifying smoothness in PDE solutions have used the Mat\'ern kernel \cite{borovitskiy2021matern, daniels2025uncertainty} or decay rate exponents \cite{kulick2025investigating}, but none have considered the Rayleigh quotient for dynamics models as we do.

Our work is perhaps most similar to \citet{keriven2022not}, who also point out that some smoothing can be useful for certain regression tasks but do not consider dynamics modeling specifically. Similarly, \citet{li2018deeper} point out that GCNs \cite{kipf2017semisupervised} can be understood as a special case of Laplacian smoothing, which is a key reason why GCNs work at all. In fact, \citet{kipf2017semisupervised} argue that their architecture can be understood as a differentiable and parameterized generalization of the $1$-dim Weisfeiler-Lehman algorithm \cite{leman1968reduction}, indicating that even randomly initialized GCNs can be performant due to the way they smooth information throughout the network. Despite these findings, there is comparatively less work studying the role of smoothness in spatio-temporal modeling tasks. While \citet{marisca2025over} study issues with message-passing based GNNs for spatio-temporal modeling, they focus on \textit{oversquashing}, where information fails to propagate to distant nodes, whereas our work addresses \textit{oversmoothing} and \textit{undersmoothing}.

\paragraph{Dynamics modeling over graphs and meshes.} Our work focuses on dynamics modeling where PDE solutions are discretized as signals on graphs and meshes through the lens of smoothness. Many physical systems, such as wave propagation \citep{dalembert1747wave}, heat diffusion \citep{baron1822theorie}, phase fields \citep{cahn1958free, li2024tutorials}, fluid flows \citep{constantin1988navier, li2020neural}, and climate systems \citep{ghil2020geophysical} can be described by systems of PDEs. Deep learning based approaches are increasingly used to solve these PDEs in these domains where numerical solving is difficult \cite{10.1145/3394486.3403198, cranmer2020lagrangian, li2020neural, li2021fourier, kashinath2021physics, cai2021physics, maurizi2022predicting, park2023modelingdynamicsmeshesgauge, liukan, doi:10.1073/pnas.2311808121, daniels2025splat}. For PDE solving on meshes, these dynamics can be formulated extrinsically by embedding the manifold into Euclidean space \cite{satorras2022enequivariantgraphneural, pfaff2021learningmeshbasedsimulationgraph}, or intrinsically by defining evolution directly in the coordinates of local tangent spaces \cite{cohen2019gaugeequivariantconvolutionalnetworks,  dehaan2021gaugeequivariantmeshcnns, mitchel2021field, mitchel2022mobius, basu2022equivariant, park2023modelingdynamicsmeshesgauge, suk2024mesh, mitchel2024single}. While \citet{cohen2019gaugeequivariantconvolutionalnetworks},  \citet{pfaff2021learningmeshbasedsimulationgraph}, and \citet{suk2024mesh} contain isolated experiments related to dynamics modeling, only our work and \citet{park2023modelingdynamicsmeshesgauge} study how the choice of Euclidean versus locally defined coordinate representations in the network can affect convergence to PDE solutions. Furthermore, our work is distinct from \citet{park2023modelingdynamicsmeshesgauge} in that only we directly assess how these design choices affect neural network \textit{smoothing behavior}. 

\paragraph{Benchmarking PDE Surrogate Models.} While the physical symmetries of many dynamical systems are well understood \cite{olver1993applications, wang2020incorporating, borovitskiy2021matern}, smoothness has received less attention. The performance of deep dynamics models is typically measured either via quantitative error metrics against the ground truth or their preservation of underlying physical laws, such as spectral energy errors \cite{wang2020incorporating} or equivariance errors \cite{wang2020incorporating, wang2022approximately, wang2022data}. Our work is novel in our application of the Rayleigh quotient in quantifying the smoothing effect of trained GNN dynamics models. Furthermore, we are among the first to design architectures with inductive biases that encourage the model to match the Rayleigh quotient of the labeled graphs. Other works have explored using the Rayleigh quotient as an auxiliary loss \cite{rowan2025solving}, as positional encodings \cite{dong2023rayleigh}, and as a hard constraint to preserve smoothness regardless of the true labels \cite{kiani2024unitary}. In contrast, only our work and the work of \citet{shao2024unifying} use architectural inductive biases to match the smoothness of labeled graphs, and only our work evaluates how well the true smoothness of dynamical systems is recovered.

\section{Background}

We first recall the definition of the Rayleigh quotient, a measure of smoothness on graphs, and provide background on unitary convolutions and their invariance to the Rayleigh quotient. We also introduce the mesh data type, which we later use to extend the unitary convolution framework from graphs to meshes.

\subsection{Rayleigh quotient}

To measure the smoothness of a signal on a graph, we use the Rayleigh quotient as defined in \citet{chung1997spectral}.

\begin{definition}[Rayleigh quotient, \cite{chung1997spectral}] \label{def:RQ}
    Given an undirected graph $\mathcal{G}=(V,E)$ with $|V|=n$ nodes and adjacency matrix $\mathbf{A} \in \{0,1\}^{n \times n}$, let $\mathbf{D} \in \mathbb{R}^{n \times n}$ be a diagonal matrix where the $i$-th entry $\mathbf{D}_{ii} = d_i$  the degree of node $i$. Let $s \colon  V\to\mathbb{C}^d$ be a function from nodes to features. Denote by $\tilde{\mathbf{A}}= \mathbf{D}^{-1/2}\mathbf{AD}^{-1/2}$  the normalized adjacency matrix and $\mathbf{X} \in \mathbb{C}^{n \times d}$ a matrix with the $i$-th row set to feature vector $s(i)$. The Rayleigh quotient is
    \begin{equation}
        R_{\mathcal{G}}(\mathbf{X}) = \frac{1}{2}\frac{{\sum_{(u, v) \in E}} \left\| \frac{s(u)}{\sqrt{d_u}} - \frac{s(v)}{\sqrt{d_v}}\right\|^2}{{\sum_{w \in V}} || s(w) ||^2}  \label{eq:rayleigh_quotient_other_og}
    \end{equation}
    or $\Tr{(\mathbf{X^\dagger(I - \tilde{A})X})} \cdot 
||\mathbf{X}||^{-2}_F$ in matrix form. We will often abbreviate the Laplacian as $\mathbf{L = (I - \tilde{A})}$.
\end{definition}

Intuitively, the Rayleigh quotient measures the mean difference in node features for adjacent nodes. A graph with identical degree-weighted node features has a Rayleigh quotient of zero.

In this work, we define both oversmoothing and undersmoothing of a model $f$ with respect to the Rayleigh quotient of a target signal.\footnote{This is in contrast to prior works which define oversmoothing as exponential decay of the Rayleigh quotient over successive GNN layers towards its minimum, e.g. \citet{fesser2026unitary}.} 
%\begin{definition}[Target Free Oversmoothing \cite{fesser2026unitary}] \label{def:target_free}
%A GNN is said to \textit{oversmooth} on a graph $\mathcal{G}$ with Laplacian $\mathbf{L}$ and input $\mathbf{X}$ if there exist constants $C$, $C_0 > 0$ such that for every layer $k \geq 1$ we have $$|R_\mathcal{G}(\mathbf{X}) - \lambda_{\rm min}(\mathbf{L)}| \leq C_0e^{-Ck}.$$
%\end{definition}}

%\textcolor{blue}{In contrast, we define oversmoothing and undersmoothing with respect to the Rayleigh quotient of a target signal.}

\begin{definition}[Relative Oversmoothing and Undersmoothing] \label{def:RQ_free}
Let $\mathbf{A}$, $\mathbf{X}$, and $\mathcal{G}$ be as in \cref{def:RQ}. Let $\mathbf{Y} \in \mathbb{R}^{n \times d_{\rm out}}$ be a node feature matrix for a target signal defined on the same graph $\mathcal{G}$. Let $f: \mathbb{R}^{n \times d_{\rm in}} \times \mathbb{R}^{n \times n} \to \mathbb{R}^{n \times d_{\rm out}}$ be a function that updates the signal defined on the graph. If the predicted signal is smoother than the ground truth ($R_{\mathcal{G}}(\mathbf{f(X;A)}) < R_{\mathcal{G}}(\mathbf{Y})$), then we say that $f$ is \textit{oversmoothing}. If the predicted signal is less smooth than the ground truth ($R_{\mathcal{G}}(\mathbf{f(X;A)}) > R_{\mathcal{G}}(\mathbf{Y})$), then $f$ is \textit{undersmoothing}.
\end{definition}

\subsection{Unitary Convolution}

\citet{kiani2024unitary} define two different models that preserve the Rayleigh quotient using unitary functions, which satisfy $\mathbf{U^\dagger U = I}$. In particular, they define the separable unitary convolution 
\begin{equation}
f_{\rm UniConv}^{\mathrm{sep}}(\mathbf{X; A}) = \mathbf{\exp(iAt)XU}, \quad \mathbf{U^\dagger U =I} \label{eq:separable}
\end{equation}
and the Lie unitary convolution 
\begin{equation}
f_{\rm UniConv}^{\mathrm{Lie}}(\mathbf{X;A}) = \mathbf{\exp(AXW)}, \quad \mathbf{W = - W^\dagger} \label{eq:lie}
\end{equation}
where $\exp(\cdot )$ denotes the matrix exponential. We provide further background material on the matrix exponential and its relationship to unitary matrices in \cref{sec:exp}. The authors show that unitary convolutions are mathematically constrained to preserve the Rayleigh quotient: 
\begin{proposition}[Invariance of Rayleigh quotient, Proposition 6 in \citet{kiani2024unitary}] \label{prop:rq_preserved}
    Given an undirected graph $\mathcal{G}$ on $n$ nodes with normalized adjacency matrix $\widetilde{\mathbf{A}} = \mathbf{D}^{-1/2}\mathbf{A}\mathbf{D}^{-1/2}$, the Rayleigh quotient is invariant under normalized unitary or orthogonal graph convolution, i.e. $R_{\mathcal{G}}(\mathbf{X}) = R_{\mathcal{G}}(f_{\rm UniConv}(\mathbf{X}))$ where $f_{\rm UniConv}$ is either seperable or Lie.
    %(see \cref{eq:separable} and \cref{eq:lie}).
\end{proposition}
%\textcolor{blue}{Unitary convolutions are therefore constrained to avoid oversmoothing with respect to \cref{def:target_free}, but may oversmooth or undersmooth \textit{a target signal} with respect to \cref{def:RQ_free}.}
%
%Crucially, separable unitary convolution is often relaxed by dropping the constraint that $\mathbf{U^\dagger U = I}$, meaning that \cref{prop:rq_preserved} is not strictly satisfied. Additionally, Taylor series truncation errors in the matrix exponential can cause both separable and Lie unitary convolution to violate \cref{prop:rq_preserved}. %\citet{kiani2024unitary} also show that conventional graph neural networks such as GCNs \cite{kipf2017semisupervised} are likely to exhibit oversmoothing. We include the proposition in  \cref{sec:GCN_smoothing} for completeness.

\subsection{Mesh Data} A (triangular) mesh $\mathcal{M}$ consists of a set $(\mathcal{V}, \mathcal{E}, \mathcal{F})$, where $\mathcal{V}$ is a set of vertices, $\mathcal{E}= \{(i,j)\}$ is a set of ordered vertex indices $i$, $j$ connected by an edge, and $\mathcal{F} = \{(i,j,k)\}$ is the set of ordered vertex indices $i$,$j$,$k$ connected by a triangular face. The mesh generalizes graphs by including high order connectivity information via the inclusion of faces. We assume that the mesh is a $2-$dimensional manifold embedded in $\mathbb{R}^3$, i.e. a manifold mesh. The definition of the manifold condition for a mesh is given in \cref{def:man_cond} (\cref{sec:mesh_laplacian}). %  In \cref{sec:symmetries}, we provide further detail on various equivariance constraints commonly employed for tasks defined on meshes. For more background on discrete differential geometry, see \citet{Crane2013DGP}.

\section{Theory: Unitary Models are Overconstrained} \label{sec:theory}

While unitary models on graphs can be useful because they preserve the Rayleigh quotient, this section illustrates how unitary models can be \textit{overly} constrained. In particular, we derive an approximation error lower bound that clarifies the approximation limits of unitary models. We start by establishing our unitary approximation learning framework. 

\subsection{Preliminaries}
Let $Z = \mathbb{C}^n$ be a domain with data probability density $p \colon Z \to \mathbb{R}$. Let $u\colon \mathbb{C}^n \to \mathbb{C}^n$ be a unitary function and let $f\colon \mathbb{C}^n  \to \mathbb{C}^n$ be the target function. Denote the regression error by 
\begin{equation*}
\errreg{u} = \int_{Z}p(z) \lVert u(z) - f(z)\rVert_2^2dz.
\end{equation*}
\paragraph{Group Invariance.} Our main result relies on the theory of approximation error for group invariant functions $h$. We review these concepts in detail in \cref{sec:unitary_learning} and provide informal definitions in the paragraph that follows. 

A group invariant function $h$ satisfies $h(z) = h(gz)$ for all $g \in G$, $z \in Z$. Let $Gz = \{gz\colon g \in G\}$ be the orbit of $z$. A \textit{fundamental domain} $F$ of a group $G$ in $Z$ is a set of orbit representatives. The domain $Z$ can be written as the union of conjugates, $Z = \cup_{g \in G}gF$, where the conjugate is defined $gF = \{gz \colon z \in F\}$. Denote the integrated density on an orbit by $p(Gz) = \int_{Gz}p(z)dz.$ Finally, denote the variance of a function $f$ on an orbit $Gz$ by $\mathbb{V}_{Gz}[f]$. The approximation error lower bound for an invariant function is given by the following proposition.
\begin{proposition}[Theorem 4.8 in \citet{wang2023general}] \label{prop:inv} For a $G$-invariant function $h$, the regression error is bounded below by $\err \geq \int_F p(Gz) \mathbb{V}_{Gz}[f]dz$.
\end{proposition}
\cref{prop:inv} was initially stated for real-valued functions in \citet{wang2023general}, but can be applied to complex valued functions without loss of generality. Furthermore, \citet{wang2023general} provide numerical evidence that \cref{prop:inv} is a tight bound. % This instills confidence that the bound we will prove in \cref{thm:uni_lb} is also tight, since it is a direct application of the proposition.

\subsection{Unitary Approximation Error Lower Bound}

In this subsection, we state our main theoretical result, which demonstrates that unitary neural networks are overconstrained when the norm of the ground truth function has a high angular dependence. Recall the definition of $\mathrm{SU}(n)$, the group of rotations in $\mathbb{C}^n$:
\begin{equation*}
    \mathrm{SU}(n) =  \{U \in \mathbb{C}^{n \times n} \colon \det (U) = 1\}.
\end{equation*}
We can now give an approximation error lower bound for unitary models. See \cref{sec:proof} for the proof.

\begin{theorem} \label{thm:uni_lb}
Let $F$ be a fundamental domain of $\mathrm{SU}(n)$ in $Z$, e.g. $F = \{ t e \colon t \in \mathbb{R}_+\}$ where $e$ is a standard basis vector of $\mathbb{C}^n$. The approximation error of $u$ of $f$ has lower bound
\[
\int_{Z}p(z)\lVert u(z) - f(z)\rVert_2^2dz \geq \int_{F} p(\lVert te \rVert) \mathbb{V}_{Gz}[\rVert f\lVert]dz.
\]
\end{theorem}

The proof of \cref{thm:uni_lb} uses the reverse triangle inequality before applying \cref{prop:inv}. Intuitively, the fundamental domain enumerates all concentric spheres $S^{2n-1}$ embedded in $\mathbb{C}^n$. Unitary functions are complex valued rotations and reflections that preserve the norm of data points that live on each sphere. The error lower bound is given by the variance of the norm of $f$ averaged over each concentric sphere where the norm of $u$ is constant. Our result suggests that unitary functions can be particularly overly constraining when the norm of $f$ has a high angular dependence.  %Because \cref{thm:uni_lb} is a direct application of \cref{prop:inv}, we are confident that \cref{thm:uni_lb} presents a tight bound as well.

\section{Unitary Convolution Constraint Relaxation} \label{sec:relax_uni}

Since \cref{thm:uni_lb} informs us that a unitary convolution network may be overconstraining when the ground truth is not perfectly smoothness preserving, we propose two methods for relaxing unitary convolutions and describe how to extend these architectures from graphs to meshes. We name these methods Taylor truncation and compositional smoothness-breaking. The Taylor truncation method allows precise control of the extent of relaxation, whereas the compositional smoothness-breaking method scales more easily with the number of parameters in the network. The Taylor truncation method is especially useful in cases where the true smoothness is known a priori, in which case theoretical results from the literature (see \cref{sec:sensitivity_more_details}) can inform what truncation order $\mathbf{T_{\rm max}}$ is needed to achieve enough relaxation without grid search or hyperparameter tuning. We will use the first relaxation for a motivating experiment in \cref{sec:motivate} and the second for a more challenging set of tasks in \cref{sec:meshes}. A comparison of the compositional smoothness-breaking and Taylor truncation methods is shown in \cref{fig:methods}. We will use $f_{\rm Layer}$ to define individual layers and \textsc{small caps text} to define architectures constructed from those layers.

\begin{figure}[!htb]
    \centering
    \includegraphics[width=1.0\linewidth]{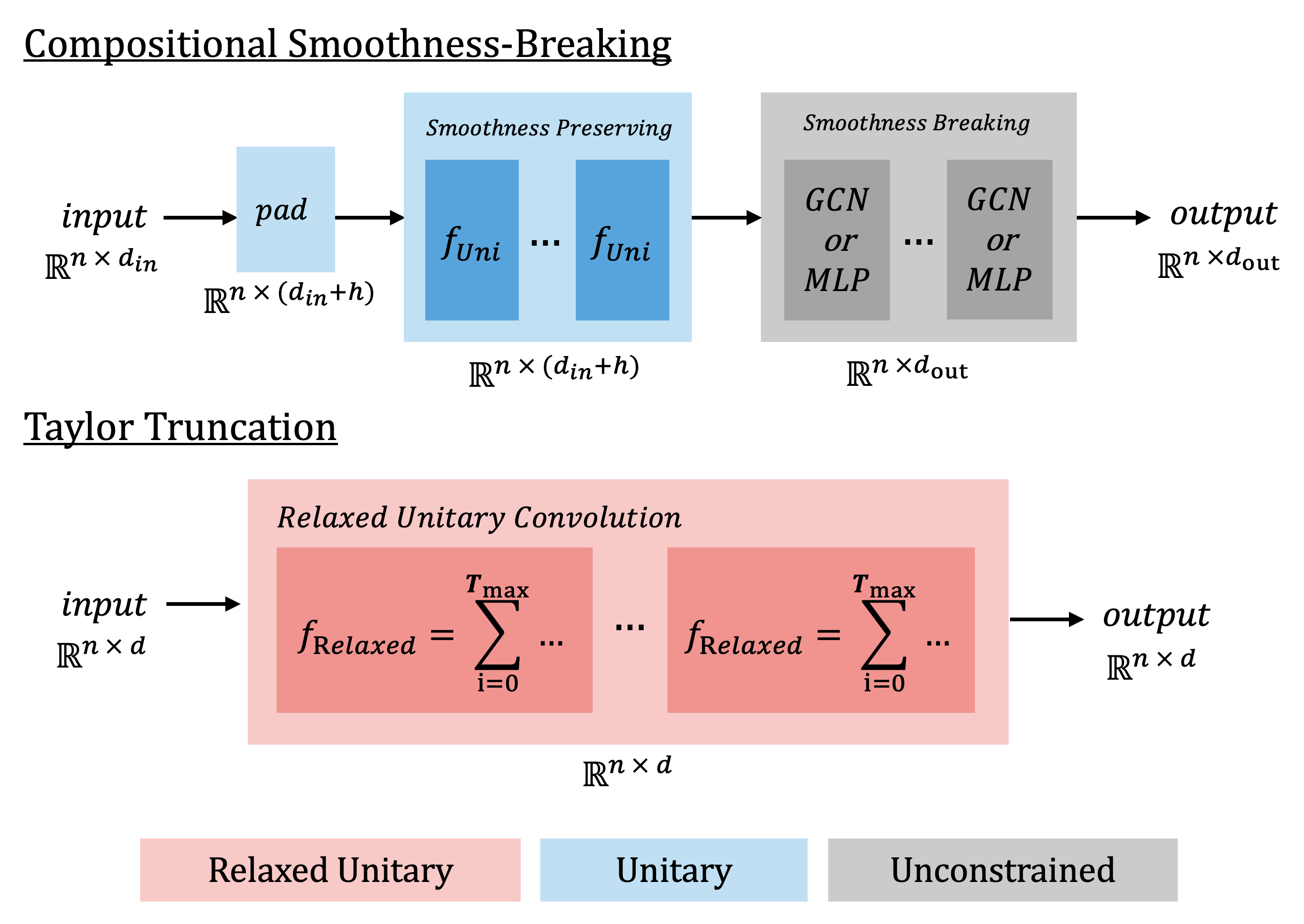}
    \caption{\textbf{Top:} After zero padding, individual unitary blocks are stacked and the output is fed into an unconstrained smoothness-breaking component. \textbf{Bottom:} Each block uses Taylor truncated unitary convolution. } 
    \label{fig:methods}
\end{figure}

\subsection{Relaxed Unitary Convolution via Taylor Truncation} \label{sec:Taylor}

We relax Lie unitary convolution by truncating the Taylor series approximations used in \cref{eq:lie}. We note that \citet{kiani2024unitary} do propose their own constraint relaxation for separable unitary convolutions \cref{eq:separable}. However, their approach conflates two sources of relaxation: Taylor series truncation of the matrix exponential and letting $\mathbf{U}$ remain unconstrained, making it difficult to measure the individual contributions of each source and to tune the extent of the relaxation. By relaxing the Lie unitary convolution rather than the separable unitary convolution, we isolate the architectural component that alters the Rayleigh quotient. Instead of approximating the matrix exponential using enough Taylor series terms so that the truncation error is negligible, we truncate at some $T= \mathbf{T_{\rm max}}$ where $\mathbf{T_{\rm max}}$ controls the extent of the relaxation. Our Taylor-relaxed unitary convolution layer is defined

% by allowing $\mathbf{U}$ to be unconstrained in \cref{eq:separable}, and that our approach instead relaxes \cref{eq:lie}. This allows us to isolate the architectural component that alters the Rayleigh quotient. In contrast, the relaxation in \citet{kiani2024unitary} can be achieved via two different mechanisms: early Taylor series truncation of the matrix exponential and letting $\mathbf{U}$ remain unconstrained in \cref{eq:separable}. The downside of this approach is that the relative contributions of each are difficult to measure and it is therefore harder to tune the extent of the relaxation. Our relaxation is simply Taylor series truncation of Lie unitary convolution (\cref{eq:lie}). Instead of approximating the matrix exponential using enough Taylor series terms so that the truncation error is negligible, we truncate at some $T= \mathbf{T_{\rm max}}$ where $\mathbf{T_{\rm max}}$ controls the extent of the relaxation. Our Taylor-relaxed unitary convolution layer is defined
%
\begin{equation}
f_{\rm Relaxed}(\mathbf{X}; \mathbf{A}, \mathbf{T}_{\rm max})= \overset{\mathbf{T}_{\rm max}}{\underset{i=0}{\sum}}\frac{1}{i!}\mathbf{L}^i(\mathbf{X}) \label{eq:relaxed}
\end{equation}
where $\mathbf{L}(\mathbf{X)} = \mathbf{AXW}$, $\mathbf{W} = -\mathbf{W^\dagger}$. While this approach does not preserve the Rayleigh quotient for small $\mathbf{T_{\rm max}}$, we recover the standard Lie unitary convolution in \cref{eq:lie} in the limit as $T \to \infty$.  %For small $\mathbf{T_{\rm max}}$, the computation time is also significantly reduced compared to standard Lie Unitary convolution.

Motivated by the desire to find an appropriate $\mathbf{T_{\rm max}}$ that applies only a small perturbation to the Rayleigh quotients of input graphs, in \cref{sec:sensitivity_analysis_taylor} we conduct a sensitivity analysis of the Rayleigh quotient to different Taylor series truncations. Consistent with \citet{kiani2024unitary}, we find that $\mathbf{T_{\rm max}}=10$ is sufficient to preserve the Rayleigh quotient. We will refer to models constructed from relaxed layers in \cref{eq:relaxed} as \oursGraph{}.  % Accordingly, we will use $f_{\rm Relaxed}(\mathbf{X}; \mathbf{A}, 10)$ interchangeably with $f_{\rm UniConv}^{\rm Lie}(\mathbf{X}; \mathbf{A})$.

\subsection{Relaxed Unitary Convolutional Models via compositional smoothness-breaking} \label{sec:encoder_decoder}

In this section, we note a limitation of relaxed unitary convolution via Taylor truncation that makes it difficult to scale, and propose an alternative relaxation method that addresses this. Since unitary layers cannot change the channel dimension of the node features, the only way to increase the number of parameters in the network is to increase the number of unitary layers. This means that scaling unitary convolutional models requires making the model deeper, which can make training unstable \citep{balduzzi2017shattered}. As an alternative, we propose a compositional smoothness-breaking method, which subsequently maps the signal through a smoothness preserving block $P$ and then through a smoothness breaking block $B$.The smoothness preserving block first zero-pads the input node features to the desired hidden dimension. This allows us to increase the parameter count without increasing depth. Zero padding also trivially preserves the Rayleigh quotient since it preserves norms. Concretely, we define our zero padding function $f_{\rm pad} \colon \mathbb{R}^{n\times d_{\rm in}} \to \mathbb{R}^{n\times d_{\rm h}}$ by $\mathbf{X} \mapsto \mathbf{X} \oplus 0$, where $0$ is the $\mathbb{R}^{n \times (d_{\rm h} - d_{\rm in})}$ zero matrix. Our approach to increasing the latent dimension is distinct from most prior works on unitary convolution which instead use learnable maps \cite{kiani2024unitary, fesser2026unitary}, though it was hypothesized that zero padding would be an effective approach for approximately unitary models in \citet{kiani2024unitary}. Next, we define our $k$-layer smoothness-preserving block $P$ by 
\begin{equation} 
P = f_{\rm UniConv}^{(k)} \circ \dots \circ f_{\rm UniConv}^{(1)}(f_{\rm pad}(\mathbf{X}), \mathbf{A})  \label{eq:encoder} 
\end{equation}
where $f_{\rm UniConv}$ is either separable or Lie. The smoothness-breaking block $B$ then serves two purposes: (\textbf{i}) map to the target node feature dimension and (\textbf{ii}) break the unitary constraint. The smoothness-breaking block can be any network, such as an MLP or GCN. 

\subsection{Relaxed Unitary Convolution on Meshes} \label{sec:relaxed_meshes}

We now generalize unitary convolutional models, relaxed unitary convolutional models, and the mathematical definition of Rayleigh quotient from graphs to meshes. This enables us to solve dynamics problems on manifolds by discretizing them as meshes, such as testing the thermal stability of mechanical parts or weather forecasting on Earth. In particular, we prove that, under modest assumptions on the mesh triangulation, unitary convolution with a weighted adjacency matrix preserves the Rayleigh quotient on meshes (\cref{def:mesh_rq}); enabling generalization of both unitary and relaxed unitary models to meshes.

\paragraph{Mesh Rayleigh Quotient.} We first generalize the Rayleigh quotient from graphs to meshes by using the mesh Laplacian instead of the graph Laplacian. The Laplacian on a mesh is typically defined as the symmetric cotangent Laplacian \citep{reuter2009discrete} given in \cref{eq:new_laplacian}, which approximates the Laplace-Beltrami operator for the continuous manifold which the mesh discritizes. For a scalar function $s$ defined on nodes,
\begin{equation}
    (\mathbf{\tilde{L}}(s))_i = \frac{1}{2A_i}\underset{j \in \mathcal{N}(i)}{\sum} \left(\cot \alpha_{ij} + \cot \beta_{ij} \right)(s_j - s_i) \label{eq:new_laplacian}
\end{equation}
where $\mathcal{N}(i)$ denotes the adjacent vertices of $i$, $\alpha_{ij}$ and $\beta_{ij}$ are the angles opposite edge $(i, j)$, and $A_i$ is the vertex area of $i$. We use the barycentric cell area for $A_i$. We note that it is invalid to define the mesh Rayleigh quotient by replacing $\mathbf{L}$ in \cref{eq:rayleigh_quotient_other_og} with the symmetric cotangent Laplacian $\mathbf{\tilde{L}}$ in \cref{eq:new_laplacian}. The cotangent weights in \cref{eq:new_laplacian} may be negative, which means that the Rayleigh quotient is no longer a valid measure of smoothness \citep[Definition 1,][]{rusch2023survey}. To address this, we use the \textit{Robust Laplacian} \cite{sharp2020laplacian}, which performs a minimal edge rewiring of the mesh so that the cotangent weights obey the \textit{intrinsic Delaunay criterion}. 

\begin{definition}[Intrinsic Delaunay Criterion, \citep{bobenko2007discrete}] \label{def:delaunay}
For all faces connected by an edge $(i,j)$ with opposite angles $\alpha_{ij}$ and $\beta_{ij}$, $\alpha_{ij} + \beta_{ij} \leq \pi$. %This is true if and only if $\cot \alpha_{ij} + \cot \beta_{ij} \geq 0$.
\end{definition}

Concretely, this means that our Laplacian weights are symmetric and the off-diagonals are nonnegative; the edge rewiring simply provides an alternative discretization of the same manifold. Denote by $\mathcal{W}$ the cotangent weights 
\begin{equation}
    \mathcal{W}_{ij} = \begin{cases}
        \frac{1}{2}\left(\cot \alpha_{ij} + \cot \beta_{ij}\right), & j \in \mathcal{N}(i)\\
        - \underset{k \in \mathcal{N}(i)}{\sum}\mathcal{W}_{ik}, & i = j\\
        0, & \text{Otherwise.}
    \end{cases} \label{eq:cot_weights}
\end{equation}
We define a novel Rayleigh quotient for meshes as follows. 

\begin{definition}[Mesh Rayleigh Quotient] \label{def:mesh_rq}
Let $\mathcal{M}=(\mathcal{V},\mathcal{E},\mathcal{F})$ be a mesh with $|V|=n$ nodes. Denote by $\mathcal{W}$ the cotangent weights corresponding to the Robust Laplacian $\tilde{\mathbf{L}}$. Denote by $\mathcal{E}'$ the rewired edge set given by $\mathcal{E}' = \{(u,v) \colon \mathcal{W}_{uv} \neq 0\}$. Let $s$ and $\mathbf{X}$ be the same as in \cref{def:RQ}. The mesh Rayleigh quotient is defined
 \begin{equation*}
        R_{\mathcal{M}}(\mathbf{X}) = \frac{1}{2}\frac{\underset{(u, v) \in \mathcal{E}'}{\sum} \mathcal{W}_{uv}\left\| \frac{s(u)}{\sqrt{d_u}} - \frac{s(v)}{\sqrt{d_v}}\right\|^2}{\underset{w \in V}{\sum} \left\| s(w) \right\|^2}  = \frac{\Tr{(\mathbf{X^\dagger \tilde{L}X})}}{
||\mathbf{X}||^2_F}.
    \end{equation*}
\end{definition}

%For further details on the cotangent Laplacian, see \cref{sec:mesh_laplacian}.

\paragraph{Unitary Mesh Convolution.} We now make similar modifications to generalize unitary convolution from graphs to meshes. Specifically, we modify the functions in \cref{eq:separable} and \cref{eq:lie} by incorporating the cotangent weights (\cref{eq:cot_weights}) into the normalized adjacency matrix $\mathbf{\tilde{A}}$. We note that this edge weighting been shown to improve PDE solving numerically \citep{crane2017heat, sharp2020laplacian} and we are the first to use it in deep dynamics models. In order to prove that incorporating these weights preserves the Rayleigh quotient given by \cref{def:mesh_rq}, we assume the mesh already satisfies the Delaunay criterion.

\begin{assumption}[Mesh Weights Obey the Delaunay Criterion] \label{ass:assumption}
For a mesh $\mathcal{M}$, the mesh is manifold and all angles obey the Delaunay Criterion given by \cref{def:delaunay}.
\end{assumption}

We note that there are existing triangulation strategies that a practitioner can use to ensure that mesh edges satisfy this criterion as a standard data preprocessing step \citep{huang2018robust, sharp2020laplacian}, see \cref{sec:mesh_laplacian} for details. With this assumption, we will now define unitary mesh convolution. Let $\mathbf{D}$ be the weighted degree matrix defined by $\mathbf{D}_{ii} = \sum_{i \neq j}\mathcal{W}_{ij}$. Let $\odot$ represent the Hadamard product which performs element-wise matrix multiplication. Let $\mathbf{\tilde{A}}$ be the normalized adjacency matrix $\mathbf{\tilde{A}} = \mathbf{D^{-1/2}}\left(\mathcal{W} \odot \mathbf{A}\right)\mathbf{D^{-1/2}}$. We define separable unitary mesh convolution as
\begin{equation}
f_{\rm UniMeshConv}^{\mathrm{Sep}}(\mathbf{X;A, \mathcal{W}}) = \exp (i \mathbf{\tilde{A}})\mathbf{XU}\label{eq:msep} 
\end{equation}
and Lie unitary mesh convolution as
\begin{equation}
f_{\rm UniMeshConv}^{\mathrm{Lie}}(\mathbf{X;A},\mathcal{W}) = \exp (\mathbf{\tilde{A}}\mathbf{XW})\label{eq:mlie}
\end{equation}
where $\mathbf{UU^\dagger = I}$ and $\mathbf{W + W^\dagger = 0}$. The following Corollary (proven in \cref{sec:mesh_conv}) states that \cref{eq:msep} and \cref{eq:mlie} preserve the Rayleigh quotient on meshes.
% , i.e., for $C=A \odot B$ we have $C_{ij} = A_{ij}B_{ij}$

\begin{corollary}[Corollary to \cref{prop:rq_preserved}] \label{cor:main_cor}
Given a mesh $\mathcal{M}$ with normalized adjacency matrix $\mathbf{\tilde{A}} = \mathbf{D^{-1/2}}(\mathcal{W} \odot \mathbf{A})\mathbf{D^{-1/2}}$ that satisfies \cref{ass:assumption}, the mesh Rayleigh quotient is invariant under normalized unitary or orthogonal graph convolution, i.e. $R_{\mathcal{M}} (\mathbf{X}) = R_\mathcal{M} (f_{\rm UniMeshConv} (\mathbf{X}))$ where $f_{\rm UniMeshConv}$ is either separable or Lie.
\end{corollary}

\paragraph{Relaxed Unitary Mesh Convolution.} We create a network architecture by coupling the compositional smoothness-breaking relaxation (\cref{sec:encoder_decoder}) with Lie unitary mesh convolution in \cref{eq:mlie}. Concretely, nodes are first zero padded. The smoothness-preserving block $P^{(k)}$ (\cref{eq:encoder}) is constructed from $k$ layers of Lie unitary mesh convolution (\cref{eq:mlie}), and a MLP or GCN smoothness-breaking block $B$ maps to the target. We name our relaxed model \ours{}. \ours{} uses the norm preserving GroupSort activation from \citet{anil2019sorting} to introduce nonlinearity. \ours{} also uses \textit{orthogonal} weights, since our modeling tasks on meshes in \cref{sec:experiments} are real valued.

\section{Experiments} \label{sec:experiments}

\subsection{Motivating Experiment: Heat Flow on Grid Graphs} \label{sec:motivate}

In this section, we motivate the use of relaxed unitary models by showing that \oursGraph{} is able to outperform both normal and Lie unitary graph convolution on predicting heat diffusion. The Taylor truncation method is key to balancing the smoothness preservation of unitary models with the flexibility to capture the true smoothness of the target heat graph. %We attribute the success of \oursGraph{} not only to constraint relaxations but also the tendencies of GCNs to oversmooth at initialization. We provide evidence for this by studying the training curves for each model. 

\paragraph{Heat Diffusion Dataset.} We use \texttt{PyGSP} \cite{pygsp} to simulate heat diffusion on $10,000$ two-dimensional grids, each initialized with $20$ randomly placed heat sources. Denote by $H \colon \mathbb{R}_+ \to \mathbb{R}^{n}$ a function that maps time $t$ to the heat distribution of the graph with $n$ nodes. In other words, the heat field on the graph is represented by a feature vector in $\mathbb{R}^{n \times 1}$. In particular, $H(t)= e^{-\tau t \mathbf{L}} H(0)$ where $\tau$ is a diffusivity constant, $\mathbf{L}$ is the graph Laplacian, and $H(0)$ is the initial heat values across the graph. The task is to predict the heat distribution on the graph at time $t=4$ given the heat distribution at time $t=3$. We denote the target heat field as $\mathbf{Y} = H(4)$. See \cref{sec:dataset_details} for further dataset details. We compare the performance in terms of MSE loss and mean Rayleigh quotient error for three models: GCN, \oursGraph{}, and a Lie unitary model. We use $\mathbf{T_{\rm max}} = 3$ for the relaxed model. The mean Rayleigh quotient error is given by 
\begin{equation*}
\mathrm{MRE}(f) = |\overline{R_{\mathcal{G}}(f(\mathbf{X}))} - \overline{R_{\mathcal{G}}(\mathbf{Y})}|.
\end{equation*}
%
%Further dataset details are given in \cref{sec:dataset_details}.

\paragraph{Results.} As shown in \cref{tab:relaxed_performance} and \cref{fig:loss_and_rq}, the relaxed model significantly outperforms the GCN and also outperforms the Lie unitary model. Moreover, the relaxed model is best able to produce graphs whose smoothness matches that of the true labels. Our results validate Proposition 7 in \citet{kiani2024unitary} (also provided in \cref{sec:GCN_smoothing}), which states that GCNs with standard weight initialization tend to increase the smoothness of input signals on the graph as measured by the Rayleigh quotient. Crucially, our experiments reveal that GCNs not only smooth the input signal, but they \textit{oversmooth} compared to the target signal, even when some smoothing is desirable, i.e. the output signal is smoother than the input. In contrast, \oursGraph{} is often initialized in an \textit{undersmooth} state and is able to learn to increase the smoothness of input signals during training to match the smoothness of the target. These insights extend to heat flow on more intricate mesh datasets, as we show in \cref{sec:meshes}.

\paragraph{Ablation.} We analyze the sensitivity of our method to different $\mathbf{T}_{\rm max}$ in \cref{sec:sensitivity_analysis_taylor}, finding that our method is stable across different degrees of relaxation.

\begin{table}[!t]
    \centering
    \tiny
    \begin{tabular}{ccc} 
    \toprule
    Model     & MSE $(\downarrow)$ & MRE $(\downarrow)$ \\ \midrule 
    \hyperlink{cite.kipf2017semisupervised}{GCN} & $1.08 \cdot 10^{-2}$ & $5.99 \cdot 10^{-2}$ \\
    \hyperlink{cite.kiani2024unitary}{Lie Uni}  & $0.14 \cdot 10^{-2}$ & $8.86 \cdot 10^{-2}$ \\ 
    \oursGraph{} \textbf{(Ours)}  & $\mathbf{0.11 \cdot 10^{-2}}$  & $\mathbf{2.07 \cdot 10^{-2}}$\\ \bottomrule
    \end{tabular}
    \caption{MSE and mean Rayleigh quotient error (MRE) of GCN, a Lie unitary convolution network, and \oursGraph{}. The best run for each method out of $5$ runs is shown. The best performance is bold.}
    \label{tab:relaxed_performance}
\end{table}

\begin{figure}[!htb]
    \centering
    \includegraphics[width=1\linewidth]{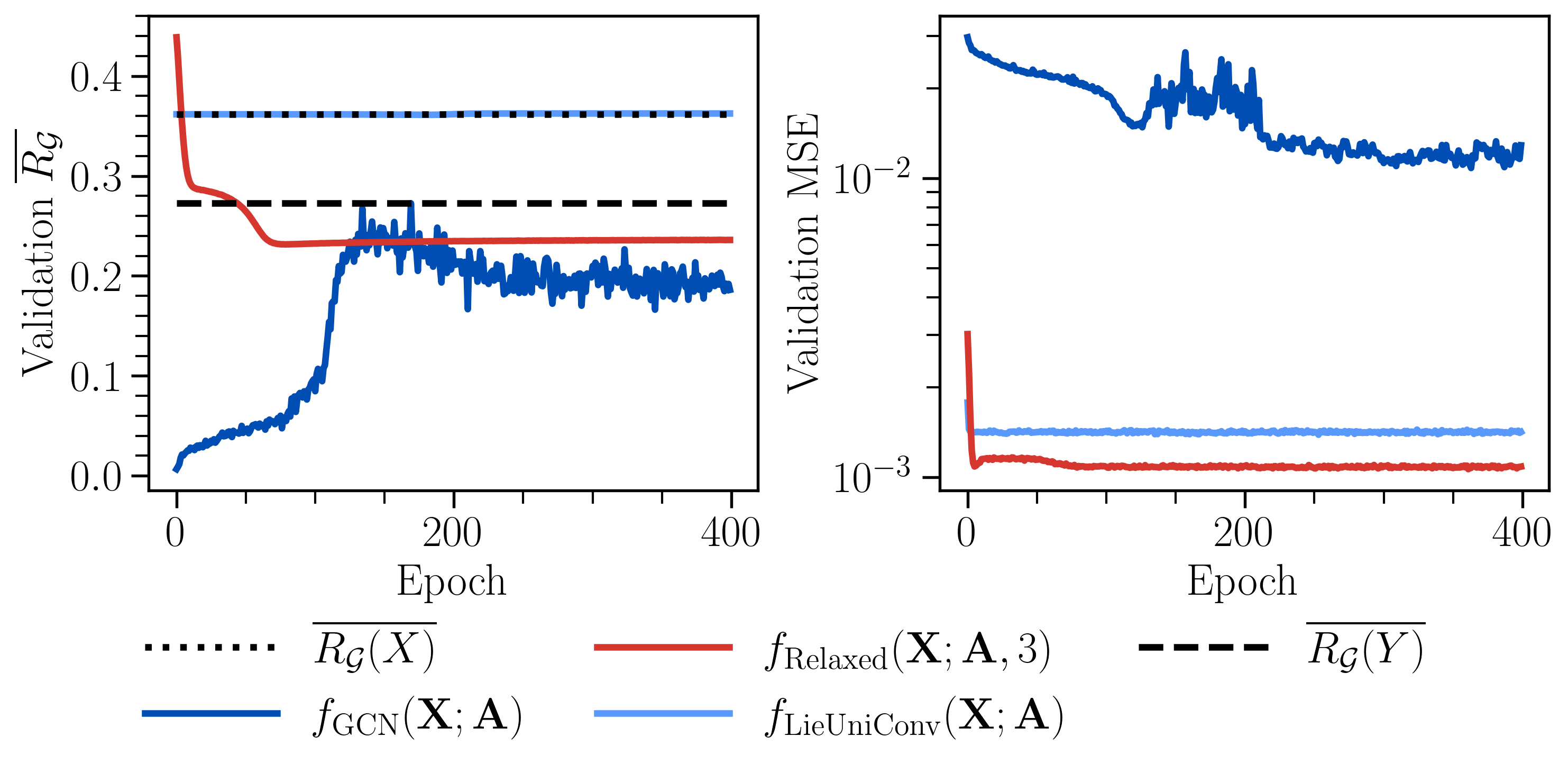}
    \caption{\textbf{Left:} The dotted lines indicate the mean Rayleigh quotient for input heat graphs at $T=3$ and target graphs at $T=4$. We also show the mean Rayleigh quotient for the best performing GCN, \oursGraph{}, and Lie unitary models. \oursGraph{} is best at capturing the true smoothness. \textbf{Right:} Validation MSE of the same three models. \oursGraph{} has the best performance.  Results for the full set of runs are provided in \cref{sec:ensemble}.}  
    \label{fig:loss_and_rq} % The unitary model is trained for less epochs due to compute limitations.
\end{figure}

\subsection{Dynamics on Mesh Manifolds} \label{sec:meshes}

We now consider more challenging and realistic tasks and show that \ours{} is competitive with strong baselines and outperforms all other models on diffusive dynamics. Specifically, our experiments reveal the following practical conclusions: \textbf{(i)} \ours{} performs as well as strong baselines such as mesh-aware transformers and equivariant neural networks; \textbf{(ii)} \ours{} is especially strong on heat diffusion tasks; \textbf{(iii)} geometric inductive biases, such as unitarity or equivariance, are important for strong performance on unseen meshes with complex geometries. We support these conclusions with simulated and real-world dynamics datasets on mesh manifolds, including PDE solving on the \pv{} \cite{sullivan2019pyvista} meshes from \citet{park2023modelingdynamicsmeshesgauge} and weather forecasting on the Earth mesh from WeatherBench$2$ (WB$2$) \cite{rasp2024weatherbench}.

\begin{table*}[!htb]
\centering
\tiny
\begin{tabular}{lcccccccc}
\toprule
& \multicolumn{3}{c}{\textbf{Convolutional}} 
& \multicolumn{2}{c}{\textbf{Attentional}} 
& \multicolumn{3}{c}{\textbf{Message Passing}} \\
\cmidrule(lr){2-4}
\cmidrule(lr){5-6}
\cmidrule(lr){7-9}
\textbf{Metric} &
\hyperlink{cite.kipf2017semisupervised}{GCN} &
\hyperlink{cite.dehaan2021gaugeequivariantmeshcnns}{GemCNN} &
\ours{} \textbf{(Ours)} &
\hyperlink{cite.basu2022equivariant}{EMAN} &
\hyperlink{cite.janny2023eagle}{Transformer} & 
\hyperlink{cite.g2017}{MPNN} &
\hyperlink{cite.satorras2022enequivariantgraphneural}{EGNN} &
\hyperlink{cite.park2023modelingdynamicsmeshesgauge}{Hermes} \\
\midrule
\multicolumn{9}{c}{\textbf{Heat} $(\alpha  = 1)$} \\ \midrule
NRMSE $(\downarrow)$
& -- & -- & $\mathbf{51.9 \pm 3.6}$ & $73.50 \pm 3.8$ & $92.5 \pm 5.6$ & $99.45 \pm 4.8$ & $344.25 \pm 110.5$ & $\underline{73.02 \pm 4.7}$ \\
SMAPE $(\downarrow)$
& --  & $375.4 \pm 0.53$ & $\mathbf{79.7 \pm 5.6}$ & $110.9 \pm 13.3$ & $213.9 \pm 2.7$ & $223.6 \pm 1.5$ & $319.33 \pm 7.59$ & $\underline{107.6 \pm 7.4}$ \\
RE $(\downarrow)$
& -- & $52.21 \pm 9.4$ & $\mathbf{9.1 \pm 7.4}$ & $\underline{14.2 \pm 1.4}$ & $46.0 \pm 3.7$ & $76.06 \pm 3.6$ & $81.5 \pm 8.77$ & $39.76 \pm 4.7$ \\
\midrule
\multicolumn{9}{c}{\textbf{Wave} $(c  = 1)$} \\ \midrule
NRMSE $(\downarrow)$
& -- & -- & $\mathbf{236.5 \pm 6.4}$ & $\underline{281.3 \pm 15.5}$ & $864.9 \pm 184.9 $ & $563.6 \pm 7.75$ & $2280.1 \pm 559.9$ & $458.5 \pm 13.0$ \\
SMAPE $(\downarrow)$
& -- & $318.8 \pm 3.9$ & $385.2 \pm 1.2$ & $\mathbf{301.0 \pm 4.2}$ & $327.0 \pm 4.4$ & $318.0 \pm 2.8$ & $354.3 \pm 11.0$ & $\underline{316.4 \pm 4.5}$ \\
RE $(\downarrow)$
& -- & $107.9 \pm 3.158$ & $93.5 \pm 25.4$ & $73.57 \pm 6.5$ & $\mathbf{48.0 \pm 7.9}$ & $139.3 \pm 10.1$ & $157.2 \pm 14.8$ & $\underline{70.03 \pm 6.1}$ \\
\midrule
\multicolumn{9}{c}{\textbf{Cahn-Hilliard}} \\ \midrule
NRMSE $(\downarrow)$
& -- & $\mathbf{121.2 \pm 1.8}$ & $123.9 \pm 2.6$ & $137.5 \pm 0.69$ & $144.4 \pm 0.8$ & $147.4 \pm 11.36$ & $1001.04 \pm 5.73$ & $\underline{122.0 \pm 7.8}$ \\
SMAPE $(\downarrow)$
& -- & $204.3 \pm 2.4$ & $\underline{167.3 \pm 10.6}$ & $\mathbf{143.7 \pm 2.5}$ & $191.7 \pm 2.0$ & $201.22 \pm 32.79$ & $336.5 \pm 2.777$ & $173.3 \pm 4.3$ \\
RE $(\downarrow)$
& -- & $\mathbf{10.68 \pm 3.3}$ & $18.9 \pm 10.4$ & $48.57 \pm 3.49$ & $27.42 \pm 2.87$ & $23.98 \pm 6.51$ & $41.8 \pm 1.997$ & $\underline{14.38 \pm 11.5}$ \\
\bottomrule
\end{tabular}
\caption{NRMSE, SMAPE, and RE averaged over all rollouts on all test meshes for the heat, wave, and Cahn-Hilliard equations. The best values are in bold and second best are underlined. Errors and standard deviations are reported over all test meshes and initializations. Cells with a dash (--) indicate models that do not converge for a given metric. Errors are scaled by $\times 196$, the number of rollout timesteps. \ours{} is competitive across all tasks and excels at solving the heat equation on unseen meshes. Models are grouped together by flavor: convolutional, attention, or message passing \cite{bronstein2021geometric}.}   
\label{tab:IRE}
\end{table*}

\paragraph{Baselines.} We include as baselines standard GNN models without any specific inductive biases for working on meshes, including a GCN \cite{kipf2017semisupervised} and an MPNN \cite{g2017}. Additionally, we study symmetry preserving equivariant models, including gauge and Euclidean equivariance (formally defined in \cref{sec:symmetries}). Informally, Euclidean equivariant models are invariant to roto-translations of the mesh in Cartesian coordinates and Gauge Equivariant GNNs are invariant to a choice of reference angle for models that work in local coordinates of the mesh-manifold. These models have been shown to be particularly strong for PDE solving on meshes \citep{park2023modelingdynamicsmeshesgauge}. We benchmark a state-of-the-art (SOTA) Euclidean equivariant model \cite{satorras2022enequivariantgraphneural} as well as different types of Gauge Equivariant GNNs, including convolutional with GemCNN \cite{dehaan2021gaugeequivariantmeshcnns}, attentional with EMAN \cite{basu2022equivariant}, and message passing with Hermes \cite{park2023modelingdynamicsmeshesgauge}. We also consider a SOTA mesh transformer \cite{janny2023eagle}. Finally, we include an operator learning method with MeshGraphNets \cite{pfaff2021learningmeshbasedsimulationgraph}. We compare these baselines with \ours{}. 

\paragraph{Datasets.} The first task is to auto-regressively predict the solution to the heat, wave, and Cahn-Hilliard equations on test meshes given an initial condition. These equations are defined formally in \cref{sec:pdes}. Informally, the heat equation describes the dissipation of temperature. The wave equation describes the propagation of oscillations through a medium. The Cahn-Hilliard equation describes phase separation and gives rise to distinct phase interfaces, for example, the separation of oil and water in a mixture. We use the same \pv{} meshes generated in \citet{park2023modelingdynamicsmeshesgauge}. These meshes are highly intricate and test the models' ability to handle nonlinear dynamics on complicated geometries. We use five different initializations for each test mesh. Further dataset and training details are provided in \cref{sec:pdes} and \cref{sec:gegnn}. We will refer to this dataset as MeshPDE in the remainder of the paper. %Sample meshes and initializations can be found in \cref{sec:pdes}. Training details are given in \cref{sec:gegnn}. We will refer to this dataset as MeshPDE in the remainder of the paper.

The second task is weather forecasting using WB$2$ \cite{rasp2024weatherbench}, a widely used benchmark for data-driven global weather forecasting based on historic data. Specifically, we auto-regressively predict future weather conditions on Earth given an initial condition. A formal problem statement is in \cref{sec:train_ps}. We train and evaluate our models on the ERA5 dataset from WB2, which is the curated version of the ERA5 reanalysis data provided by the European Centre for Medium-Range Weather Forecasts (ECMWF) \cite{https://doi.org/10.1002/qj.3803}. We use $1.5$ ($240 \times 120$) degree spatial resolution data with a $6$ hour temporal resolution, consistent with the evaluation performed in WB2. Further details on mesh construction can be found in \cref{sec:earth_dis}. We evaluate on two variables, temperature at pressure level 850 (T850) and geopotential at pressure level 500 (Z500). We take data from 2013-01-01 to 2019-12-31 UTC as training data. We use a smaller subset of the ERA5 data that is commonly used for other large scale data-based weather models due to compute constraints, but remain consistent to WB2 in evaluating on data from 2020-01-01 to 2020-12-31. %Note that unlike many large scale weather forecasting models that predict multiple variables and levels simultaneously, we train one model for each variable and level pair."

%\subsubsection{Evaluation} \label{sec:eval} 

\paragraph{Evaluation.}

For MeshPDE, our metrics include normalized root mean squared error (NRMSE), symmetric mean absolute percentage error (SMAPE), and Rayleigh quotient errors aggregated over all time-steps. In particular, the Rayleigh error (RE) is given by:
\begin{equation*}
\text{RE}(f) = \frac{1}{\mathbf{T_{\rm max}}}\sum_{t} ^{\mathbf{T_{\rm max}}}| R_\mathcal{M}(\mathbf{Y_t}) - R_\mathcal{M}(f(\mathbf{X_t})) |. 
\end{equation*}
Further details on these metrics are provided in \cref{sec:eval_details} and we compare RE with more multiscale and local smoothness metrics in \cref{sec:cf} and \cref{sec:local}. We supplement these metrics with qualitative diagnostic figures in \cref{sec:diagnostics} and with videos on our \href{https://github.com/EdwardBerman/rayleigh_analysis/tree/main/sup_mat}{GitHub repository}, showing in particular that Rayleigh errors are consistent with visual assessments of smoothness. For WB$2$, we report the root mean squared error (RMSE) and the anomaly correlation coefficient (ACC), both latitude weighted as recommended by the benchmark authors. RMSE measures forecast accuracy, while ACC is the Pearson correlation coefficient between forecast anomalies and ground-truth anomalies relative to a climatological baseline. Precise definitions are provided in \cref{sec:wb_deets}. %of the latitude weights, RMSE, ACC, and the climatology computation 

\begin{figure}[!htb]
    \centering
    \includegraphics[width=0.6\linewidth]{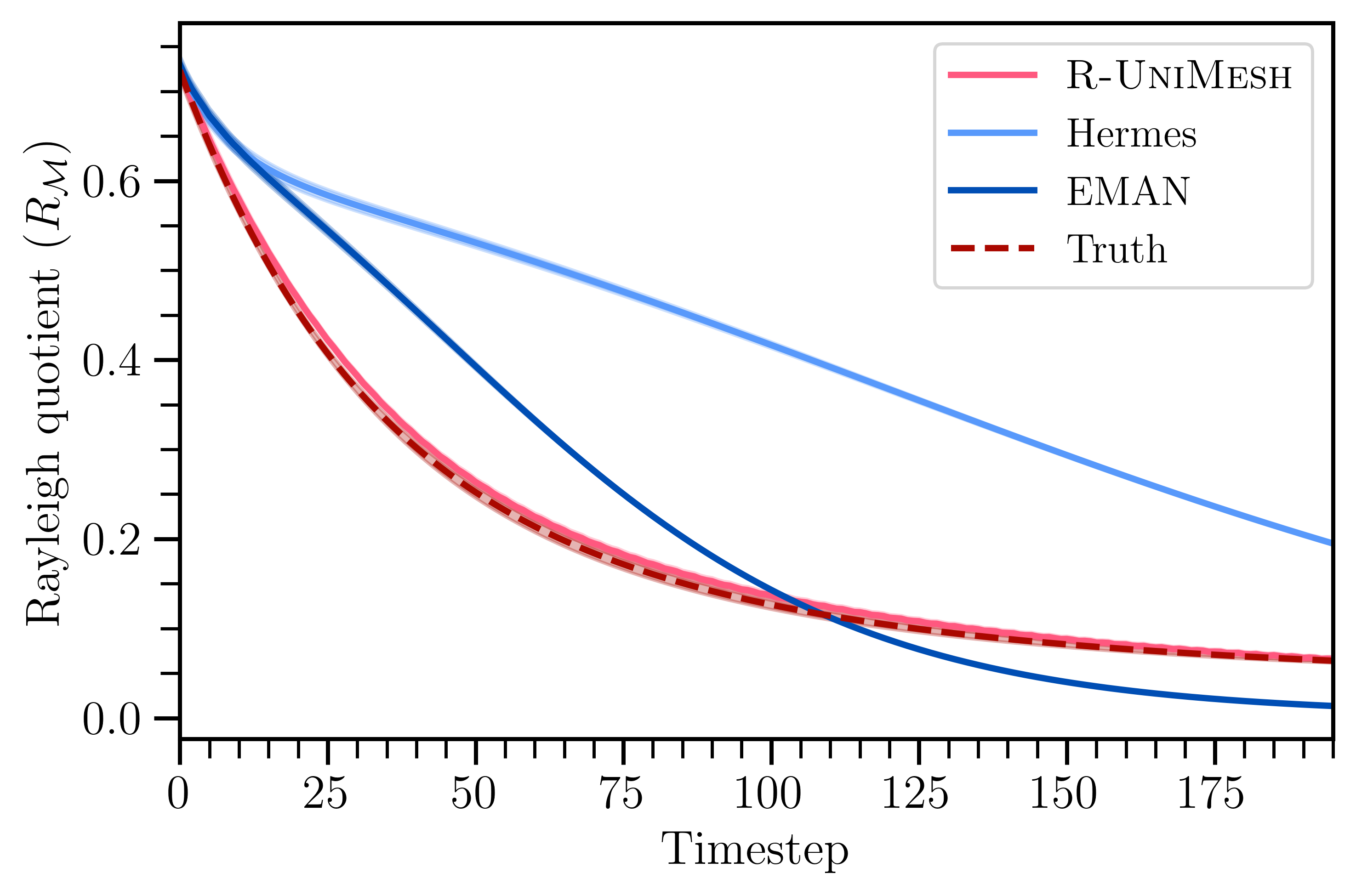}
    \caption{The Rayleigh quotient for each timestep on an unseen mesh for Hermes, EMAN, and \ours{} models. The \ours{} is the best at capturing the true smoothness for heat.}  
    \label{fig:rq_rollout}
\end{figure}

\begin{figure}[!htb]
    \centering
    \includegraphics[width=0.9\linewidth]{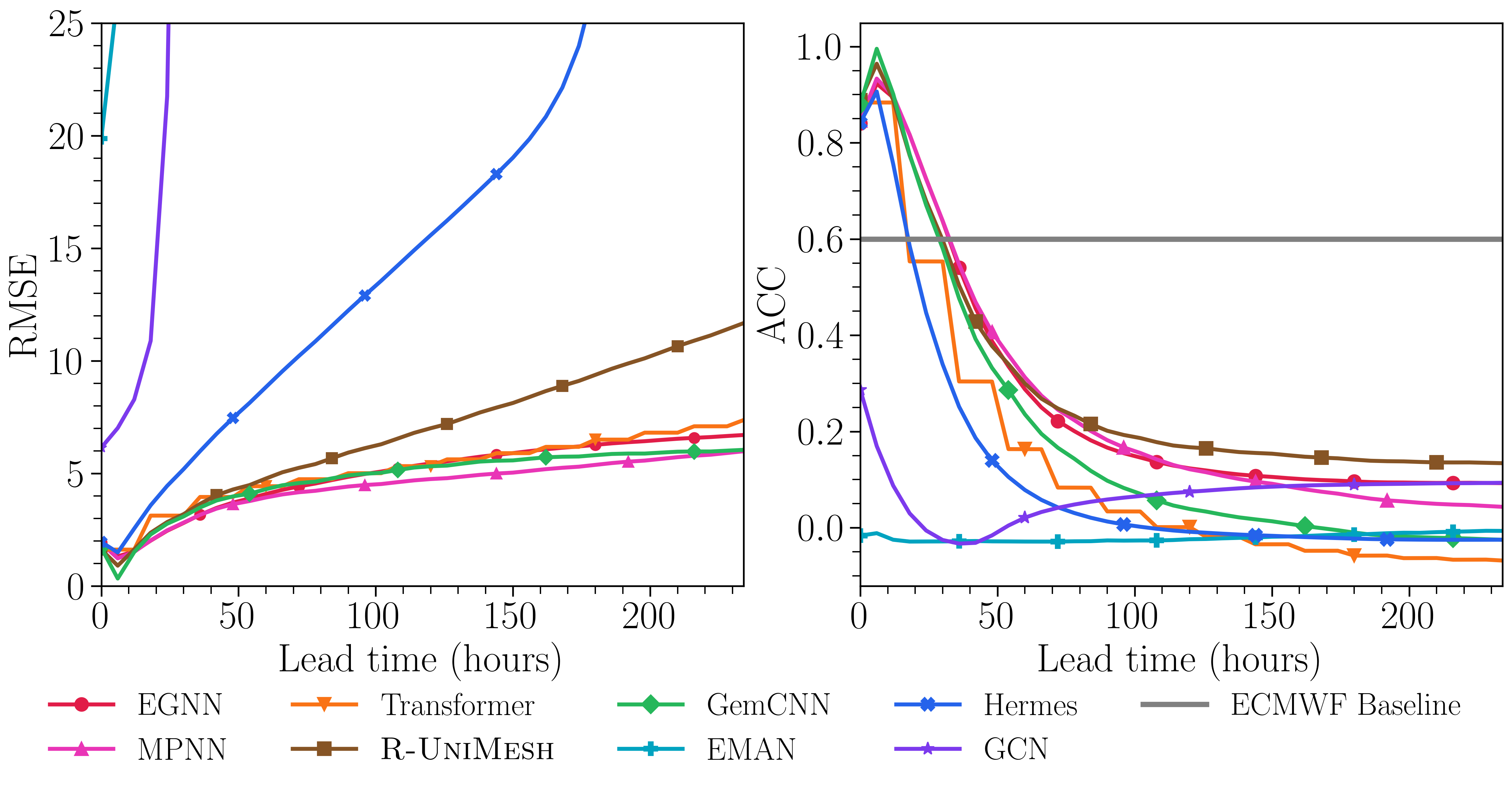}
    \caption{RMSE and ACC as a function of lead time for all models temperature prediction. \ours{} has a competitive RMSE, especially at early lead time. \ours{} also maintains viability for lead times of roughly $2$ days according to the ECMWF baseline. Exact values recorded in \cref{tab:weather_results} (\cref{sec:wb_geo_extended_results}).}  
    \label{fig:wb2}
\end{figure}

%\subsubsection{Results} 
\paragraph{Results.}

Our main result is that our \ours{} model outperforms all baselines at solving highly diffusive PDEs and captures the true smoothness while remaining competitive across all other tasks. On MeshPDE, as shown in \cref{tab:IRE}, \ours{} performs particularly well on heat modeling, where it achieves the lowest NRMSE, SMAPE, and RE. On heat, \ours{} matches the true smoothness almost exactly at every timestep, as seen in \cref{fig:rq_rollout}, illustrating that \ours{} is best at capturing the underlying differential structure of the PDE solution. The convergence and smoothness errors reported by our metrics also agree visually with our qualitative diagnostics in \cref{fig:salesman} and in \cref{sec:diagnostics}. 

Another important finding is that nearly all models are able to perform comparably well on the Cahn-Hilliard equation, where the test mesh (toroid) is simple. The only exceptions are GCN and EGNN, which struggle across all tasks. This suggests that stronger geometric inductive biases such as unitarity or gauge equivariance are necessary for strong performance on unseen meshes with complex geometries. 

In addition to the attention and message-passing based baselines, \ours{} also outperforms operator learning methods on diffusive dynamics. We provide this comparison in \cref{sec:opppp}.

On the real-world WB$2$ benchmark (\cref{fig:wb2}), we see that our best performing models are comparable with the state of the art \citep{rasp2024weatherbench} in RMSE and ACC, coming within a couple of degrees of SOTA even at $10$ lead days. This is despite restricting our training set size and model parameters due to compute limitations. We show in \cref{sec:wb_geo_extended_results} that \ours{} is among the best models for the geopotential prediction task, though all models are below the SOTA in \citet{rasp2024weatherbench}. Our results also support our earlier finding that geometric inductive biases matter for complex geometries: the equivariant and unitary models show no significant advantages in this setting, where there is no cross mesh generalization. Finally, we show in \cref{sec:wb_smooth_extended_results} that \ours{} is competitive in terms of RE on both WB$2$ variables.

\paragraph{Ablation.} Given that \ours{} is not just relaxed unitary convolution and contains other various architectural considerations such as the choice of readout layers in $B$ and choice of activation function, we conduct thorough ablation study in \cref{sec:abl} in order to isolate the effect of the proposed relaxation mechanism and find that it is necessary for faithful modeling. Most importantly, we show that replacing the unitary convolution layers with ordinary graph convolutions causes a significant detriment to performance on diffusive dynamics.

% Our results on WB$2$ (\cref{fig:wb2}) further support this finding: the equivariant and unitary models show no significant advantages in this setting, where there is no cross mesh generalization. We also note that, despite restricting our training set size and model parameters due to compute limitations, our best performing models in \cref{fig:wb2} are comparable with the state of the art \citep{rasp2024weatherbench} according to RMSE and ACC. In terms of RMSE, the best performing models are within a couple of degrees of SOTA even at $10$ lead days. We show in \cref{sec:wb_geo_extended_results} that \ours{} is among the best models for the geopotential prediction task, though all models are below the SOTA in \citet{rasp2024weatherbench}. Finally, we show in \cref{sec:wb_smooth_extended_results} that \ours{} is competitive in terms of RE on both WB$2$ variables.

\section{Discussion}

\paragraph{Limitations.} We note the following limitations in our study. On MeshPDE, our results are strongest for the heat equation, whereas the performance gains are more mixed for wave and Cahn-Hilliard. For wave, we hypothesize that this is because the Rayleigh quotient for wave is highly oscillatory in time. This causes unitarity to be over constraining and our model is therefore more heavily reliant on the unconstrained component to approximate the true dynamics. For Cahn-Hilliard, we hypothesize that our mixed results are because \textit{global} smoothness metrics do not fully capture the underlying dynamics. In particular, the Cahn-Hilliard solution contains \textit{local} and sharp discontinuities in the PDE solution.  We note that our experiments on WB2 are reduced in scale and are not intended as a serious claim to SOTA weather forecasting performance. Rather, we use this application as a testbed to study the behavior and inductive biases of our models. On the theoretical side, our bound in \cref{thm:uni_lb} requires knowledge of the data distribution and the ability to write the true function $f$ in angular coordinates in order to be computed on real datasets. 

\paragraph{Conclusion.} Our work clarifies the approximation limits of smoothness preserving (unitary) functions and unitary convolution networks and shows how constraint relaxations can aid performance on various dynamics modeling tasks on graphs and meshes. We contribute \oursGraph{} and \ours{}, which provide SOTA performance on diffusive dynamics problems and excel at capturing the true smoothness of the system. Future work will explore using approximately unitary networks for solving PDEs under partial observability by using them as backbones for generative models. Additionally, incorporating boundary conditions into the PDE-solving framework would significantly broaden applicability in engineering settings. Finally, an interesting direction involves the usage of a dynamic $\mathbf{T}_{\max}$ in the Taylor-truncation method, which could yield improvements over a static $\mathbf{T}_{\max}$ in systems with oscillatory smoothness, such as the wave equation. 

%Future work will explore using approximately unitary networks for solving PDEs under partial observability by using them as backbones for generative models.

\clearpage

\section*{Acknowledgements}

E.B. thanks \href{https://melanie-weber.com/}{Melanie Weber}, \href{https://lfesser97.github.io/}{Lukas Fesser}, \href{https://bkiani.github.io/}{Bobak Kiani}, and members of the \href{https://weber.seas.harvard.edu/}{Geometric Machine Learning Group} for helpful discussions. L.L. was supported by a Northeastern University Undergraduate Research and Fellowships PEAK Experiences Award. R.W. would like to acknowledge support from NSF Grants $2442658$ and $2134178$. This work is supported by the National Science Foundation under Cooperative Agreement PHY-$2019786$ (The NSF AI Institute for Artificial Intelligence and Fundamental Interactions, \url{http://iaifi.org/}). We also thank Northeastern Research Computing for GPU access through the Explorer cluster. This paper was reviewed on OpenReview at \href{https://openreview.net/forum?id=Uf3SMkP2Wd}{https://openreview.net/forum?id=Uf3SMkP2Wd}.

\section*{Impact Statement}

This paper presents work whose goal is to advance the field of Machine Learning. There are many potential societal consequences of our work, none which we feel must be specifically highlighted here.

% \section LLM Usage
% LLMs were used for small gramatical help and coding assitance. Main ideas, experiments, and writing were contributed by the authors.

\newpage

\bibliography{example_paper}

@article{bronstein2021geometric,
  title={Geometric deep learning: Grids, groups, graphs, geodesics, and gauges},
  author={Bronstein, Michael M and Bruna, Joan and Cohen, Taco and Veli{\v{c}}kovi{\'c}, Petar},
  journal={arXiv preprint arXiv:2104.13478},
  year={2021}
}

@article{kiani2024unitary,
  title={Unitary convolutions for learning on graphs and groups},
  author={Kiani, Bobak and Fesser, Lukas and Weber, Melanie},
  journal={Advances in Neural Information Processing Systems},
  volume={37},
  pages={136922--136961},
  year={2024}
}

@article{rusch2023survey,
  title={A survey on oversmoothing in graph neural networks},
  author={Rusch, T Konstantin and Bronstein, Michael M and Mishra, Siddhartha},
  journal={arXiv preprint arXiv:2303.10993},
  year={2023}
}

@article{cai2020note,
  title={A note on over-smoothing for graph neural networks},
  author={Cai, Chen and Wang, Yusu},
  journal={arXiv preprint arXiv:2006.13318},
  year={2020}
}

@article{bodnar2022neural,
  title={Neural sheaf diffusion: A topological perspective on heterophily and oversmoothing in gnns},
  author={Bodnar, Cristian and Di Giovanni, Francesco and Chamberlain, Benjamin and Lio, Pietro and Bronstein, Michael},
  journal={Advances in Neural Information Processing Systems},
  volume={35},
  pages={18527--18541},
  year={2022}
}

@article{keriven2022not,
  title={Not too little, not too much: a theoretical analysis of graph (over) smoothing},
  author={Keriven, Nicolas},
  journal={Advances in Neural Information Processing Systems},
  volume={35},
  pages={2268--2281},
  year={2022}
}

@incollection{hall2013lie,
  title={Lie groups, Lie algebras, and representations},
  author={Hall, Brian C},
  booktitle={Quantum Theory for Mathematicians},
  pages={333--366},
  year={2013},
  publisher={Springer}
}

@inproceedings{
gruver2024liederivativemeasuringlearned,
title={The Lie Derivative for Measuring Learned Equivariance},
author={Nate Gruver and Marc Anton Finzi and Micah Goldblum and Andrew Gordon Wilson},
booktitle={The Eleventh International Conference on Learning Representations },
year={2023},
url={https://openreview.net/forum?id=JL7Va5Vy15J}
}

@article{wang2023general,
  title={A general theory of correct, incorrect, and extrinsic equivariance},
  author={Wang, Dian and Zhu, Xupeng and Park, Jung Yeon and Jia, Mingxi and Su, Guanang and Platt, Robert and Walters, Robin},
  journal={Advances in Neural Information Processing Systems},
  volume={36},
  pages={40006--40029},
  year={2023}
}

@inproceedings{
arroyo2025vanishing,
title={On Vanishing Gradients, Over-Smoothing, and Over-Squashing in {GNN}s: Bridging Recurrent and Graph Learning},
author={Alvaro Arroyo and Alessio Gravina and Benjamin Gutteridge and Federico Barbero and Claudio Gallicchio and Xiaowen Dong and Michael M. Bronstein and Pierre Vandergheynst},
booktitle={The Thirty-ninth Annual Conference on Neural Information Processing Systems},
year={2025},
url={https://openreview.net/forum?id=N4cyRMuLyl}
}

@inproceedings{xiao2018dynamical,
  title={Dynamical isometry and a mean field theory of cnns: How to train 10,000-layer vanilla convolutional neural networks},
  author={Xiao, Lechao and Bahri, Yasaman and Sohl-Dickstein, Jascha and Schoenholz, Samuel and Pennington, Jeffrey},
  booktitle={International conference on machine learning},
  pages={5393--5402},
  year={2018},
  organization={PMLR}
}

@inproceedings{
kipf2017semisupervised,
title={Semi-Supervised Classification with Graph Convolutional Networks},
author={Thomas N. Kipf and Max Welling},
booktitle={International Conference on Learning Representations},
year={2017},
url={https://openreview.net/forum?id=SJU4ayYgl}
}

@book{chung1997spectral,
  title={Spectral graph theory},
  author={Chung, Fan RK},
  volume={92},
  year={1997},
  publisher={American Mathematical Soc.}
}

@inproceedings{li2018deeper,
  title={Deeper insights into graph convolutional networks for semi-supervised learning},
  author={Li, Qimai and Han, Zhichao and Wu, Xiao-Ming},
  booktitle={Proceedings of the AAAI conference on artificial intelligence},
  volume={32},
  number={1},
  year={2018}
}

@article{leman1968reduction,
  title={A reduction of a graph to a canonical form and an algebra arising during this reduction},
  author={Leman, Andrei and Weisfeiler, Boris},
  journal={Nauchno-Technicheskaya Informatsiya},
  volume={2},
  number={9},
  pages={12--16},
  year={1968}
}

@misc{pygsp,
  title = {PyGSP: Graph Signal Processing in Python},
  author = {Defferrard, Micha\"el and Martin, Lionel and Pena, Rodrigo and Perraudin, Nathana\"el},
  doi = {10.5281/zenodo.1003157},
  url = {https://github.com/epfl-lts2/pygsp/},
  year={2017}
}

@inproceedings{
wang2020incorporating,
title={Incorporating Symmetry into Deep Dynamics Models for Improved Generalization},
author={Rui Wang and Robin Walters and Rose Yu},
booktitle={International Conference on Learning Representations},
year={2021},
url={https://openreview.net/forum?id=wta_8Hx2KD}
}

@book{olver1993applications,
  title={Applications of Lie groups to differential equations},
  author={Olver, Peter J},
  volume={107},
  year={1993},
  publisher={Springer Science \& Business Media}
}

@inproceedings{10.1145/3394486.3403198,
author = {Wang, Rui and Kashinath, Karthik and Mustafa, Mustafa and Albert, Adrian and Yu, Rose},
title = {Towards Physics-informed Deep Learning for Turbulent Flow Prediction},
year = {2020},
isbn = {9781450379984},
publisher = {Association for Computing Machinery},
address = {New York, NY, USA},
url = {https://doi.org/10.1145/3394486.3403198},
doi = {10.1145/3394486.3403198},
booktitle = {Proceedings of the 26th ACM SIGKDD International Conference on Knowledge Discovery \& Data Mining},
pages = {1457–1466},
numpages = {10},
location = {Virtual Event, CA, USA},
series = {KDD '20}
}

@article{cai2021physics,
  title={Physics-informed neural networks for heat transfer problems},
  author={Cai, Shengze and Wang, Zhicheng and Wang, Sifan and Perdikaris, Paris and Karniadakis, George Em},
  journal={Journal of Heat Transfer},
  volume={143},
  number={6},
  pages={060801},
  year={2021},
  publisher={American Society of Mechanical Engineers}
}

@article{maurizi2022predicting,
  title={Predicting stress, strain and deformation fields in materials and structures with graph neural networks},
  author={Maurizi, Marco and Gao, Chao and Berto, Filippo},
  journal={Scientific reports},
  volume={12},
  number={1},
  pages={21834},
  year={2022},
  publisher={Nature Publishing Group UK London}
}

@inproceedings{li2020neural,
  title={Neural operator: Graph kernel network for partial differential equations},
  author={Anandkumar, Anima and Azizzadenesheli, Kamyar and Bhattacharya, Kaushik and Kovachki, Nikola and Li, Zongyi and Liu, Burigede and Stuart, Andrew},
  booktitle={ICLR 2020 workshop on integration of deep neural models and differential equations},
  year={2020}
}

@article{park2023modelingdynamicsmeshesgauge,
  title={Modeling dynamics over meshes with gauge equivariant nonlinear message passing},
  author={Park, Jung Yeon and Wong, Lawson and Walters, Robin},
  journal={Advances in Neural Information Processing Systems},
  volume={36},
  pages={15277--15302},
  year={2023}
}

@inproceedings{
dehaan2021gaugeequivariantmeshcnns,
title={Gauge Equivariant Mesh CNNs: Anisotropic convolutions on geometric graphs},
author = {{de Haan}, Pim and Weiler, Maurice and Cohen, Taco and Welling, Max},
booktitle={International Conference on Learning Representations},
year={2021},
url={https://openreview.net/forum?id=Jnspzp-oIZE}
}

@inproceedings{cohen2019gaugeequivariantconvolutionalnetworks,
  title={Gauge equivariant convolutional networks and the icosahedral CNN},
  author={Cohen, Taco and Weiler, Maurice and Kicanaoglu, Berkay and Welling, Max},
  booktitle={International conference on Machine learning},
  pages={1321--1330},
  year={2019},
  organization={PMLR}
}

@inproceedings{satorras2022enequivariantgraphneural,
  title={E (n) equivariant graph neural networks},
  author={Satorras, V{\i}ctor Garcia and Hoogeboom, Emiel and Welling, Max},
  booktitle={International conference on machine learning},
  pages={9323--9332},
  year={2021},
  organization={PMLR}
}

@inproceedings{
pfaff2021learningmeshbasedsimulationgraph,
title={Learning Mesh-Based Simulation with Graph Networks},
author={Tobias Pfaff and Meire Fortunato and Alvaro Sanchez-Gonzalez and Peter Battaglia},
booktitle={International Conference on Learning Representations},
year={2021},
url={https://openreview.net/forum?id=roNqYL0_XP}
}

@article{
basu2022equivariant,
title={Equivariant Mesh Attention Networks},
author={Sourya Basu and Jose Gallego-Posada and Francesco Vigan{\`o} and James Rowbottom and Taco Cohen},
journal={Transactions on Machine Learning Research},
issn={2835-8856},
year={2022},
url={https://openreview.net/forum?id=3IqqJh2Ycy},
note={Expert Certification}
}

@inproceedings{borovitskiy2021matern,
  title={Mat{\'e}rn Gaussian processes on graphs},
  author={Borovitskiy, Viacheslav and Azangulov, Iskander and Terenin, Alexander and Mostowsky, Peter and Deisenroth, Marc and Durrande, Nicolas},
  booktitle={International Conference on Artificial Intelligence and Statistics},
  pages={2593--2601},
  year={2021},
  organization={PMLR}
}

@article{daniels2025uncertainty,
  title={Uncertainty-Aware Diagnostics for Physics-Informed Machine Learning},
  author={Daniels, Mara and Hodgkinson, Liam and Mahoney, Michael},
  journal={arXiv preprint arXiv:2510.26121},
  year={2025}
}

@inproceedings{suinterplay,
  title={On the Interplay between Graph Structure and Learning Algorithms in Graph Neural Networks},
  author={Su, Junwei and Wu, Chuan},
  booktitle={Forty-second International Conference on Machine Learning},
year={2025}
}

@article{sullivan2019pyvista,
  title={PyVista: 3D plotting and mesh analysis through a streamlined interface for the Visualization Toolkit (VTK)},
  author={Sullivan, C and Kaszynski, Alexander},
  journal={Journal of Open Source Software},
  volume={4},
  number={37},
  pages={1450},
  year={2019},
  publisher={The Open Journal}
}

@article{reuter2009discrete,
  title={Discrete Laplace--Beltrami operators for shape analysis and segmentation},
  author={Reuter, Martin and Biasotti, Silvia and Giorgi, Daniela and Patan{\`e}, Giuseppe and Spagnuolo, Michela},
  journal={Computers \& Graphics},
  volume={33},
  number={3},
  pages={381--390},
  year={2009},
  publisher={Elsevier}
}

@article{ferrandi2022homogeneous,
  title={A homogeneous Rayleigh quotient with applications in gradient methods},
  author={Ferrandi, Giulia and Hochstenbach, Michiel E},
  journal={Journal of Computational and Applied Mathematics},
  volume={437},
  pages={115440},
  year={2024},
  publisher={Elsevier}
}

@article{pandya2025iaemu,
  title        = {IAEmu: Learning {Galaxy} {Intrinsic} {Alignment} {Correlations}},
  author       = {Pandya, Sneh and Yang, Yuanyuan and Van Alfen, Nicholas and Blazek, Jonathan and Walters, Robin},
  journal      = {The Open Journal of Astrophysics},
  volume       = {8},
  year         = {2025},
  month        = {dec 2},
  doi          = {10.33232/001c.151749},
  publisher    = {Maynooth Academic Publishing}
}

@inproceedings{
dong2023rayleigh,
title={Rayleigh Quotient Graph Neural Networks for Graph-level Anomaly Detection},
author={Xiangyu Dong and Xingyi Zhang and Sibo Wang},
booktitle={The Twelfth International Conference on Learning Representations},
year={2024},
url={https://openreview.net/forum?id=4UIBysXjVq}
}

@article{rowan2025solving,
  title={Solving engineering eigenvalue problems with neural networks using the Rayleigh quotient},
  author={Rowan, Conor and Doostan, Alireza and Maute, Kurt and Evans, John},
  journal={International Journal for Numerical Methods in Engineering},
  volume={126},
  number={24},
  pages={e70209},
  year={2025},
  publisher={Wiley Online Library}
}

@article{kashinath2021physics,
  title={Physics-informed machine learning: case studies for weather and climate modelling},
  author={Mustafa, M and Wu, JL and Jiang, C and Wang, R and others},
  journal={Philosophical Transactions of the Royal Society A},
  volume={379},
  number={2194},
  pages={20200093},
  year={2021},
  publisher={The Royal Society Publishing}
}

@inproceedings{sharp2020laplacian,
  title={A laplacian for nonmanifold triangle meshes},
  author={Sharp, Nicholas and Crane, Keenan},
  booktitle={Computer Graphics Forum},
  volume={39},
  number={5},
  pages={69--80},
  year={2020},
  organization={Wiley Online Library}
}

@misc{
shao2024unifying,
title={Unifying over-smoothing and over-squashing in graph neural networks: A physics informed approach and beyond},
author={Zhiqi Shao and Dai Shi and Andi Han and Yi Guo and Qibin Zhao and Junbin Gao},
year={2024},
url={https://openreview.net/forum?id=swPf2hwKl8}
}

@inproceedings{
kulick2025investigating,
title={Investigating Zero-Shot Size Transfer of Graph Neural Differential Equations for Learning Graph Diffusion Dynamics},
author={Charles Kulick and Bj{\"o}rn Birnir and Sui Tang},
booktitle={Topology, Algebra, and Geometry in Data Science},
year={2025},
url={https://openreview.net/forum?id=qgbyLknKXy}
}

@article{marisca2025over,
  title={Over-squashing in Spatiotemporal Graph Neural Networks},
  author={Marisca, Ivan and Bamberger, Jacob and Alippi, Cesare and Bronstein, Michael M},
  journal={arXiv preprint arXiv:2506.15507},
  year={2025}
}

@article{suk2024mesh,
  title={Mesh neural networks for SE (3)-equivariant hemodynamics estimation on the artery wall},
  author={Suk, Julian and de Haan, Pim and Lippe, Phillip and Brune, Christoph and Wolterink, Jelmer M},
  journal={Computers in biology and medicine},
  volume={173},
  pages={108328},
  year={2024},
  publisher={Elsevier}
}

@inproceedings{
janny2023eagle,
title={{EAGLE}: Large-scale Learning of Turbulent Fluid Dynamics with Mesh Transformers},
author={Steeven Janny and Aur{\'e}lien B{\'e}n{\'e}teau and Madiha Nadri and Julie Digne and Nicolas Thome and Christian Wolf},
booktitle={The Eleventh International Conference on Learning Representations },
year={2023},
url={https://openreview.net/forum?id=mfIX4QpsARJ}
}

@article{rasp2024weatherbench,
  title={WeatherBench 2: A benchmark for the next generation of data-driven global weather models},
  author={Rasp, Stephan and Hoyer, Stephan and Merose, Alexander and Langmore, Ian and Battaglia, Peter and Russell, Tyler and Sanchez-Gonzalez, Alvaro and Yang, Vivian and Carver, Rob and Agrawal, Shreya and others},
  journal={Journal of Advances in Modeling Earth Systems},
  volume={16},
  number={6},
  pages={e2023MS004019},
  year={2024},
  publisher={Wiley Online Library}
}

@inproceedings{wang2022approximately,
  title={Approximately equivariant networks for imperfectly symmetric dynamics},
  author={Wang, Rui and Walters, Robin and Yu, Rose},
  booktitle={International Conference on Machine Learning},
  pages={23078--23091},
  year={2022},
  organization={PMLR}
}

@inproceedings{gao2025discretization,
  title={Discretization-invariance? on the discretization mismatch errors in neural operators},
  author={Gao, Wenhan and Xu, Ruichen and Deng, Yuefan and Liu, Yi},
  booktitle={The Thirteenth International Conference on Learning Representations},
  year={2025}
}

@article{gupta2023generalized,
  title={Generalized neural closure models with interpretability},
  author={Gupta, Abhinav and Lermusiaux, Pierre FJ},
  journal={Scientific Reports},
  volume={13},
  number={1},
  pages={10634},
  year={2023},
  publisher={Nature Publishing Group UK London}
}

@inproceedings{
li2021fourier,
title={Fourier Neural Operator for Parametric Partial Differential Equations},
author={Zongyi Li and Nikola Borislavov Kovachki and Kamyar Azizzadenesheli and Burigede liu and Kaushik Bhattacharya and Andrew Stuart and Anima Anandkumar},
booktitle={International Conference on Learning Representations},
year={2021},
url={https://openreview.net/forum?id=c8P9NQVtmnO}
}

@inproceedings{
schlaginhaufen2021learning,
title={Learning Stable Deep Dynamics Models for Partially Observed or Delayed Dynamical Systems},
author={Andreas Schlaginhaufen and Philippe Wenk and Andreas Krause and Florian D{\"o}rfler},
booktitle={Advances in Neural Information Processing Systems},
editor={A. Beygelzimer and Y. Dauphin and P. Liang and J. Wortman Vaughan},
year={2021},
url={https://openreview.net/forum?id=u8HmtBBSVJS}
}

@incollection{schneider2006weak,
  title={Weak gravitational lensing},
  author={Schneider, Peter},
  booktitle={Gravitational lensing: strong, weak and micro},
  pages={269--451},
  year={2006},
  publisher={Springer}
}

@article{jarvis2004skewness,
  title={The skewness of the aperture mass statistic},
  author={Jarvis, Mike and Bernstein, Gary and Jain, Bhuvnesh},
  journal={Monthly Notices of the Royal Astronomical Society},
  volume={352},
  number={1},
  pages={338--352},
  year={2004},
  publisher={Blackwell Science Ltd Oxford, UK}
}

@article{cahn1958free,
  title={Free energy of a nonuniform system. I. Interfacial free energy},
  author={Cahn, John W and Hilliard, John E},
  journal={The Journal of chemical physics},
  volume={28},
  number={2},
  pages={258--267},
  year={1958},
  publisher={American Institute of Physics}
}

@inproceedings{g2017,
  title={Neural message passing for quantum chemistry},
  author={Gilmer, Justin and Schoenholz, Samuel S and Riley, Patrick F and Vinyals, Oriol and Dahl, George E},
  booktitle={International conference on machine learning},
  pages={1263--1272},
  year={2017},
  organization={Pmlr}
}

@misc{nvidia2025h200,
  title={NVIDIA H200 Tensor Core GPU Datasheet},
  author={NVIDIA},
  year={2025},
  note={Retrieved from NVIDIA website},
  url={https://resources.nvidia.com/en-us-data-center-overview-mc/en-us-data-center-overview/hpc-datasheet-sc23-h200}
}

@article{huang2018robust,
  title={Robust watertight manifold surface generation method for shapenet models},
  author={Huang, Jingwei and Su, Hao and Guibas, Leonidas},
  journal={arXiv preprint arXiv:1802.01698},
  year={2018}
}

@article{li2024tutorials,
  title={Tutorials: Physics-informed machine learning methods of computing 1D phase-field models},
  author={Li, Wei and Fang, Ruqing and Jiao, Junning and Vassilakis, Georgios N and Zhu, Juner},
  journal={APL Machine Learning},
  volume={2},
  number={3},
  year={2024},
  publisher={AIP Publishing}
}

@book{constantin1988navier,
  title={Navier-stokes equations},
  author={Constantin, Peter and Foia{\c{s}}, Ciprian},
  year={1988},
  publisher={University of Chicago press}
}

@incollection{ghil2020geophysical,
  title={Geophysical fluid dynamics, nonautonomous dynamical systems, and the climate sciences},
  author={Ghil, Michael and Simonnet, Eric},
  booktitle={Mathematical Approach to Climate Change and its Impacts: MAC2I},
  pages={3--81},
  year={2020},
  publisher={Springer}
}

@inproceedings{tonshoffdid,
  title={Where Did the Gap Go? Reassessing the Long-Range Graph Benchmark},
  author={T{\"o}nshoff, Jan and Ritzert, Martin and Rosenbluth, Eran and Grohe, Martin},
  booktitle={The Second Learning on Graphs Conference},
  year={2023}
}

@article{tonshoffdid_tmlr,
  title={Where Did the Gap Go? Reassessing the Long-Range Graph Benchmark},
  author={T{\"o}nshoff, Jan and Ritzert, Martin and Rosenbluth, Eran and Grohe, Martin},
  journal={Transactions on Machine Learning Research},
  year={2024}
}

@inproceedings{mitchel2021field,
  title={Field convolutions for surface cnns},
  author={Mitchel, Thomas W and Kim, Vladimir G and Kazhdan, Michael},
  booktitle={Proceedings of the IEEE/CVF International Conference on Computer Vision},
  pages={10001--10011},
  year={2021}
}

@book{Artin1998,
  author = {Artin, Michael},
  publisher = {Birkh{\"a}user},
  title = {Algebra},
  username = {algebradresden},
  year = 1998
}

@article{esteves2020theoretical,
  title={Theoretical aspects of group equivariant neural networks},
  author={Esteves, Carlos},
  journal={arXiv preprint arXiv:2004.05154},
  year={2020}
}

@article{
doi:10.1073/pnas.2311808121,
author = {Rose Yu  and Rui Wang },
title = {Learning dynamical systems from data: An introduction to physics-guided deep learning},
journal = {Proceedings of the National Academy of Sciences},
volume = {121},
number = {27},
pages = {e2311808121},
year = {2024},
doi = {10.1073/pnas.2311808121},
URL = {https://www.pnas.org/doi/abs/10.1073/pnas.2311808121},
eprint = {https://www.pnas.org/doi/pdf/10.1073/pnas.2311808121}}

@article{https://doi.org/10.1002/qj.3803,
author = {Hersbach, Hans and Bell, Bill and Berrisford, Paul and Hirahara, Shoji and Horányi, András and Muñoz-Sabater, Joaquín and Nicolas, Julien and Peubey, Carole and Radu, Raluca and Schepers, Dinand and Simmons, Adrian and Soci, Cornel and Abdalla, Saleh and Abellan, Xavier and Balsamo, Gianpaolo and Bechtold, Peter and Biavati, Gionata and Bidlot, Jean and Bonavita, Massimo and De Chiara, Giovanna and Dahlgren, Per and Dee, Dick and Diamantakis, Michail and Dragani, Rossana and Flemming, Johannes and Forbes, Richard and Fuentes, Manuel and Geer, Alan and Haimberger, Leo and Healy, Sean and Hogan, Robin J. and Hólm, Elías and Janisková, Marta and Keeley, Sarah and Laloyaux, Patrick and Lopez, Philippe and Lupu, Cristina and Radnoti, Gabor and de Rosnay, Patricia and Rozum, Iryna and Vamborg, Freja and Villaume, Sebastien and Thépaut, Jean-Noël},
title = {The ERA5 global reanalysis},
journal = {Quarterly Journal of the Royal Meteorological Society},
volume = {146},
number = {730},
pages = {1999-2049},
doi = {https://doi.org/10.1002/qj.3803},
url = {https://rmets.onlinelibrary.wiley.com/doi/abs/10.1002/qj.3803},
eprint = {https://rmets.onlinelibrary.wiley.com/doi/pdf/10.1002/qj.3803},
year = {2020}
}

@misc{lam2023graphcastlearningskillfulmediumrange,
      title={GraphCast: Learning skillful medium-range global weather forecasting}, 
      author={Remi Lam and Alvaro Sanchez-Gonzalez and Matthew Willson and Peter Wirnsberger and Meire Fortunato and Ferran Alet and Suman Ravuri and Timo Ewalds and Zach Eaton-Rosen and Weihua Hu and Alexander Merose and Stephan Hoyer and George Holland and Oriol Vinyals and Jacklynn Stott and Alexander Pritzel and Shakir Mohamed and Peter Battaglia},
      year={2023},
      eprint={2212.12794},
      archivePrefix={arXiv},
      primaryClass={cs.LG},
      url={https://arxiv.org/abs/2212.12794}, 
}

@inproceedings{anil2019sorting,
  title={Sorting out Lipschitz function approximation},
  author={Anil, Cem and Lucas, James and Grosse, Roger},
  booktitle={International conference on machine learning},
  pages={291--301},
  year={2019},
  organization={PMLR}
}

@article{huang2024diffusionpde,
  title={DiffusionPDE: Generative PDE-solving under partial observation},
  author={Huang, Jiahe and Yang, Guandao and Wang, Zichen and Park, Jeong Joon},
  journal={Advances in Neural Information Processing Systems},
  volume={37},
  pages={130291--130323},
  year={2024}
}

@inproceedings{morelpredicting,
  title={Predicting partially observable dynamical systems via diffusion models with a multiscale inference scheme},
  author={Morel, Rudy and Ramunno, Francesco Pio and Shen, Jeff and Bietti, Alberto and Cho, Kyunghyun and Cranmer, Miles and Golkar, Siavash and GUGNIN, OLEXANDR and Krawezik, Geraud and Marwah, Tanya and others},
  booktitle={The Thirty-ninth Annual Conference on Neural Information Processing Systems},
  year={2025}
}

@inproceedings{cranmer2020lagrangian,
  title={Lagrangian Neural Networks},
  author={Cranmer, Miles and Greydanus, Sam and Hoyer, Stephan and Battaglia, Peter and Spergel, David and Ho, Shirley},
  booktitle={ICLR 2020 Workshop on Integration of Deep Neural Models and Differential Equations},
  year={2020}
}

@inproceedings{liukan,
  title={KAN: Kolmogorov--Arnold Networks},
  author={Liu, Ziming and Wang, Yixuan and Vaidya, Sachin and Ruehle, Fabian and Halverson, James and Soljacic, Marin and Hou, Thomas Y and Tegmark, Max},
  booktitle={The Thirteenth International Conference on Learning Representations},
  year={2024}
}

@article{daniels2025splat,
  title={Splat Regression Models},
  author={Daniels, Mara and Rigollet, Philippe},
  journal={arXiv preprint arXiv:2511.14042},
  year={2025}
}

@article{wang2022data,
  title={Data augmentation vs. equivariant networks: A theory of generalization on dynamics forecasting},
  author={Wang, Rui and Walters, Robin and Yu, Rose},
  journal={International Conference on Machine Learning (ICML) Principles of Distribution Shift Workshop},
  year={2022}
}

@book{baron1822theorie,
  title={Th{\'e}orie analytique de la chaleur},
  author={baron de Fourier, Jean Baptiste Joseph},
  year={1822},
  publisher={Firmin Didot}
}

@misc{dalembert1747wave,
  author       = {d'Alembert, Jean le Rond},
  title        = {Recherches sur la courbe que forme une corde tendu{\"e} mise en vibration},
  year         = {1747},
  note         = {Presented at the ORESME Reading Group Meeting, September 30, 2017},
  howpublished = {\url{https://www.exhibit.xavier.edu/oresme_2017Sept/4}},
}

@article{cui2024transformer,
  title={Transformer based deep learning accelerated numerical simulation for incompressible flow},
  author={Cui, Qingjie and Zhang, Meina and Xiao, Min and Ni, Guoxi},
  journal={Physics of Fluids},
  volume={36},
  number={12},
  year={2024},
  publisher={AIP Publishing}
}

@inproceedings{mitchel2022mobius,
  title={M{\"o}bius convolutions for spherical cnns},
  author={Mitchel, Thomas W and Aigerman, Noam and Kim, Vladimir G and Kazhdan, Michael},
  booktitle={ACM SIGGRAPH 2022 Conference Proceedings},
  pages={1--9},
  year={2022}
}

@inproceedings{mitchel2024single,
  title={Single mesh diffusion models with field latents for texture generation},
  author={Mitchel, Thomas W and Esteves, Carlos and Makadia, Ameesh},
  booktitle={Proceedings of the IEEE/CVF Conference on Computer Vision and Pattern Recognition},
  pages={7953--7963},
  year={2024}
}

@article{bobenko2007discrete,
  title={A discrete Laplace--Beltrami operator for simplicial surfaces},
  author={Bobenko, Alexander I and Springborn, Boris A},
  journal={Discrete \& Computational Geometry},
  volume={38},
  number={4},
  pages={740--756},
  year={2007},
  publisher={Springer}
}

@inproceedings{Crane2013DGP,
  author    = {Crane, Keenan and de Goes, Fernando and Desbrun, Mathieu and Schr{\"o}der, Peter},
  title     = {Digital Geometry Processing with Discrete Exterior Calculus},
  booktitle = {ACM SIGGRAPH 2013 Courses},
  year      = {2013},
  location  = {Anaheim, California},
  publisher = {ACM},
  address   = {New York, NY, USA},
  numpages  = {126}
}

@article{crane2017heat,
  title={The heat method for distance computation},
  author={Crane, Keenan and Weischedel, Clarisse and Wardetzky, Max},
  journal={Communications of the ACM},
  volume={60},
  number={11},
  pages={90--99},
  year={2017},
  publisher={ACM New York, NY, USA}
}

@incollection{meyer2003discrete,
  title={Discrete differential-geometry operators for triangulated 2-manifolds},
  author={Meyer, Mark and Desbrun, Mathieu and Schr{\"o}der, Peter and Barr, Alan H},
  booktitle={Visualization and mathematics III},
  pages={35--57},
  year={2003},
  publisher={Springer}
}

@inproceedings{balduzzi2017shattered,
  title={The shattered gradients problem: If resnets are the answer, then what is the question?},
  author={Balduzzi, David and Frean, Marcus and Leary, Lennox and Lewis, JP and Ma, Kurt Wan-Duo and McWilliams, Brian},
  booktitle={International conference on machine learning},
  pages={342--350},
  year={2017},
  organization={PMLR}
}

@misc{
fesser2026unitary,
title={Unitary Convolutions for Message-passing and Positional Encodings on Directed Graphs},
author={Lukas Fesser and Bobak Kiani and Melanie Weber},
year={2026},
url={https://openreview.net/forum?id=xfwlgpe6st}
}
\bibliographystyle{icml2026}

%%%%%%%%%%%%%%%%%%%%%%%%%%%%%%%%%%%%%%%%%%%%%%%%%%%%%%%%%%%%%%%%%%%%%%%%%%%%%%%
%%%%%%%%%%%%%%%%%%%%%%%%%%%%%%%%%%%%%%%%%%%%%%%%%%%%%%%%%%%%%%%%%%%%%%%%%%%%%%%
% APPENDIX
%%%%%%%%%%%%%%%%%%%%%%%%%%%%%%%%%%%%%%%%%%%%%%%%%%%%%%%%%%%%%%%%%%%%%%%%%%%%%%%
%%%%%%%%%%%%%%%%%%%%%%%%%%%%%%%%%%%%%%%%%%%%%%%%%%%%%%%%%%%%%%%%%%%%%%%%%%%%%%%
\newpage
\appendix
\onecolumn

%\phantomsection
\section*{Appendix Table of Contents}
\begin{itemize}
  \item[\ref{app:deferred_theory}] \hyperref[app:deferred_theory]{Deferred Theory} \dotfill \pageref{app:deferred_theory}
  \item[\ref{app:heat_diffusion_sim}] \hyperref[app:heat_diffusion_sim]{Simulated Heat Diffusion Further Details and Results} \dotfill \pageref{app:heat_diffusion_sim}
  \item[\ref{app:MeshPDE_app}] \hyperref[app:MeshPDE_app]{MeshPDE Further Details and Results} \dotfill \pageref{app:MeshPDE_app}
  \item[\ref{sec:wb_deets}] \hyperref[sec:wb_deets]{WeatherBench2 Further Details and Results} \dotfill \pageref{sec:wb_deets}
\end{itemize}

\section{Deferred Theory} 
\label{app:deferred_theory}

This section provides both theoretical background and deferred proofs from the main text.

\subsection{Lie Algebras and the Exponential Map} \label{sec:exp}

In this section we review the formalism behind Lie algebras and the exponential map. A group is a mathematical structure that formalizes what it means for something to be \textit{symmetric}. We say that a group is a
matrix \textit{Lie group}, if it is a differentiable manifold and a subgroup of the set of invertible $n \times n$ matrices. Lie groups are equipped with a \textit{Lie algebra}, which is the tangent space at the identity element. Our work encounters the orthogonal and unitary Lie groups 
\begin{equation*}
    \mathrm{O}(n) = \{O \in \mathbb{R}^{n \times n} \colon OO^T = I\}, \qquad \mathrm{U}(n) = \{U \in \mathbb{C}^{n \times n} \colon UU^\dagger = I\}
\end{equation*}
as well as the special unitary group 
\begin{equation*}
    \mathrm{SU}(n) =  \{U \in \mathbb{C}^{n \times n} \colon \det (U) = 1\}.
\end{equation*}
The associated Lie algebras for $O(n)$ and $U(n)$ are given by 
\begin{equation*}
    \mathfrak{o}(n) = \{M \in \mathbb{R}^{n\times n}\colon M + M^T = 0 \}, \qquad \mathfrak{u}(n) = \{M \in \mathbb{C}^{n\times n}\colon M + M^\dagger = 0\}.
\end{equation*}
The exponential map provides a mechanism of parameterizing Lie groups with elements in the Lie algebra. For matrix Lie groups, the exponential map is simply the matrix exponential:
\begin{equation*}
    \exp (\mathbf{X}) = \overset{\infty }{\underset{i}{\sum}}\frac{1}{i!}\mathbf{X}^i.
\end{equation*}
Applying the exponential map to a linear operator is given by 
\begin{equation*}
    \exp (\mathbf{L}) (\mathbf{X}) = \overset{\infty }{\underset{i}{\sum}}\frac{1}{i!}\mathbf{L}^i(\mathbf{X}) = \mathbf{X} +\mathbf{L}(\mathbf{X}) + \frac{1}{2}\mathbf{L} \circ \mathbf{L}(\mathbf{X}) + \frac{1}{6}\mathbf{L} \circ \mathbf{L} \circ \mathbf{L}(\mathbf{X}) + \dots 
\end{equation*}
In the case of \cref{eq:lie}, $\mathbf{L}$ is graph convolution,  $\mathbf{L}(\mathbf{X)} = \mathbf{AXW}$. Further background on group theory and abstract algebra can be found in \citet{Artin1998}, \citet{hall2013lie}, and \citet{esteves2020theoretical}.

\subsection{Convolutional oversmoothing} \label{sec:GCN_smoothing}

This section provides a result from \citet{kiani2024unitary} which establishes that Graph Convolution Networks \cite{kipf2017semisupervised} have a high probability to exhibit smoothing.

\begin{proposition}[Proposition 7 in \citet{kiani2024unitary}] \label{thm:gcns_init}
    Given a simple undirected graph $\mathcal{G}$ on $n$ nodes with normalized adjacency matrix $\widetilde{\mathbf{A}} = \mathbf{D}^{-1/2}\mathbf{A}\mathbf{D}^{-1/2}$ and node degree bounded by $D$, let $\mathbf{X} \in \mathbb{R}^{n \times d}$ have rows drawn i.i.d. from the uniform distribution on the hypersphere in dimension $d$. Let $f_{conv}(\mathbf{X}) = \widetilde{\mathbf{A}} \mathbf{X}  \mathbf{W}$ denote convolution with orthogonal feature transformation matrix $\mathbf{W} \in O(d)$. Then, the event below holds with probability $1-\exp(-\Omega(\sqrt{n}))$:
    \begin{equation*}
        R_{\mathcal{G}}(\mathbf{X}) \geq 1-O\left(\frac{1}{n^{1/4}} \right)
        \quad \text{and} 
        \quad
        R_{\mathcal{G}}(f_{conv}(\mathbf{X})) \leq 1 -\frac{\Tr(\widetilde{\mathbf{A}}^3)}{\Tr(\widetilde{\mathbf{A}}^2)} + O\left(\frac{1}{n^{1/4}} \right).
    \end{equation*}
\end{proposition}

\subsection{Gauge and Euclidean Equivariance} \label{sec:symmetries}

In this section, we introduce the necessary background and formal definitions for the equivariance constraints commonly applied to tasks defined on meshes. While working with arbitrary meshes, many commonly used network architectures compute distances between node positions. One has the option of computing these distances in either global Cartesian coordinates or in local tangent spaces of the mesh. In both cases, we may exploit the symmetry of these coordinate systems by enforcing equivariance with respect to transformations from a certain symmetry group into the network architecture, which allows the network to automatically generalize across orbits.

We now give precise definitions of equivariance and invariance. 

\begin{definition}
    Let $f: \mathcal{X} \rightarrow \mathcal{Y}$ be a map between input and output vector spaces $\mathcal{X}$ and $\mathcal{Y}$.
Let $G$ be a group with representations $\rho^{\mathcal{X}}$ and $\rho^{\mathcal{Y}}$ which transform vectors in $\mathcal{X}$ and $\mathcal{Y}$ respectively. Representations are group homomorphisms which map group elements to invertible linear transformations. The map $f: \mathcal{X} \rightarrow \mathcal{Y}$ is \emph{equivariant} if 
    \begin{align*}
    \label{eq:equivariance_app}
        \rho^{\mathcal{Y}}(g)f(x)=f(\rho^{\mathcal{X}}(g)x) \ , \ \text{for all } g\in G, x \in \mathcal{X} \ .
    \end{align*}
\end{definition}

Invariance is a special case of equivariance in which $\rho^{\mathcal{Y}} = \rm{Id}^{\mathcal{Y}}$ for all $g \in G$. With an invariant operator, the output of $f$ is unaffected by the transformations applied to the input. 

\begin{definition}
    A map $f: \mathcal{X} \rightarrow \mathcal{Y}$ is \emph{invariant} if
    \begin{align*}
        f(x)=f(\rho^{\mathcal{X}}(g)x) \ , \ \text{for all } g\in G, x \in \mathcal{X} .
    \end{align*}
\end{definition}

\subsubsection{Euclidean Equivariance}

For a mesh defined over a global coordinate system, a common choice of symmetry constraint is equivariance to the Euclidean group in $n$ dimensions, $E(n)$. In this setting, the mesh is treated as a graph with positional encodings, and the equivariance constraint ensures generalization to different roto-translations of the mesh. 

\begin{definition}
\label{thm:en_equiv_const}
Let $t \in \mathbb{R}^n$ be a translation vector and $Q \in \mathbb{R}^{n \times n}$ an orthogonal matrix representing a rotation or reflection. A function $f$ is equivariant to the Euclidean group $E(n)$ if for any $t \in \mathbb{R}^n$ and $Q \in \mathbb{R}^{n \times n}$ we have
    \begin{equation*}
        f(Qx + t) = Qf(x) + t.
    \end{equation*}
\end{definition}

\subsubsection{Gauge Equivariance}

We may also choose to embed coordinates locally, using coordinates that are intrinsic to the 2D mesh rather than the extrinsic 3D coordinates of the embedding space. This approach arises from the desire for a general convolution-like operator over arbitrary manifolds discretized as a mesh. To encode data over a mesh it is still necessary to make a choice of local coordinate frame at each vertex. In order to guarantee the equivalence of the features resulting from different choices of reference frames, the model should be invariant to change of coordinates frame at each vertex, i.e. gauge equivariant. 

We specifically adapt the strategy described in \citet{dehaan2021gaugeequivariantmeshcnns} and define the local coordinate frame at each vertex in terms of a reference neighboring vertex. Denote $v_a$ as the reference neighbor for gauge $A$, in which the neighbors have angles $\theta_A$, and denote $v_b$ as the reference neighbor for gauge $B$ with angles $\theta_B$. Comparing the two gauges, we see that they are related by a rotation of angle $\phi$, so that $\theta_B = \theta_A - \phi$. This change of gauge is called a gauge transformation of angle $g := \phi$.  

\begin{definition}[Equations 3 and 4 in \citet{dehaan2021gaugeequivariantmeshcnns}] \label{thm:gauge_equiv_constraint}
Let $\rho_{\text{in}}$ and $\rho_{\text{out}}$ be input and output types with dimensions $C_{\text{in}}$ and $C_{\text{out}}$. Let $K_{\text{self}} \in \mathbb{R^{C_{\text{out}} \times C_{\text{in}}}}$ and $K_{\text{neigh}} \colon [0, 2\pi) \rightarrow \mathbb{R}^{C_{\text{out}} \times C_{\textit{in}}}$ be two kernels. We say the kernels are \textit{gauge equivariant} if for any gauge transformation $g \in [0, 2\pi) $ and for any angle $\theta \in [0, 2\pi)$ we have
    \begin{equation*}
        K_{\text{neigh}}(\theta - g) = \rho_{\text{out}}(-g)K_{\text{neigh}}(\theta)\rho_{\text{in}}(g)
        , \qquad K_{\text{self}} = \rho_{\text{out}} (-g)K_{\text{self}} \, \rho_{\text{in}}(g).
    \end{equation*}
\end{definition}

Finally, as features at different nodes live in different tangent spaces and thus have different gauges, it is invalid to sum them directly. Let $f_p$ and $f_q$ be node features of a pair of neighboring nodes $p$ and $q$. Before performing gauge equivariant convolution, we must parallel transport each $f_q$ to $T_pM$ along the mesh edge that connects the two vertices for them to be in the same gauge. For more details, we refer the reader to  \citet{dehaan2021gaugeequivariantmeshcnns}. 

\subsection{Unitary Learning Framework} \label{sec:unitary_learning}

This section provides rigorous definitions for the mathematical tools used in the main text and additionally clarifies necessary hypotheses.

We start with the fundamental domain. Assume $X$ has dimension $n$. Let $d$ be the dimension of a generic orbit of $G$ in $X$.  Let $\nu$ be the $(n-d)$ dimensional Hausdorff measure in $X$.

\begin{definition}[Fundamental Domain, Definition 4.1 in \citet{wang2023general}] A closed subset $F$ of $X$ is called a fundamental domain of $G$ in $X$ if $X$ is the union of conjugates of $F$, i.e., $X = \cup_{g \in G}gF$, and the intersection of any two conjugates has measure $0$ under $\nu$.    
\end{definition}

Next, we note that our proof of \cref{thm:uni_lb} satisfies the integrability assumption on the fundamental domain $F$ and orbits $Gz$ established in \citet{wang2023general}:

\begin{assumption}[Integrability Hypothesis, Sec. A in \citet{wang2023general}] The fundamental domain $F$ and orbit $Gx$ are differentiable manifolds and the union of all pairwise intersections $\cap_{g_1 \neq g_2}(g_1F \cap g_2F)$ has measure zero. 
\end{assumption}

We now provide more formal definitions for $\mathbb{E}_{Gx}[f]$ and $\mathbb{V}_{Gx}[f]$ used in \cref{prop:inv}. Denote by $q(z) = \frac{p(z)}{p(Gx)}$ the density of the orbit $Gx$ so that $\int_{Gx}q(z)dz = 1$. The mean and variance of a function $f$ on $Gx$ are given by 
\begin{equation*}
 \mathbb{E}_{Gx}[f] = \int_{Gx}q(z)f(z)dz  , \qquad \mathbb{V}_{Gx}[f] =  \int_{Gx} q(z) \lVert \mathbb{E}_{Gx}[f] - f(z)\rVert_2^2dz. 
\end{equation*}
\subsection{Proof of Main Theorem} \label{sec:proof}

We now provide proof of our main theoretical result in the main text. We repeat the theorem here for convenience.

\begin{theorem*}
Let $F$ be a fundamental domain of $\mathrm{SU}(n)$ in $Z$. In particular, $F = \{ t e \colon t \in \mathbb{R}_+\}$ where $e$ is a standard basis vector of $\mathbb{C}^n$. The approximation error lower bound can be expressed as $$\int_{Z}p(z)\lVert u(z) - f(z)\rVert_2^2dz \geq \int_{F} p(\lVert te \rVert) \mathbb{V}_{Gz}[\rVert f\lVert]dz.$$
\end{theorem*}

\begin{proof}[Proof of \cref{thm:uni_lb}]
By the reverse triangle inequality, 
\begin{align*}
\int_{Z}p(z)\lVert u(z) - f(z)\rVert_2^2dz  &\geq \int_{Z}p(z)\left( \lVert u(z)\rVert - \lVert f(z)\rVert \right)^2dz.
\end{align*}
Notice that $\lVert u(z)\rVert$ is invariant under the action of $\mathrm{SU}(n)$ on the sphere $S^{2n-1}$ with radius $\lVert te \rVert$ and recall that $\mathrm{SU}(n)$ acts transitively on the sphere. Thus, $F = \{ t e \colon t \in \mathbb{R}_+\}$ is a valid fundamental domain that indexes each orbit $Gz$, the spheres with radii $\lVert te \rVert$. Our theorem then follows from \cref{prop:inv}.
\end{proof}

In the following example, we show how this bound may be computed.

\begin{example}[Variance on the Unit Disk]
Denote by $D_2$ the unit disk $D_2 = \{ (x,y) \colon x^2 + y^2 \leq 1\}$. Let $Z = \mathbb{R}^2$ be a domain with density $$p = \begin{cases}
    \frac{1}{\pi }, & (x,y) \in D_2\\
    0,& \text{Otherwise}.
\end{cases}$$ 
Denote by $f$ a function in polar coordinates given by
\begin{align*}
    f \colon (\theta,r) &\to \mathbb{R}^2\\
    \theta, r &\mapsto (1 + \sin \theta) \cdot \left(\cos(\theta + r), \sin(\theta + r)\right) .
\end{align*}
On each orbit, we compute 
\begin{equation*}
\mathbb{E}_{G_z}[\lVert f\rVert ] = 1, \qquad \mathbb{V}_{G_z}[\lVert f\rVert] = \frac{1}{2}, \qquad   p(r) = \begin{cases}
        2 r, & r \leq 1\\
        0 , & \text{otherwise}.
    \end{cases}  
\end{equation*}
The approximation error bound of a unitary function $u$ of $f$ is then  
\begin{align*}
   \int_{Z}p(z)\lVert u(z) - f(z)\rVert_2^2dz  \geq  \int_F p(r) \mathbb{V}_{G_z}[\lVert f\rVert ] dr &= \int_0^1 2r (1/2) dr = 1/2.
\end{align*}
\end{example}

\subsection{Unitary Convolution on Meshes} \label{sec:mesh_conv}

In this section, we prove \cref{cor:main_cor} stated in \cref{sec:relaxed_meshes} and repeated here for convenience.

\begin{corollary*}[Corollary to \cref{prop:rq_preserved}] 
Given a mesh $\mathcal{M}$ with normalized adjacency matrix $\mathbf{\tilde{A}} = \mathbf{D^{-1/2}}(\mathcal{W} \odot \mathbf{A})\mathbf{D^{-1/2}}$ that satisfies \cref{ass:assumption}, the mesh Rayleigh quotient is invariant under normalized unitary or orthogonal graph convolution, i.e. $R_{\mathcal{M}} (\mathbf{X}) = R_\mathcal{M} (f_{\rm UniMeshConv} (\mathbf{X}))$ where $f_{\rm UniMeshConv}$ is either separable or Lie.
\end{corollary*}

Our proof follows the same structure as the proof of \cref{prop:rq_preserved} in \citet{kiani2024unitary} with modifications to account for the weighted adjacency matrix. Namely, we invoke \cref{ass:assumption} which ensures that $f_{\rm UniMeshConv}$ is norm preserving and therefore the the strategy in \citet{kiani2024unitary} still holds.

\begin{proof} We first prove invariance for \cref{eq:msep}. By the circulant property of the trace,
\begin{equation*}
 \Tr \left(\left( \exp (i \tilde{\mathbf{A}}) \mathbf{XU}\right)^\dagger\left(\mathbf{I - \tilde{A}}  \right) \left( \exp (i \tilde{\mathbf{A}}) \mathbf{XU}\right) \right) = \Tr \left( \mathbf{X}^\dagger\exp (-i \tilde{\mathbf{A}})\left(\mathbf{I - \tilde{A}}  \right)  \exp (i \tilde{\mathbf{A}}) \mathbf{X} \right).
\end{equation*}
Because $\exp(-i\mathbf{\tilde{A}})$, $\exp(i\mathbf{\tilde{A}})$, and $(\mathbf{I - \tilde{A}})$ share an eigenbasis, they commute, so $$\Tr \left(\left( \exp (i \tilde{\mathbf{A}}) \mathbf{XU}\right)^\dagger\left(\mathbf{I - \tilde{A}}  \right) \left( \exp (i \tilde{\mathbf{A}}) \mathbf{XU}\right) \right) = \Tr \left(\mathbf{X}^\dagger \left(\mathbf{I - \tilde{A}}\right)\mathbf{X}\right).$$
For the denominator, we need to show that $\lVert \exp (i \mathbf{\tilde{A}})\mathbf{XU}\rVert_F^2 = \lVert \mathbf{X}\rVert_F^2$. By \cref{ass:assumption} we have that $\mathcal{W}$ is symmetric. Because $\mathbf{A}$ is also symmetric, we have that $i\tilde{\mathbf{A}}$ is skew hermitian and therefore $\exp (i \tilde{\mathbf{A}}) \in SU(n)$. Thus, $\lVert \exp (i \mathbf{\tilde{A}})\mathbf{XU}\rVert_F^2 = \lVert \mathbf{X}\rVert_F^2$ and finally $R_{\mathcal{M}} (\mathbf{X}) = R_\mathcal{M} (f_{\rm UniMeshConv} (\mathbf{X}))$.

We now show that \cref{eq:mlie} also preserves the Rayleigh quotient. First, we need to show that $\lVert \exp ( \mathbf{AXW)} \rVert_F^2 = \lVert \mathbf{X} \rVert_F^2$. To do this, we note that \cref{eq:mlie} can equivalently be viewed as a function that acts on a vector in $\mathbb{C}^{nd}$. By properties of the Kronecker tensor product, 
\begin{equation*}
f_{\rm UniMeshConv}(\mathbf{X; A}) = \exp (\mathbf{AXW}) \iff \text{vec}\left(f_{\rm UniMeshConv}(\mathbf{X; A})\right) = \exp\left( \mathbf{A \otimes W^T} \right)\text{vec}(\mathbf{X}).
\end{equation*}
Since $$\left( \mathbf{A \otimes W^T} \right) + \left( \mathbf{A \otimes W^T} \right)^\dagger = \mathbf{A} \otimes \left(\mathbf{W + W^\dagger}\right)^T = 0,$$ we have that $\left( \mathbf{A \otimes W^T} \right)$ is in the lie algebra of the unitary group and therefore preserves the norm of $\text{vec}(\mathbf{X})$. This holds for any symmetric edge weighting $\tilde{\mathbf{A}} = \mathcal{W} \odot \mathbf{A}$, which is guaranteed by \cref{ass:assumption}. Thus, $\lVert \exp ( \mathbf{AXW)} \rVert_F^2 = \lVert \mathbf{X} \rVert_F^2$. Next, note that $\exp \left(  \mathbf{\tilde{A} \otimes W^T} \right)$ commutes with $(\mathbf{\tilde{A} \otimes I})$. Thus,
\begin{align*}
\Tr &\left( f_{\rm UniMeshConv}(\mathbf{X; \tilde{A}}) ^\dagger (\mathbf{I - \tilde{A}})f_{\rm UniMeshConv}(\mathbf{X; \tilde{A}})\right) \\
&= \text{vec}(\mathbf{X})^\dagger \exp (\mathbf{\tilde{A}\otimes W^T})^\dagger \left[(\mathbf{I - \tilde{A}} ) \otimes \mathbf{I}\right] \exp (\mathbf{\tilde{A} \otimes W^T}) \text{vec}(\mathbf{X})\\
&= \text{vec}(\mathbf{X})^\dagger \left[(\mathbf{I - \tilde{A}} ) \otimes \mathbf{I}\right]  \text{vec}(\mathbf{X}).
\end{align*}
Multipliying the above by $\lVert \mathbf{X}\rVert_F^{-2}$ recovers $R_{\mathcal{M}}(\mathbf{X})$. We conclude that $R_{\mathcal{M}}(\mathbf{X}) = R_{\mathcal{M}}(f(\mathbf{X}))$.
\end{proof}

\begin{remark}
\cref{cor:main_cor} was applied to convolution with the symmetric cotangent weights in \cref{eq:cot_weights}, but the proof extends without loss of generality to any set of symmetric weights.
\end{remark}

\subsection{Discrete Differential Geometry} \label{sec:mesh_laplacian}

We provide further details and visualizations for concepts in discrete differential geometry, including the mesh manifold condition, the cotangent Laplacian, and Delaunay criterion. We also review results from the literature that suggests that the mesh edge rewiring algorithm used by the Robust Laplacian is a safe choice for the task of PDE solving. For additional reference on these topics, see \citet{meyer2003discrete} and \citet{Crane2013DGP}.

\paragraph{The Manifold Condition.}

Our work makes use of the assumption that a given mesh is manifold. We recall the following definition of the mesh manifold condition:

\begin{definition}[Manifold Condition, \citep{sharp2020laplacian}] \label{def:man_cond}
An interior (or boundary) edge $ij$ is manifold if it is contained in exactly two (or one) triangles; an interior (or boundary) vertex $i$ is manifold if the boundary of all triangles incident on $i$ forms a single loop (or path) of edges. 
\end{definition}

Alternative definitions for the manifold condition often state that a mesh is manifold if interior vertices have local neighborhoods that are homeomorphic to the unit disk and boundary vertices are homeomorphic to the half disk. We refer specifically to \cref{def:man_cond} when we talk about a mesh being manifold in the paper.

\paragraph{Cotangent Laplacian Area Normalization.}

Recall that for a scalar function $s$ on mesh vertices we the cotangent Laplacian is defined as
\begin{equation*}
    (\mathbf{\tilde{L}}(s))_i = \frac{1}{2A_i}\underset{j \in \mathcal{N}(i)}{\sum} \left(\cot \alpha_{ij} + \cot \beta_{ij} \right)(s_j - s_i)
    \label{eq:laplacian_copy}
\end{equation*}
where $\mathcal{N}(i)$ denotes the adjacent vertices of $i$, $\alpha_{ij}$ and $\beta_{ij}$ are the angles opposite edge $(i, j)$, and $A_i$ is the vertex area of $i$ and that we use the barycentric cell area for $A_i$. In particular, let $\mathcal{A}_{abc}$ be the area of a triangular face with vertices $abc$ and let $\mathcal{F}(i)$ be the set of faces containing vertex $i$. The barycentric cell area is defined
\begin{equation*}
    A_i = \underset{abc \in \mathcal{F}(i)}{\sum} \mathcal{A}_{abc}/3.
\end{equation*}
Normalization by the cell area was used for the dataset construction in \citet{park2023modelingdynamicsmeshesgauge}, but it is not used in the definition of the Robust Laplacian \citep{sharp2020laplacian}. In fact, the cell area introduces asymmetry in the edge weights. This is undesirable, as unitary mesh convolution depends on symmetric edge weights in order to preserve the Rayleigh quotient. 

\paragraph{Cotangent Laplacian Edge Weights and Robust Rewiring.} As noted in \cref{sec:relaxed_meshes}, an arbitrary mesh may have negative cotangent weights. These cotangent weights have the following geometric meaning. For vertices $ij$ connected by an edge, we say that the edge is \textit{primal}. For manifold meshes, for each primal edge connecting two triangle there exists a \textit{dual edge} that connects the triangle circumcenters. The cotangent weights correspond to the ratio of the primal and dual edge lengths for vertices $ij$  \citep{Crane2013DGP}. The weights are positive when the angles $\alpha_{ij} + \beta_{ij} \leq \pi$. The Robust Laplacian applies two edge rewiring algorithms sequentially, the tufted cover algorithm and the Delaunay edge flip algorithm. The tufted cover algorithm ensures the mesh is manifold so that the Delaunay edge flip algorithm can be applied. The Delaunay edge flip algorithm then ensures that the mesh satisfies the intrinsic Delaunay criterion (\cref{def:delaunay}). A sample edge flip is illustrated in \cref{fig:edge_flip}.

\begin{figure}[!htb]
    \centering
    \includegraphics[width=0.5\linewidth]{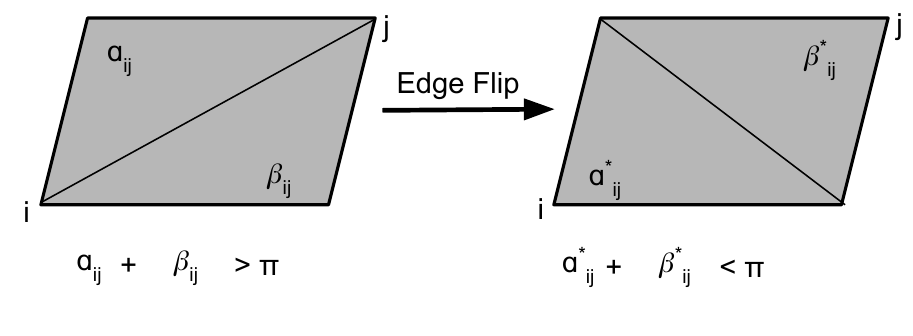}
    \caption{Illustration of an intrinsic Delaunay edge flip performed by the Robust Laplacian edge rewiring algorithm. This figure is a reproduced version of Figure 7, \citet{sharp2020laplacian}.}
    \label{fig:edge_flip}
\end{figure}

\citet{sharp2020laplacian} note that from a finite element perspective, changing the triangulation via Delaunay edge-flipping effectively just provides a different set of linear basis functions for the same polyhedral domain. Thus, the only practical concern is whether the tufted cover algorithm, which ensures that the mesh is manifold, does not dramatically change the connectivity. Empirical results from \citet{sharp2020laplacian} find that the Robust Laplacian greatly \textit{improves} performance computing geodesic distances with the heat method \citep{crane2017heat}, which depends on numerical PDE solving on highly nonmanifold meshes with the cotangent Laplacian. This provides confidence that the tufted cover algorithm improves the edge-weighting scheme for unitary mesh convolution. See \citet{crane2017heat} for additional reference on numerical PDE solving on nonmanifold meshes.

\subsection{Rayleigh Quotient Sensitivity} \label{sec:sensitivity_more_details}

We include results from \citet{ferrandi2022homogeneous} and \citet{dong2023rayleigh} that illustrate the sensitivity of the Rayleigh quotient to small perturbations of the input, such as Taylor series truncation errors. While the hypotheses are stronger than what we may actually see in practice, the following proposition provides an intuition for the Rayleigh quotient sensitivity.

\begin{proposition}[Proposition 4 in \citet{ferrandi2022homogeneous}] \label{prop:sensitivity_f}
Suppose $\mathbf{u = x +e}$ is an approximate eigenvector corresponding to a simple eigenvalue $\lambda \neq 0$ of a symmetric $A$, with $\lVert \mathbf{x} \rVert = 1$, $\mathbf{e\perp x}$, and $\varepsilon = \lVert \mathbf{e}\rVert$. Then, up to $\mathcal{O}(\varepsilon^4)$-terms, for the sensitivity of the Rayleigh quotient (as a function of $\mathbf{u}$) it holds that $$\underset{\lambda_i \neq \lambda}{\min}\frac{|\lambda_i - \lambda|}{|\lambda_i|} \varepsilon^2\lesssim \frac{|R_{\mathcal{G}}(\mathbf{u}) - \lambda|}{|\lambda|}\lesssim \underset{\lambda_i \neq \lambda}{\max}\frac{|\lambda_i - \lambda|}{|\lambda_i|} \varepsilon^2.$$  
\end{proposition}

This indicates that the Rayleigh quotient sensitivity is quadratic in perturbations $\varepsilon$. For $\varepsilon < 1$, this means that the sensitivity of the Rayleigh quotient is even less than the truncation error. We also have the following results from \citet{dong2023rayleigh}:

\begin{proposition}[Theorem 1 in \citet{dong2023rayleigh}] \label{prop:dong_1}
For any given graph $G$, if there exists a perturbation $\Delta$ on $\mathbf{L}$, the change of Rayleigh quotient can be bounded by $\lVert \Delta \rVert_2$.
\end{proposition}

\begin{proposition}[Theorem 2 in \citet{dong2023rayleigh}] \label{prop:dong_2}
For any given graph $G$, if there exists a perturbation $\delta$ on $\mathbf{x}$, the change of Rayleigh quotient can be bounded by $2\mathbf{x}^T\mathbf{L}\delta +o(\delta)$. If $\delta$ is small enough, in which case $o(\delta)$ can be ignored,
the change can be further bounded by $2\mathbf{x}^T \mathbf{L}\delta$.
\end{proposition}

The results from \citet{dong2023rayleigh} state fewer hypotheses than \citet{ferrandi2022homogeneous}.  \cref{prop:dong_1} outlines a bound similar to \cref{prop:sensitivity_f} in that they are both related to the norm of the perturbing vector. \cref{prop:dong_2} states an alternative bound related to a perturbation on the input node features instead of the graph's Laplacian. Importantly, one can estimate the truncation error in a unitary convolution layer using Taylor's theorem and substitute this value for $\delta$ in \cref{prop:dong_2}. Comparing the deviation in the Rayleigh quotient with the expected energy dissipation of the PDE gives a model selection criterion for choosing $\mathbf{T}_{\rm max}$.

\clearpage

\section{Simulated Heat Diffusion Further Details and Results}

\label{app:heat_diffusion_sim}

\subsection{Simulated Heat Diffusion Dataset} \label{sec:dataset_details}

This section details dataset generation specifications for our experiment in \cref{sec:motivate}. We generate grid-graphs with an average of $10$ nodes and a standard deviation of $2$ nodes. On the grid we randomly set $20$ nodes to be heat sources. They are given a heat value of $1$ and all other nodes start at $0$. Using \texttt{PyGSP}, We simulate heat flow on $10,000$ graphs for training, and the task is to predict the next time step given the previous one. The simulation proceeds until time $T=10$ in increments of $\Delta T = 0.5$ time steps.  A sample graph data point is given in \cref{fig:3x3grid}.

\begin{figure}[!htb]
  \centering
    \includegraphics[width=0.2\linewidth]{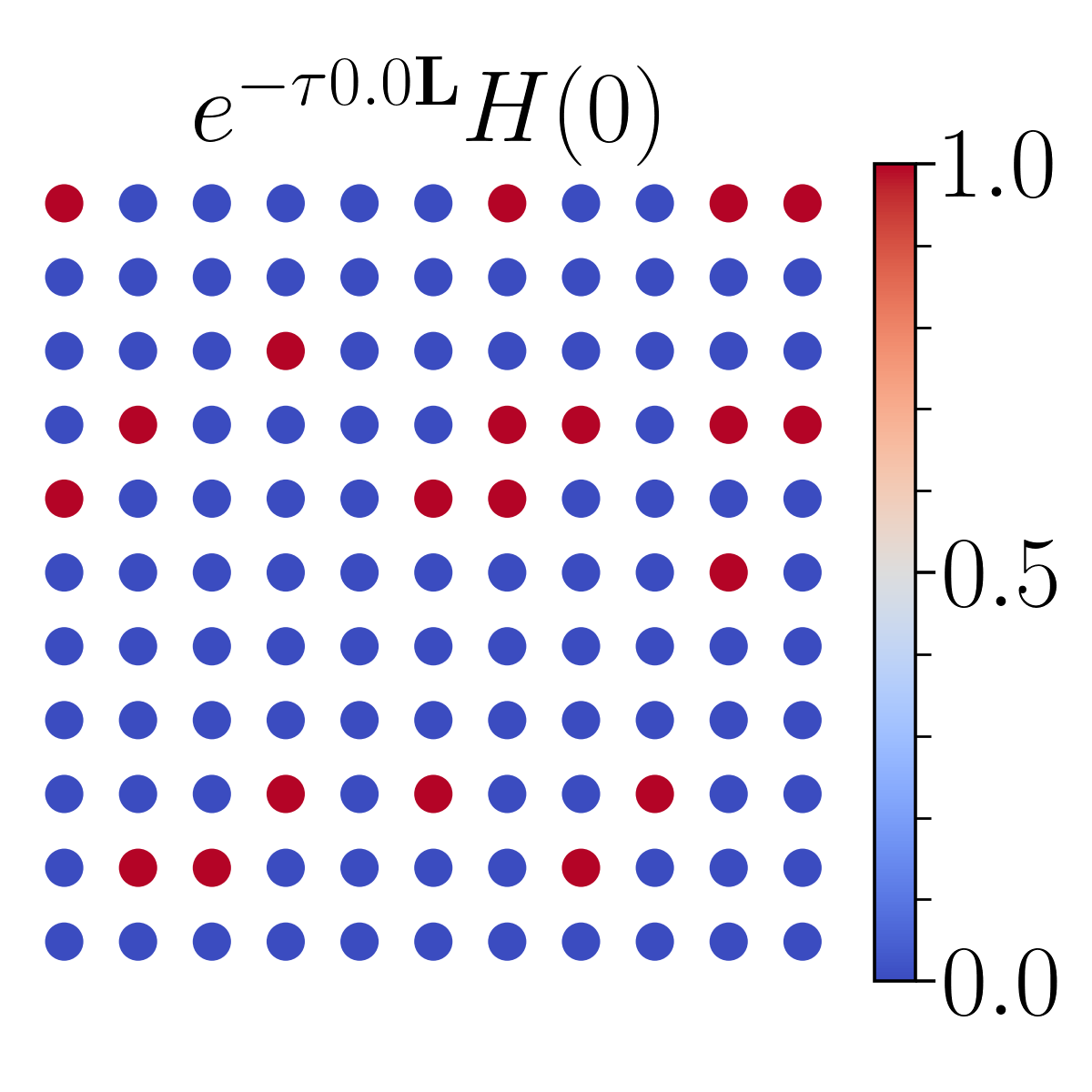}
    \label{fig:sub1}
    \includegraphics[width=0.2\linewidth]{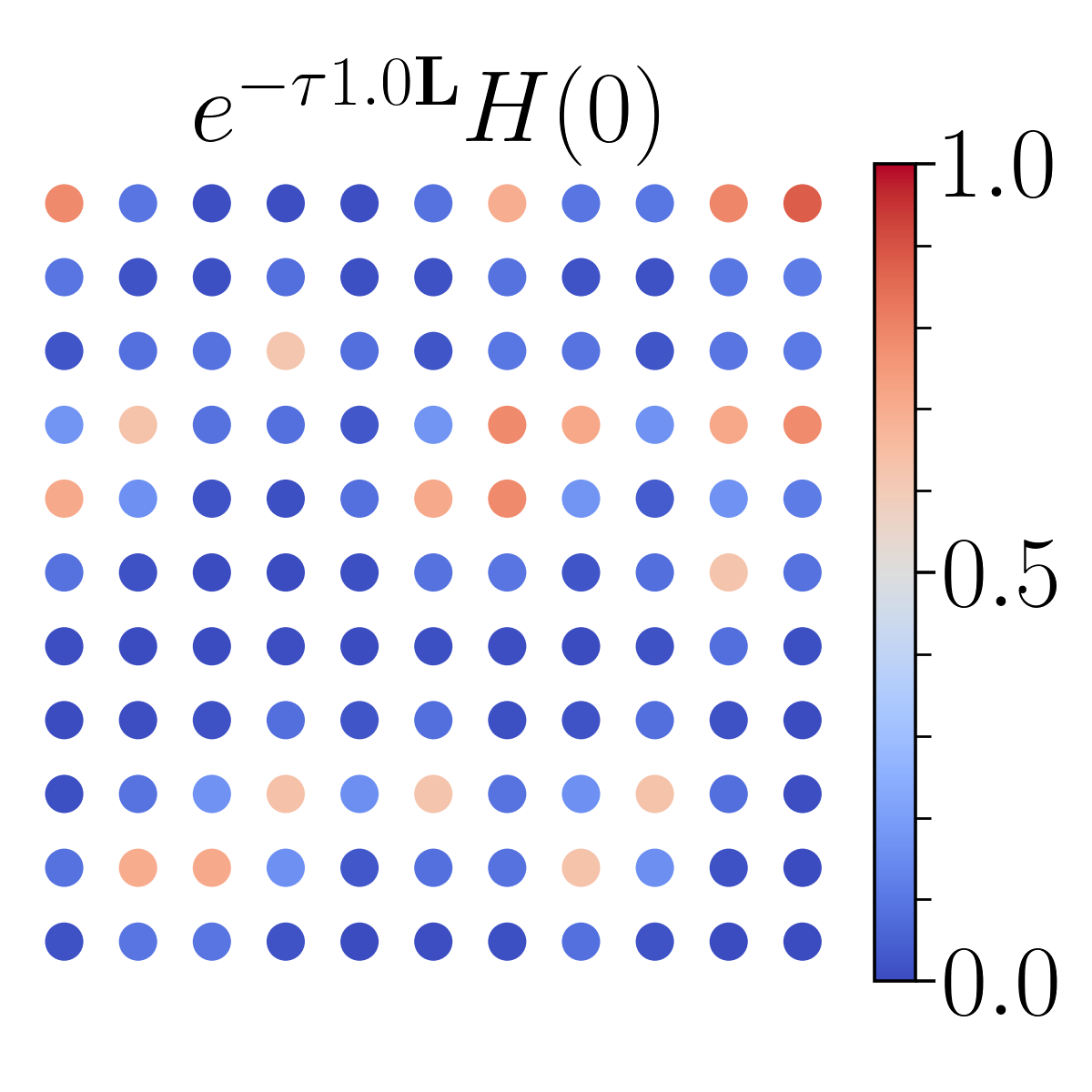}
    \includegraphics[width=0.2\linewidth]{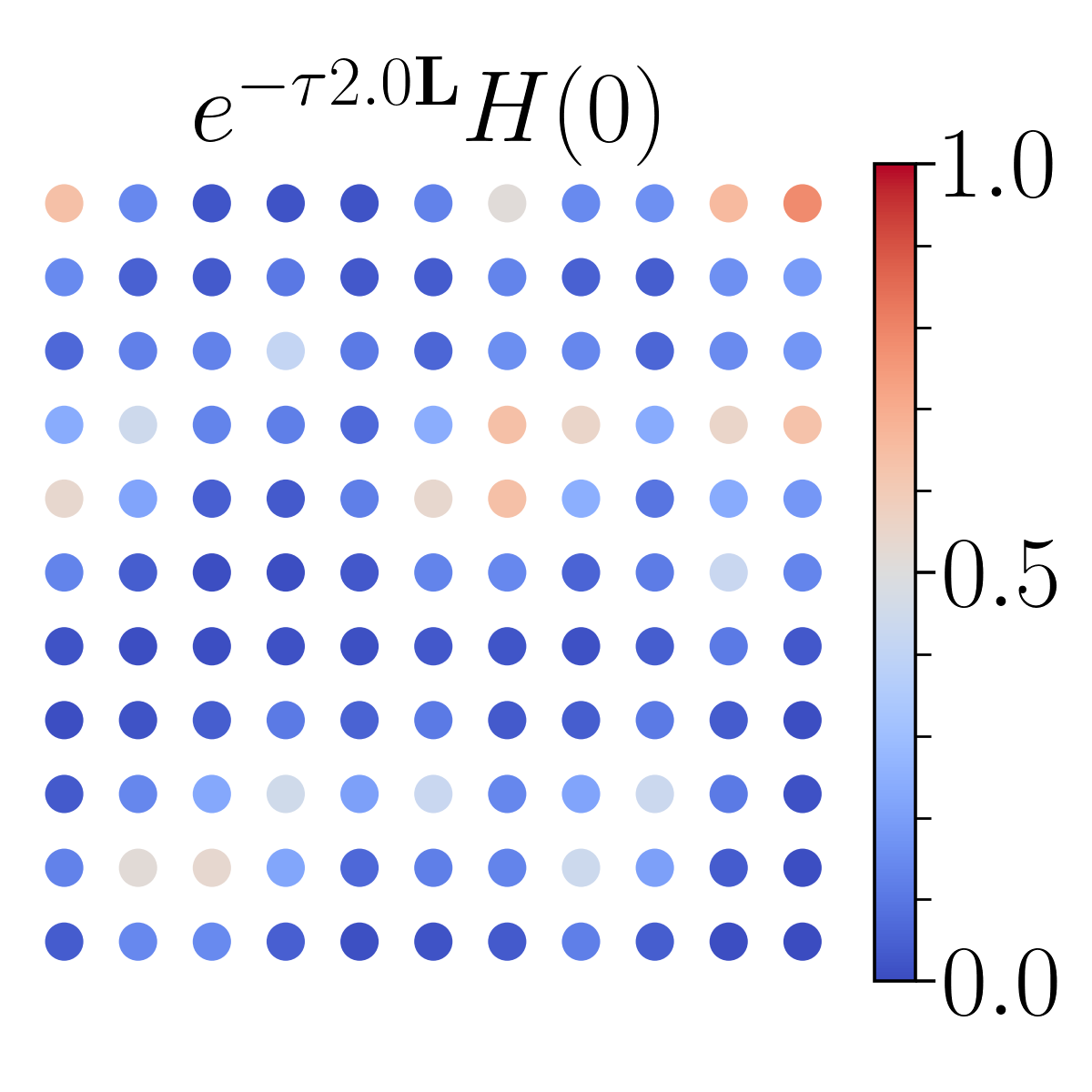}
    \includegraphics[width=0.2\linewidth]{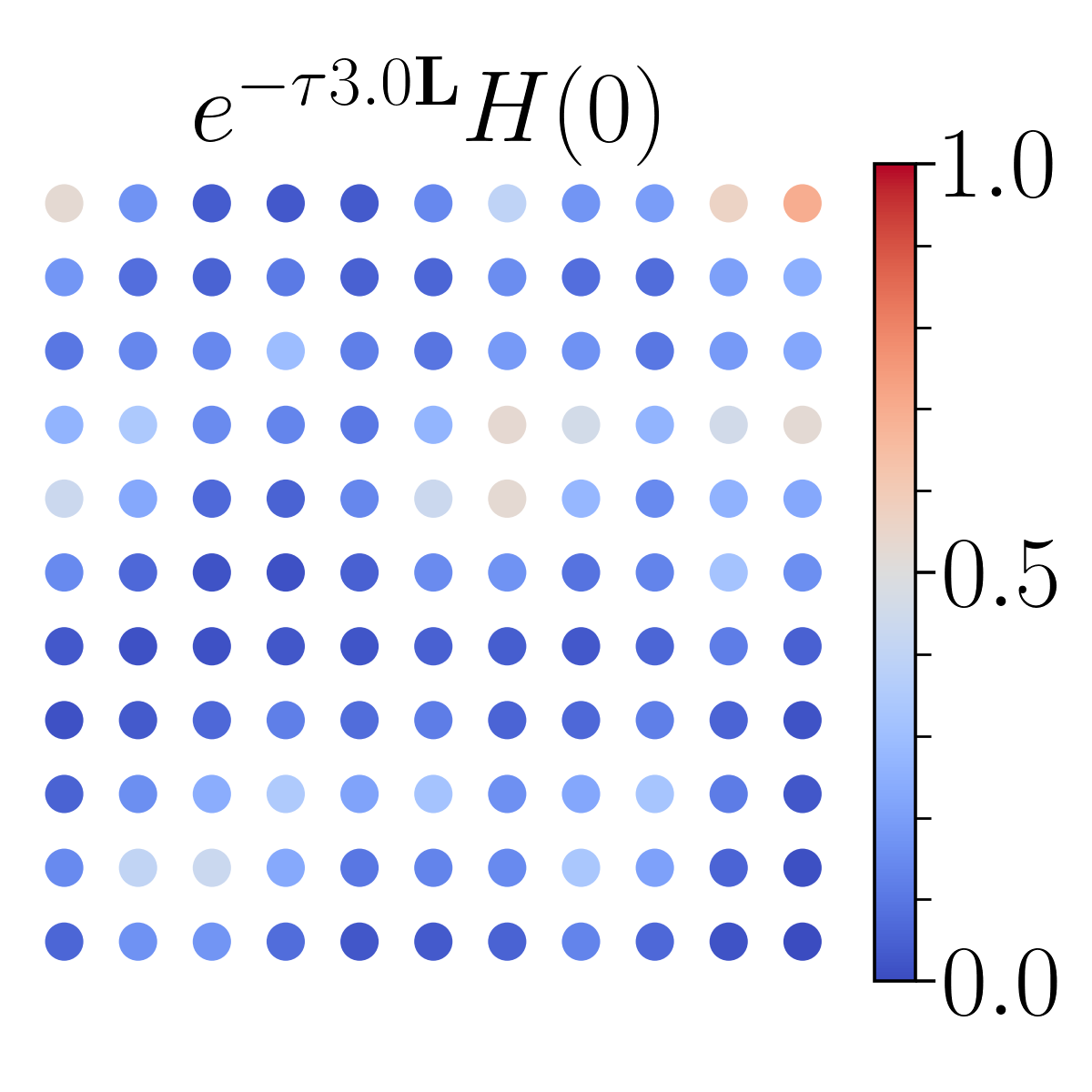}
  \caption{Sample heat diffusion process on a grid discretized as a graph. Node neighbors are the nodes that sit adjacent in the grid.}
  \label{fig:3x3grid}
\end{figure}

\subsection{Taylor Series Sensitivity Analysis} \label{sec:sensitivity_analysis_taylor}

We conduct a sensitivity analysis of \oursGraph{} to different Taylor series truncations. For completeness, we also compare with standard GCNs and Separable unitary networks. First, we study these tendencies at initialization for the heat diffusion dataset that is used for the experiment in \cref{sec:motivate}, described further in \cref{sec:dataset_details}. Our analysis echos a theme similar to \citet{gruver2024liederivativemeasuringlearned} and \citet{gao2025discretization} that practitioners should be more thorough in evaluating when numerical approximations break strict theoretical guarantees. Secondly, we study the impact of different truncation terms on downstream performance.

\begin{figure}[!htb]
    \centering
    \href{https://anonymous.4open.science/r/rayleigh_analysis-BD52/assets/LieUni_timelapse_loop.gif}{
    \includegraphics[width=0.5\linewidth]{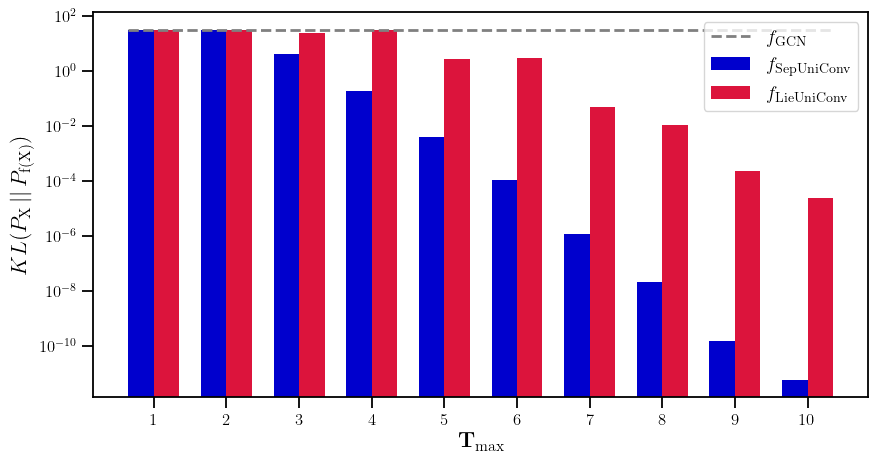}
    }
    \caption{KL divergence between distribution of Rayleigh quotients before and after applying the model. Results are averaged over 10 runs.}
    \label{fig:truncation}
\end{figure}

\paragraph{Experimental Setup.} We simulate heat diffusion on a grid graph and use time step $3$ to conduct the sensitivity analysis. We evaluate on the models $f_{\rm GCN}$, $f_{\rm Sep Uni Conv}$, and $f_{\rm Lie Uni Conv}$. For each model $f$ and truncation length $\mathbf{T_{\rm max}} \in \{1, \dots, 10\}$, we compute the Rayleigh quotients $R_\mathcal{G}(\mathbf{X})$ and $R_\mathcal{G}(f(\mathbf{X}))$ for all graph mini batches $\mathbf{X}$. We denote the distribution of Rayleigh quotients before applying the model by $\PX$ and after applying the model by $\PfX$. To quantify the deviation between these distributions, we compute the KL divergence $\kl{\PX}{\PfX}$, which measures the change in the distribution of Rayleigh quotients caused by the model at initialization. We then examine the impact of these truncations on downstream performance in terms of MSE and Rayleigh error for \oursGraph{}.

\paragraph{Results.} We see in \cref{fig:truncation} the effect of Taylor series truncation on the unitarity of the network. In particular, we observe that the KL divergence between the two distributions decreases exponentially with the number of terms. This is to be expected, we know from Taylor's theorem that a truncation at term $t$ gives truncation error $\mathcal{O}\left(\frac{\left[\lVert \mathbf{AXW}\rVert_{\mathbf{O}}\right]^{t+1} \lVert\mathbf{X}\rVert_2}{(t+1)!}\right)$ where $\mathbf{\lVert \cdot \rVert}_{\mathbf{O}}$ is the operator norm. Furthermore, works such as \citet{ferrandi2022homogeneous} and \citet{dong2023rayleigh} show theoretically that small truncation errors will not compound into large deviations in the Rayleigh quotient. For details on the relevant propositions from \citet{ferrandi2022homogeneous} and \citet{dong2023rayleigh}, see \cref{sec:sensitivity_more_details}. 

In our \href{https://github.com/EdwardBerman/rayleigh_analysis/tree/main/sup_mat}{GitHub repository} we include a video that shows the evolving Rayleigh quotient distribution as we increase $\mathbf{T_{\rm max}}$. \cref{fig:truncation} is also clickable and links to the same video in our anonymous artifact.

We quantify the effect of these truncations on downstream performance in \cref{fig:tmax_abl}. We see that $\mathbf{T}_{\rm max} = 3$ achieves the best balance of convergence under both MSE loss and MRE. Moreover, we see that \oursGraph{} is not overly sensitive to different choices of $\mathbf{T}_{\rm max}$ in terms of MSE. Interestingly, the $\mathbf{T}_{\rm max} = 1$ model still performs well in terms of MSE despite oversmoothing the target signal. The $\mathbf{T}_{\rm max} = 1$ model can be thought of as a GCN with a hermitian projection of the weights. While the smoothness behavior is similar to what we see for the unconstrained GCN (cf. \cref{fig:loss_and_rq_ensemble}), the convergence under MSE loss is not. This suggests that there exists a local minimum where the predicted node features are mostly constant that is close to the true heat field. Moreover, it raises the possibility that the benefits of approximately unitary networks lie not only in smoothness preservation but in gradient stability properties. For example, unitary convolutions are known to be perfectly dynamically isometric and $1$-Lipschitz \cite{xiao2018dynamical, kiani2024unitary}. Networks with these properties tend to avoid vanishing or exploding gradients.

\begin{figure}[!htb]
    \centering
    \includegraphics[width=1\linewidth]{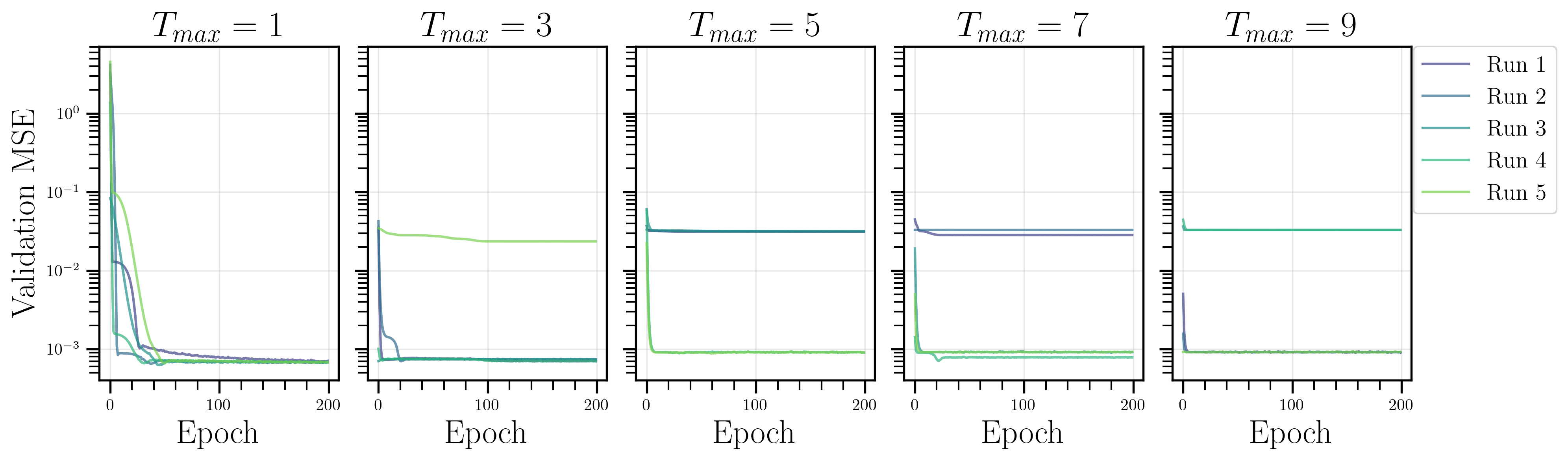}
    \includegraphics[width=1\linewidth]{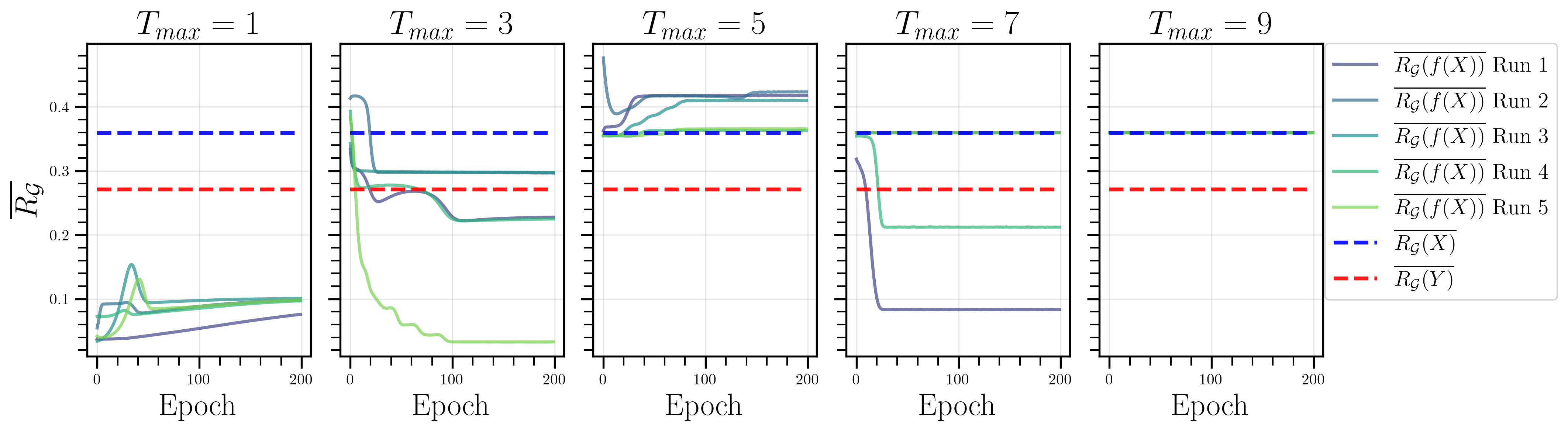}
    \caption{Validation MSE and mean Rayleigh quotient for $5$ different Lie Unitary Convolution runs. MSE tends to increase as $\mathbf{T}_{\rm max}$ increases. $\mathbf{T}_{\rm max}=3$ offers the best balance between accuracy under MSE loss and smoothness errors.}
    \label{fig:tmax_abl}
\end{figure}

\subsection{Heat Diffusion Model Ensembles} \label{sec:ensemble}

This section provides the results for \cref{sec:motivate} for a larger ensemble of models. In \cref{fig:loss_and_rq_ensemble}, we see that the behavior in the training runs in \cref{fig:loss_and_rq} occur frequently, with only a couple runs diverging.

\begin{figure}[!htb]
    \centering
    \includegraphics[width=0.75\linewidth]{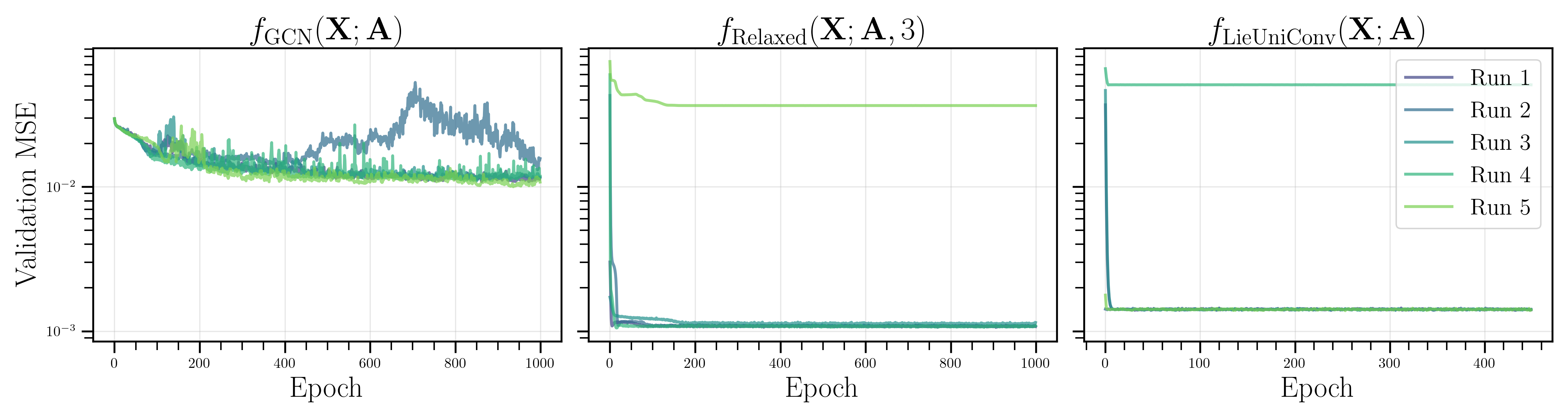}
    \includegraphics[width=0.75\linewidth]{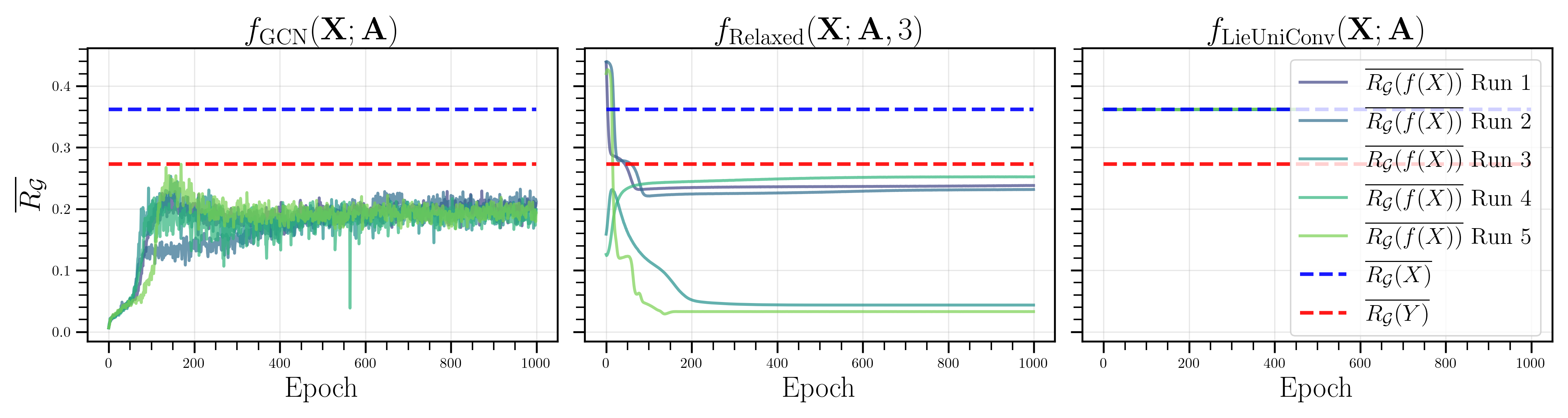}
    \caption{\textbf{Top:} Validation MSE for an ensemble of $5$ runs for a GCN (left), \oursGraph{} (middle), and a Lie unitary convolution network (right) at timestep $t=3$. \oursGraph{} significantly outperforms the GCN and also outperforms the Lie unitary network. \textbf{Bottom:} The average Rayleigh quotient over all graphs for an ensemble of $5$ runs for the same models at timestep $t=3$. The GCN is under constrained and biased towards oversmoothing at initialization. \oursGraph{} is able to roughly match the true smoothness of the labeled graphs. The Lie unitary network is overconstrained and can not model the Rayleigh quotient of the labels because it is forced to preserve the Rayleigh quotient of the input graphs.}
    \label{fig:loss_and_rq_ensemble}
\end{figure}

\clearpage 

\section{MeshPDE Further Details and Results} 

\label{app:MeshPDE_app}

This section provides extra experimental details and results for our dynamical systems modeling on \pv{} meshes in MeshPDE.

\subsection{MeshPDE Dynamical Systems} \label{sec:pdes}

This section provides details on the PDEs to be solved on the \pv{} meshes in MeshPDE. Before defining the PDEs to be solved, let us establish notation. Let $\alpha$ be the thermal diffusivity, and $c$ a constant. The heat and wave equations on the mesh are then given by
\begin{align*}
    \frac{\partial u}{\partial t}= \alpha \mathbf{\tilde{L}}u, \quad \text{(Heat)} &&  \frac{\partial^2 u}{\partial t^2}= c^2 \mathbf{\tilde{L}}u, \quad \text{(Wave)}. \label{eq:dynamics}
\end{align*}
We now define the Cahn-Hilliard equation \citep{cahn1958free}. Let $c$ be the fluid concentration, $M$ the diffusion coefficient, $\mu$ the chemical potential, $f$ the double-free energy function, and $\lambda$ a positive constant. The Cahn-Hilliard equation is often represented by the following two coupled second order equations:
\begin{equation*}
  \frac{\partial c}{\partial t} - M\mathbf{\tilde{L}}(\mu) = 0 ,\quad   \mu - \frac{\partial f}{\partial c} + \lambda \mathbf{\tilde{L}}(c) = 0. \label{eq:ch} 
\end{equation*}
Here, the double-free energy function is given by $f(c) = 100c^2(1-c^2)$. Sample initial conditions for each equation on the \pv{} meshes are shown in \cref{tab:sample} and \cref{tab:ch_sample} in \cref{sec:SampleInit}.

\subsection{MeshPDE Training Details} \label{sec:gegnn}

We extend the publicly available code base from \citet{park2023modelingdynamicsmeshesgauge} to train our baselines for the \pv{} and WeatherBench$2$ datasets: \url{https://github.com/jypark0/hermes/}. For GemCNN, EMAN, and Hermes, we use the already available pretrained model checkpoints. For GCN, \ours, Mesh Transformer, MPNN, and EGNN we train our own models.  We use the same train and test splits for the meshes as in \citet{park2023modelingdynamicsmeshesgauge}. Crucially, test meshes are completely held out during training.  We performed ablations over learning rate and latent space sizes. Following \citet{park2023modelingdynamicsmeshesgauge}  we keep models within a $\sim 40,000 - 50,000$ parameter budget. We note that this budget is relatively small, and that models that diverge in our experiments could potentially perform better under a more forgiving budget. All runs were performed on a single H$200$ GPU \cite{nvidia2025h200}. We use the previous $5$ time steps as input node feature vectors and backpropagate through $3$ steps of auto-regressive inference.

Hyperparameters are given in our configuration files on \href{https://github.com/EdwardBerman/rayleigh_analysis}{GitHub}. Defaults are taken from \citet{park2023modelingdynamicsmeshesgauge} if provided and otherwise optimized via grid search. Considered hyperparameters include learning rate, optimizer, training epochs, latent size, and skip connections. We also consider z-scoring of normed edge lengths for EGNN, different smoothness-breaking heads for \ours, and the number of clusters for the mesh transformer. 

As observed in \citet{park2023modelingdynamicsmeshesgauge}, we notice that residual connections can be key for performance with the gauge equivariant models. For \ours, we found that using a MLP readout with sinusoidal activation functions was a key ingredient for strong performance on MeshPDE. This supports previous work on how to train GNNs for long range tasks \citep{tonshoffdid, tonshoffdid_tmlr}. However, the GCN head exhibited the best performance on WB$2$.

\subsection{MeshPDE Evaluation Details} \label{sec:eval_details}

In order to aggregate smoothness errors over all time steps, we introduce a new metric. Define the Rayleigh Error (RE) by $\int_0^\infty | R_\mathcal{M}(\mathbf{Y_t}) - R_\mathcal{M}(f(\mathbf{X_t})) |dt$. In practice we approximate this by summing over the time steps where we are able to perform inference and normalize to the max timestep: 
\begin{equation*}
\text{RE}(f) = \frac{1}{\mathbf{T_{\rm max}}}\sum_{t} ^{\mathbf{T_{\rm max}}}| R_\mathcal{M}(\mathbf{Y_t}) - R_\mathcal{M}(f(\mathbf{X_t})) |.   
\end{equation*}
Following \citet{janny2023eagle} and \citet{pandya2025iaemu}, we also consider the scale invariant metrics NRMSE and SMAPE averaged over the entire rollout:
\begin{align*}
 \text{NRMSE}(f) &= \frac{1}{\mathbf{T_{\rm max}}}\sum_{t}^{\mathbf{T_{\rm max}}} \sqrt{\frac{\frac{1}{n} \sum_{i=1}^n\left(f(\mathbf{X_t})_i - (\mathbf{Y_t})_i\right)^2}{\frac{1}{n}\sum_{i=1}^n (\mathbf{Y_t})_i^2}}\\
 \text{SMAPE}(f) &= \frac{1}{\mathbf{T_{\rm max}}}\sum_{t}^{\mathbf{T_{\rm max}}} \frac{1}{n}\overset{n}{\underset{i=1}{\sum}}\frac{2|(\mathbf{Y_t})_i - f(\mathbf{X_t})_i|}{|(\mathbf{Y_t})_i| + |f(\mathbf{X_t})_i| + \varepsilon}
\end{align*}
where $\varepsilon = 10^{-8}$ is a stability constant. SMAPE is generally more robust than NMRSE in that it is less sensitive to outliers, but it is also more sensitive to small values. The scale invariant property of these metrics is crucial especially for heat diffusivity because solutions tend to decrease proportionally to $e^{-t}$. Thus, we need to consider deviations across several orders of magnitude in order to see how accurately we are modeling the decay.

\subsection{MeshPDE Qualitative Diagnostics} \label{sec:diagnostics}

In this section, we validate the superior performance of \ours{} on solving the heat equation with qualitative diagnostics. In \cref{tab:diagnostic} we show that \ours{} is the best at capturing the true smoothness of an unseen mesh during each step of the rollout. In our \href{https://github.com/EdwardBerman/rayleigh_analysis/tree/main/sup_mat}{GitHub repository} we include a video corresponding to the rollout in \cref{tab:diagnostic} over all timesteps.

\begin{table*}[!htb]
    \centering
    \begin{tabular}{ccccc}
    Time  & Truth & \ours{} \textbf{(Ours)} & EMAN & Hermes \\ \toprule 
$10$ &       \includegraphics[width=0.2\textwidth]{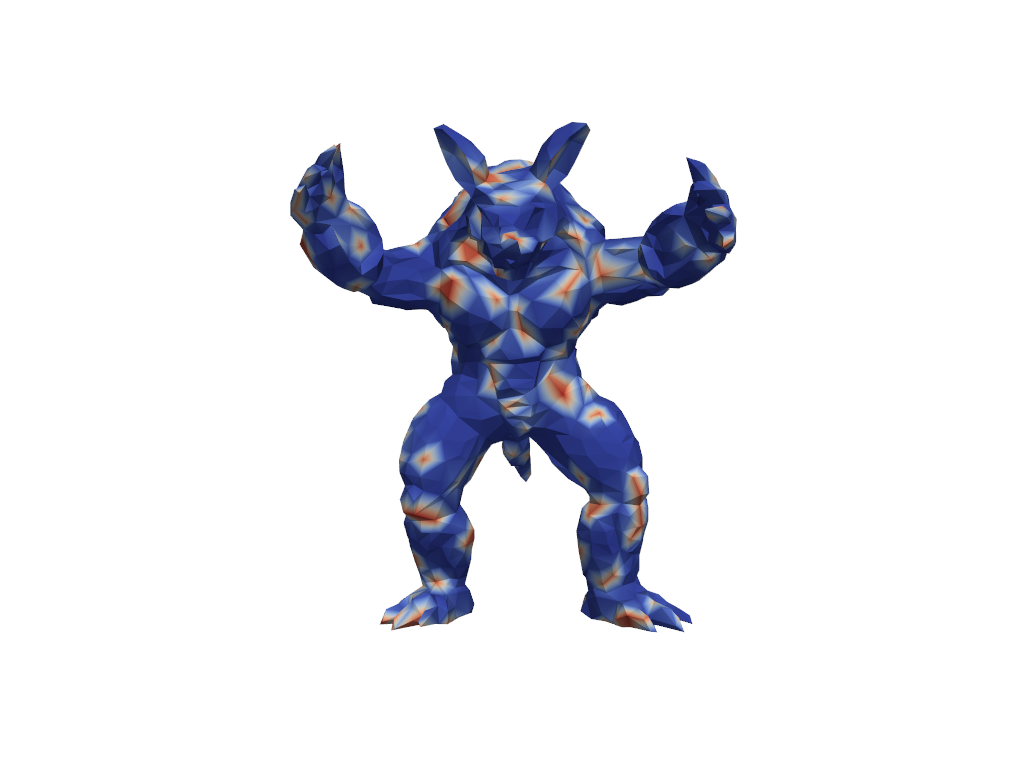}  &  \includegraphics[width=0.2\textwidth]{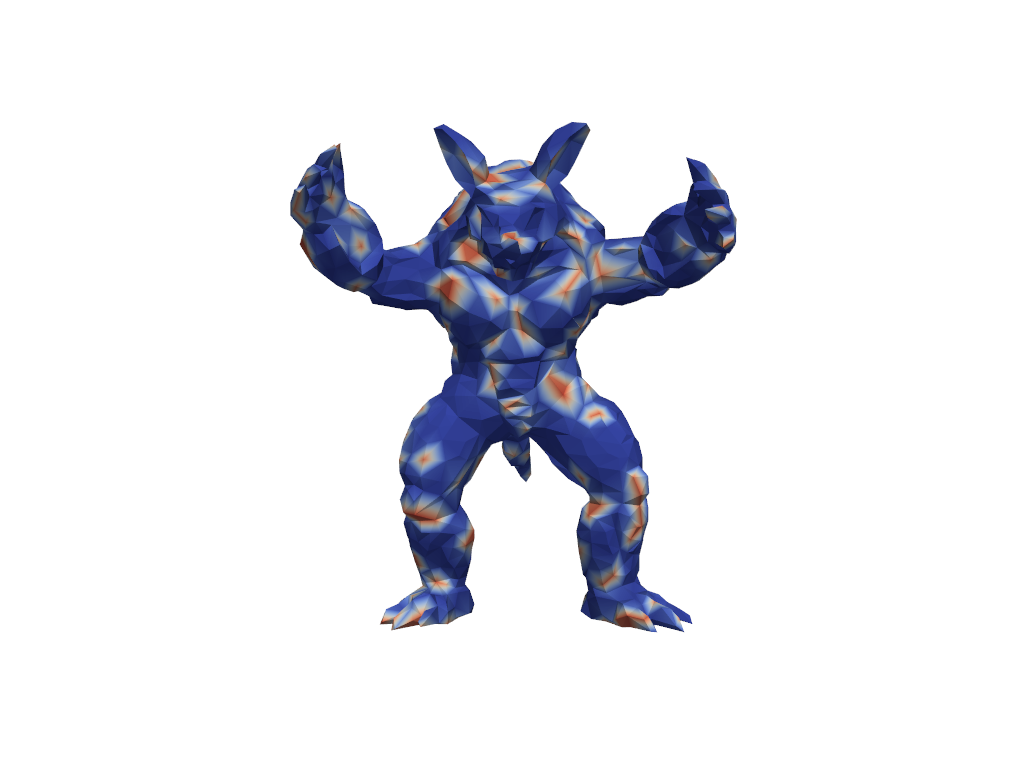}  & \includegraphics[width=0.2\textwidth]{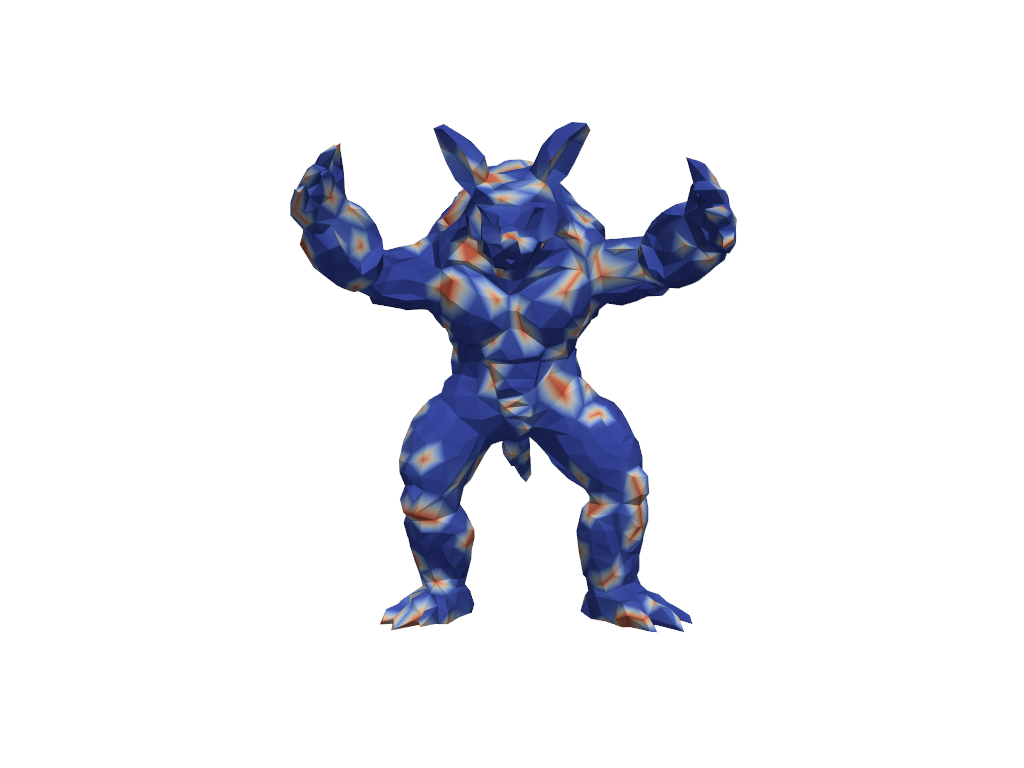} &  \includegraphics[width=0.2\textwidth]{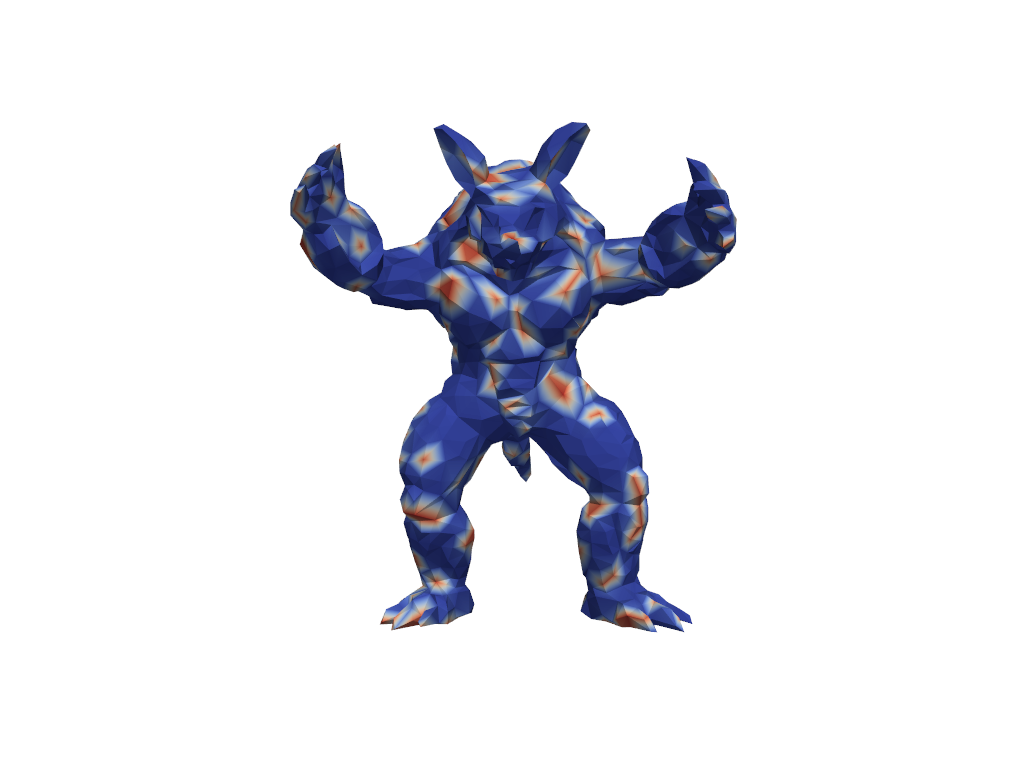} \\ \midrule
 $50$ &     \includegraphics[width=0.2\textwidth]{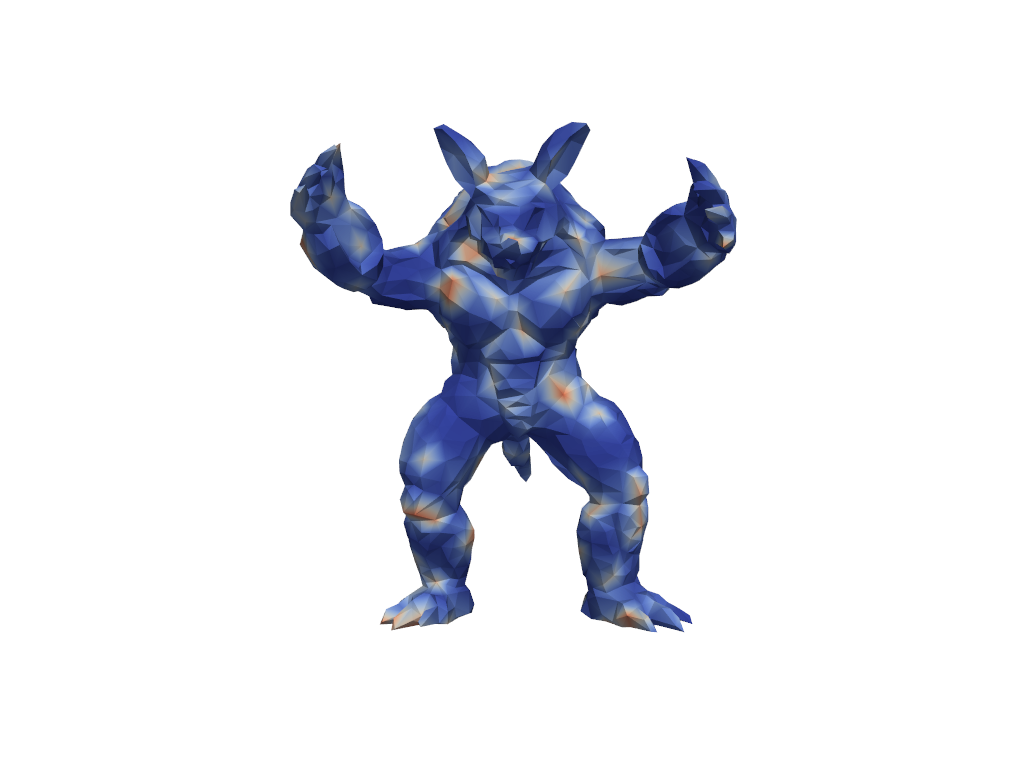}  &  \includegraphics[width=0.2\textwidth]{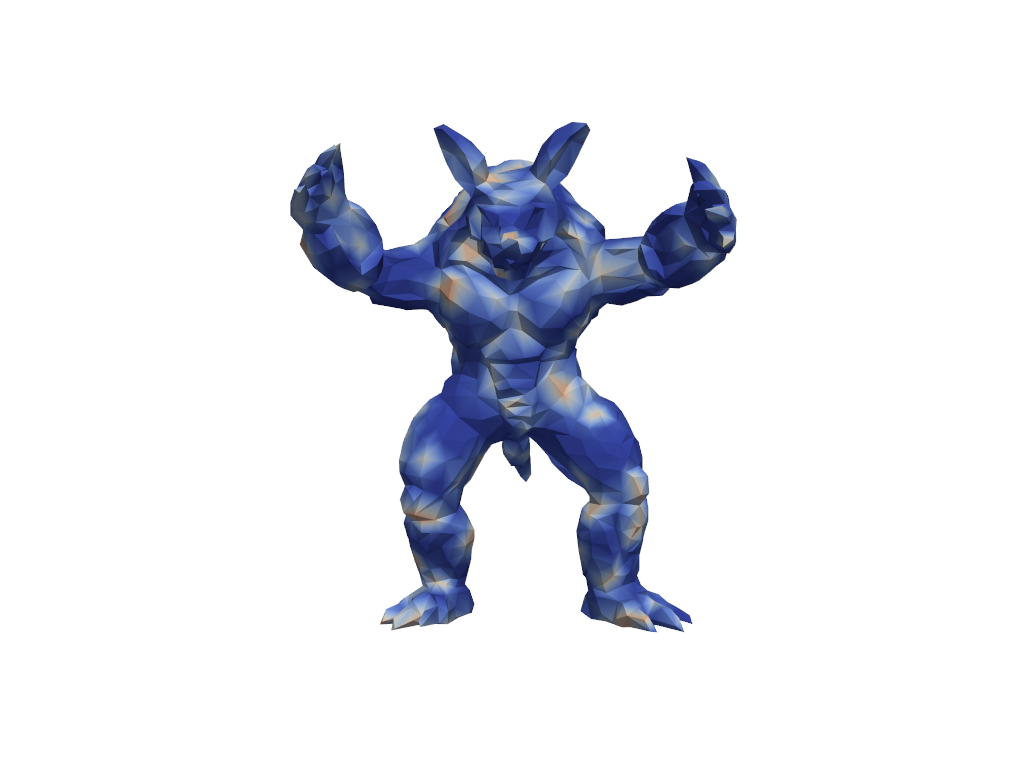}  & \includegraphics[width=0.2\textwidth]{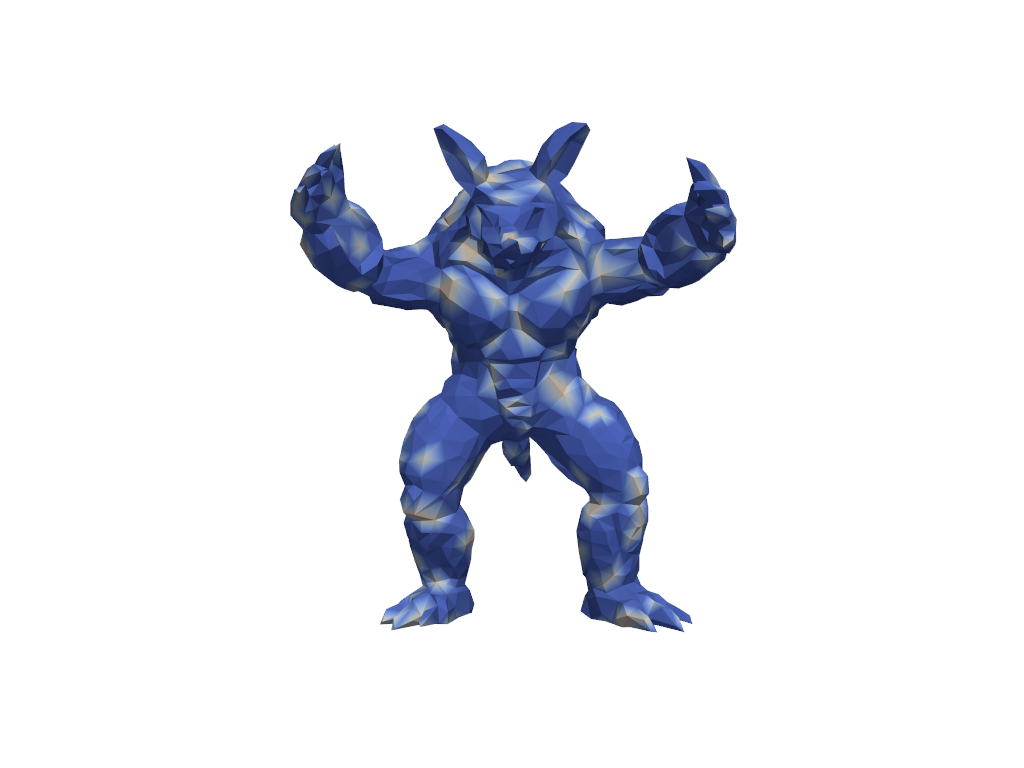} & \includegraphics[width=0.2\textwidth]{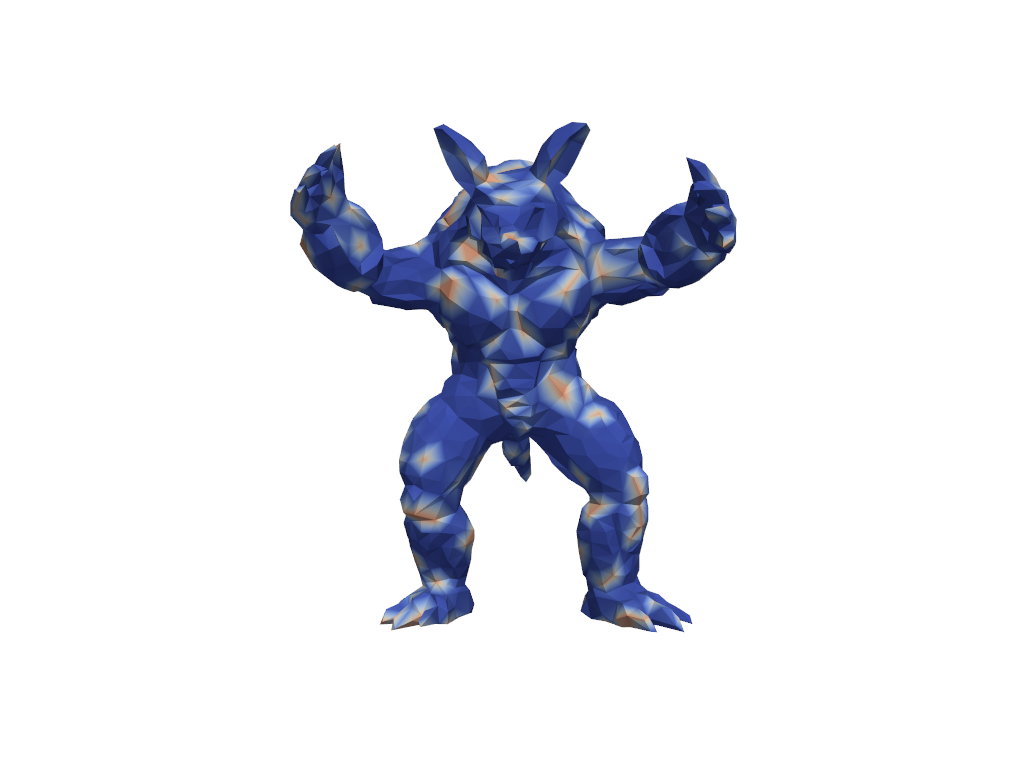} \\ \midrule
 $100$ &     \includegraphics[width=0.2\textwidth]{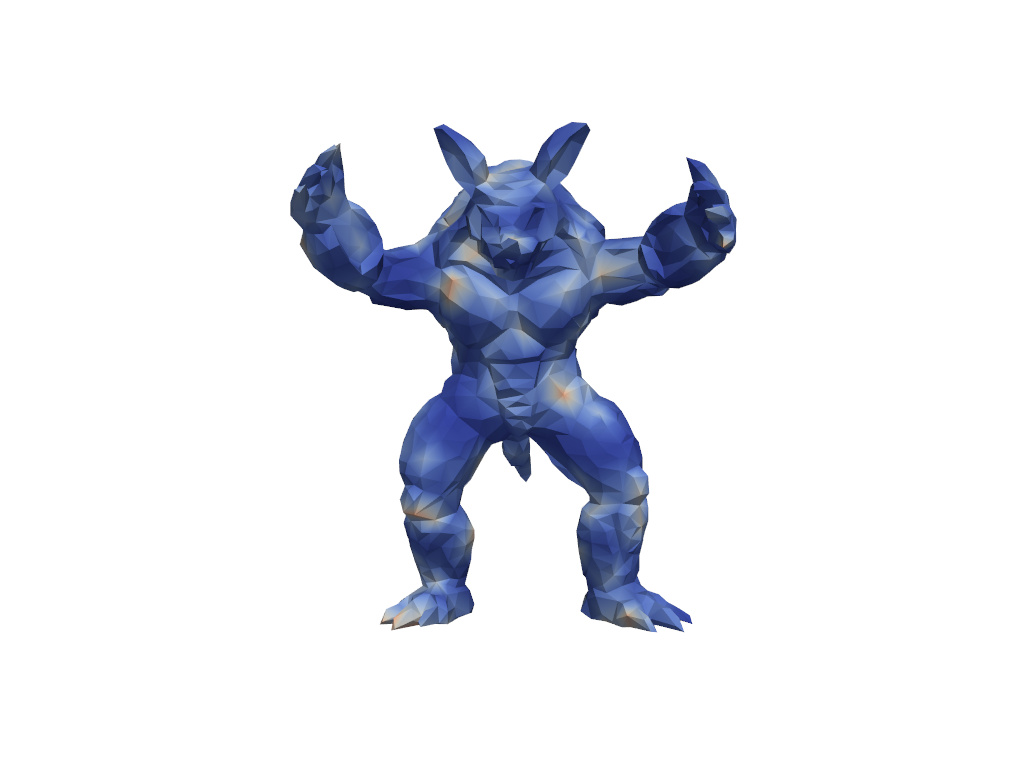}  &  \includegraphics[width=0.2\textwidth]{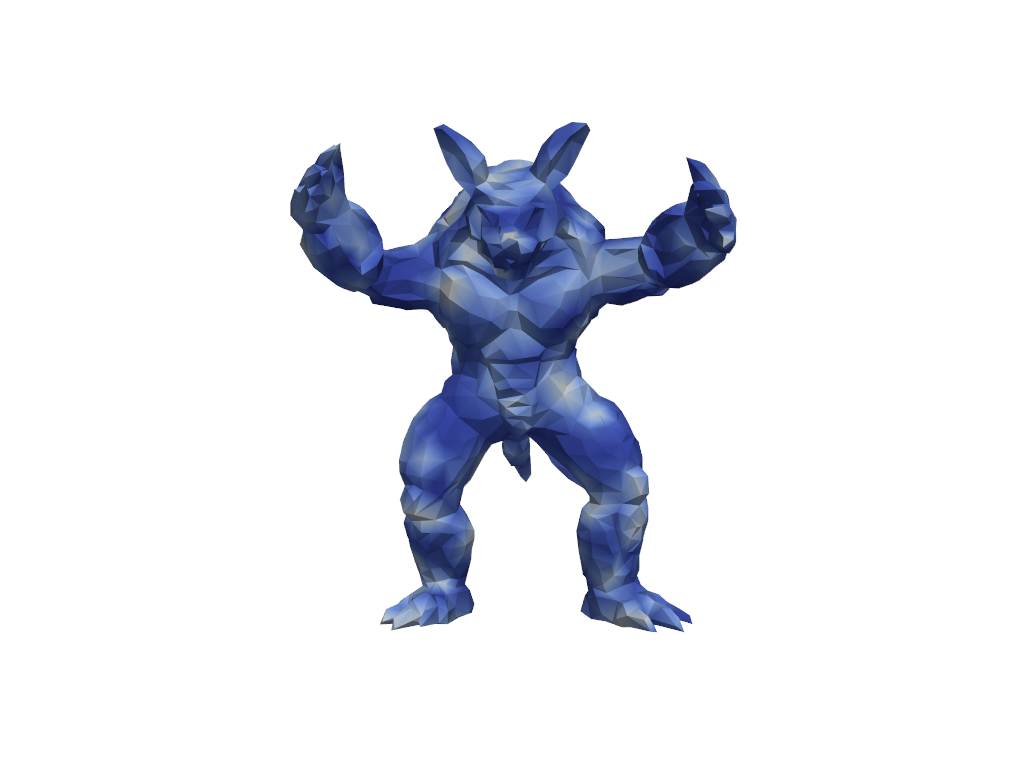}  & \includegraphics[width=0.2\textwidth]{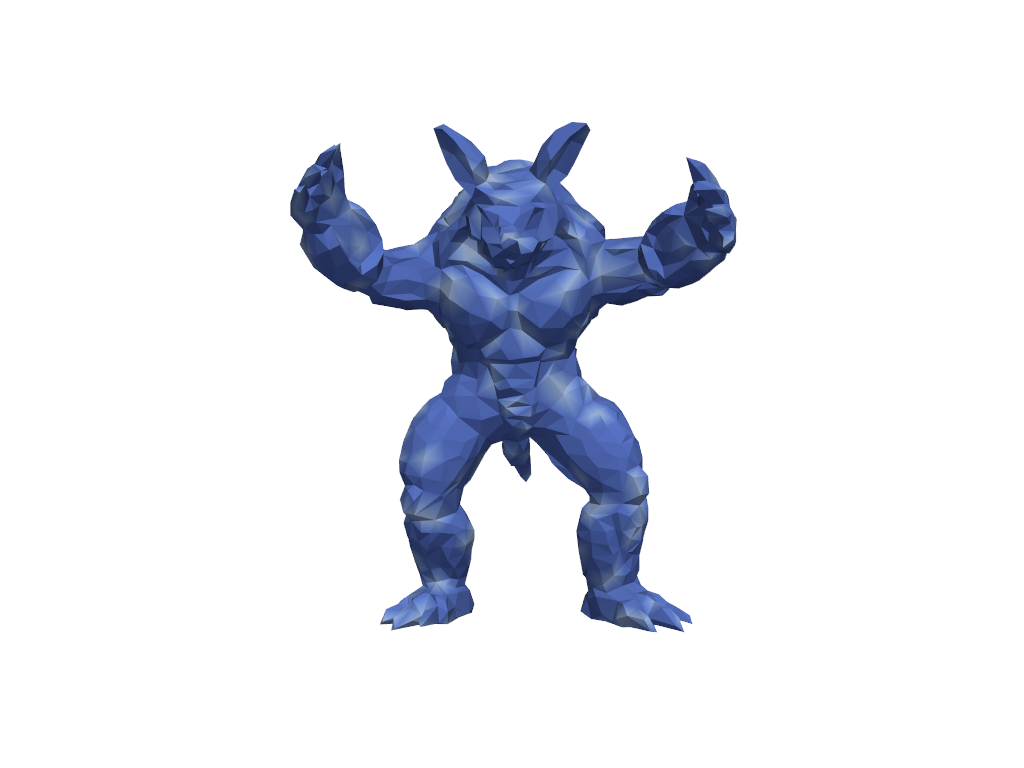} & \includegraphics[width=0.2\textwidth]{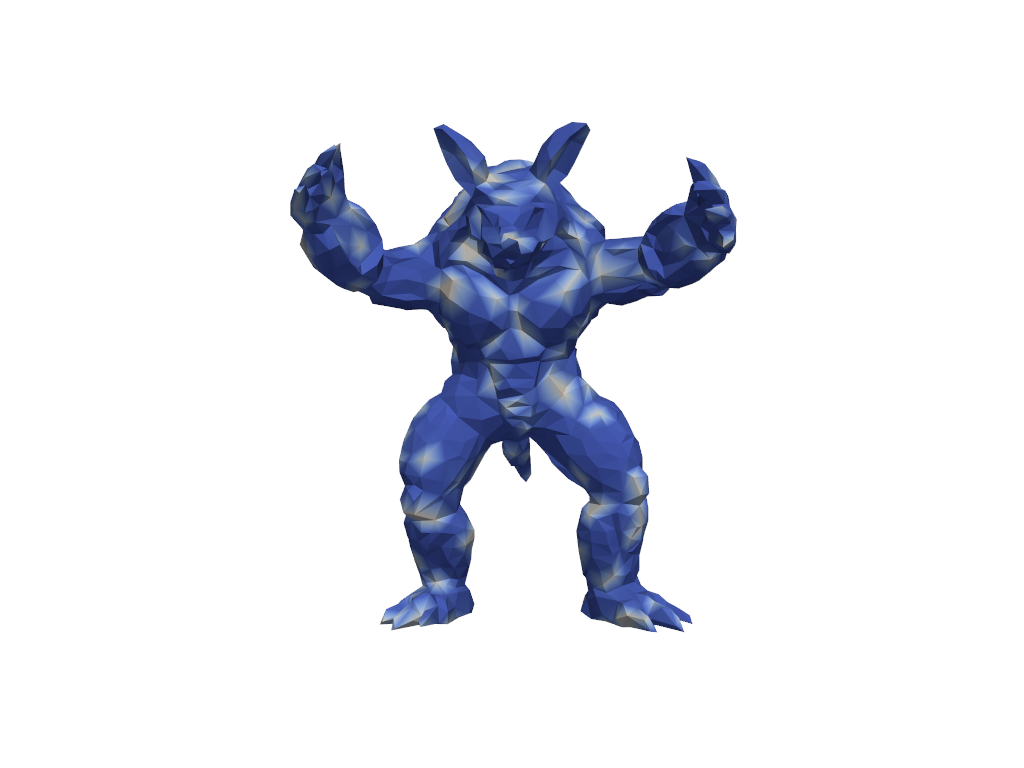} \\ \midrule
 $150$ &     \includegraphics[width=0.2\textwidth]{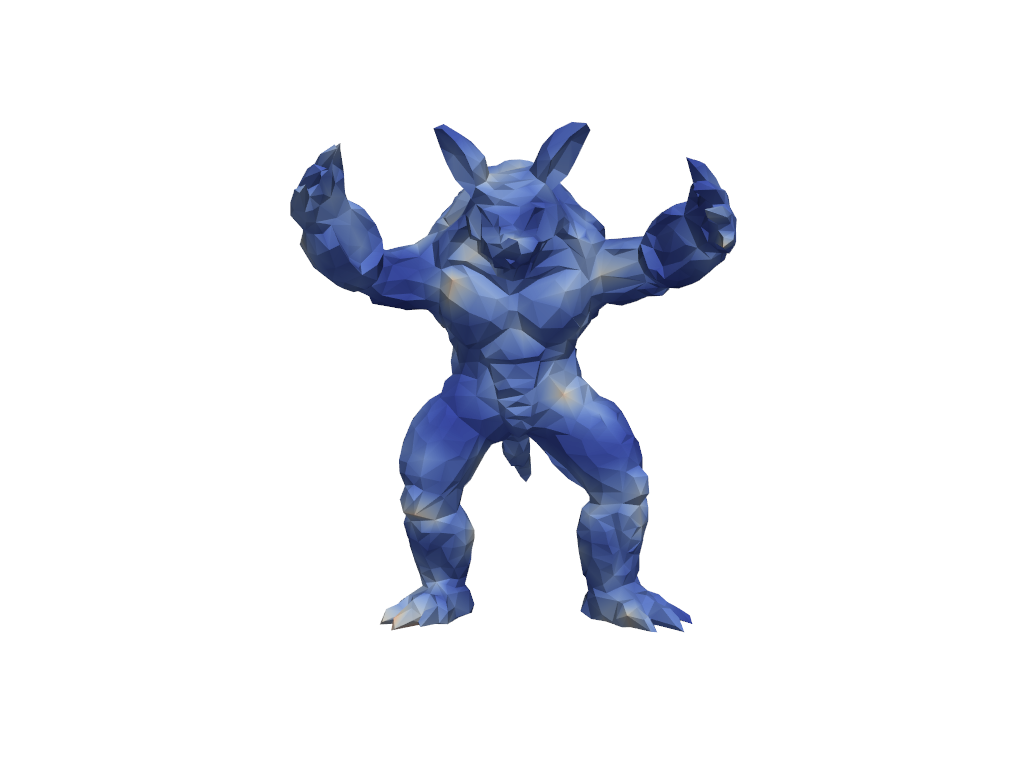}  &  \includegraphics[width=0.2\textwidth]{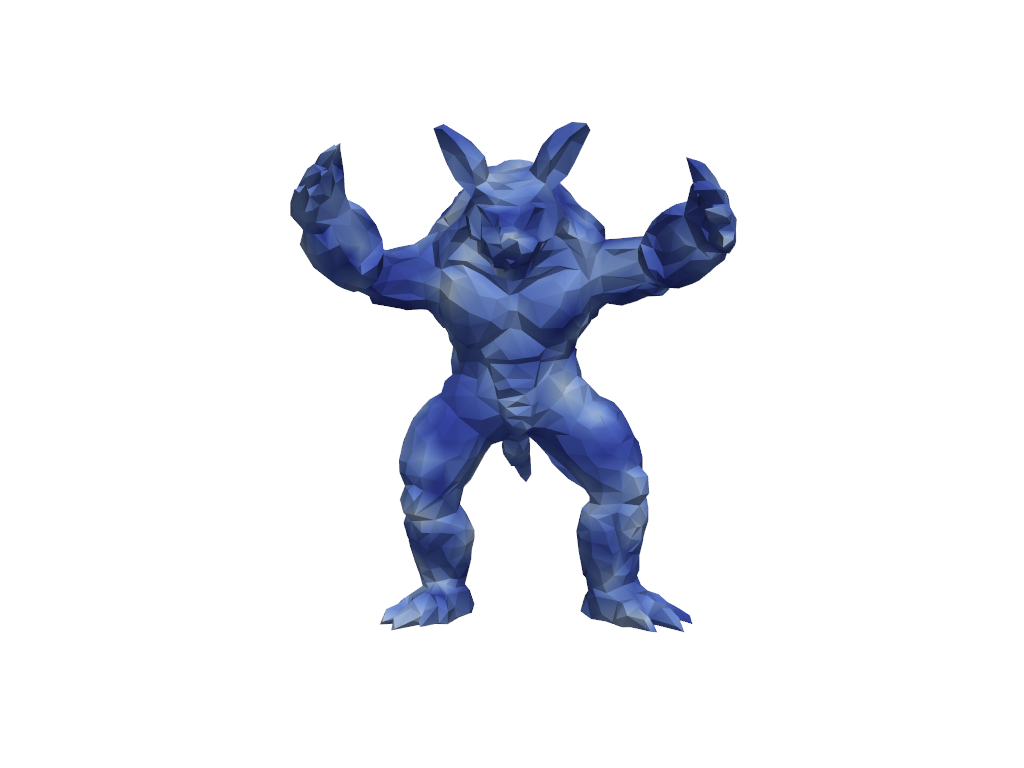}  & \includegraphics[width=0.2\textwidth]{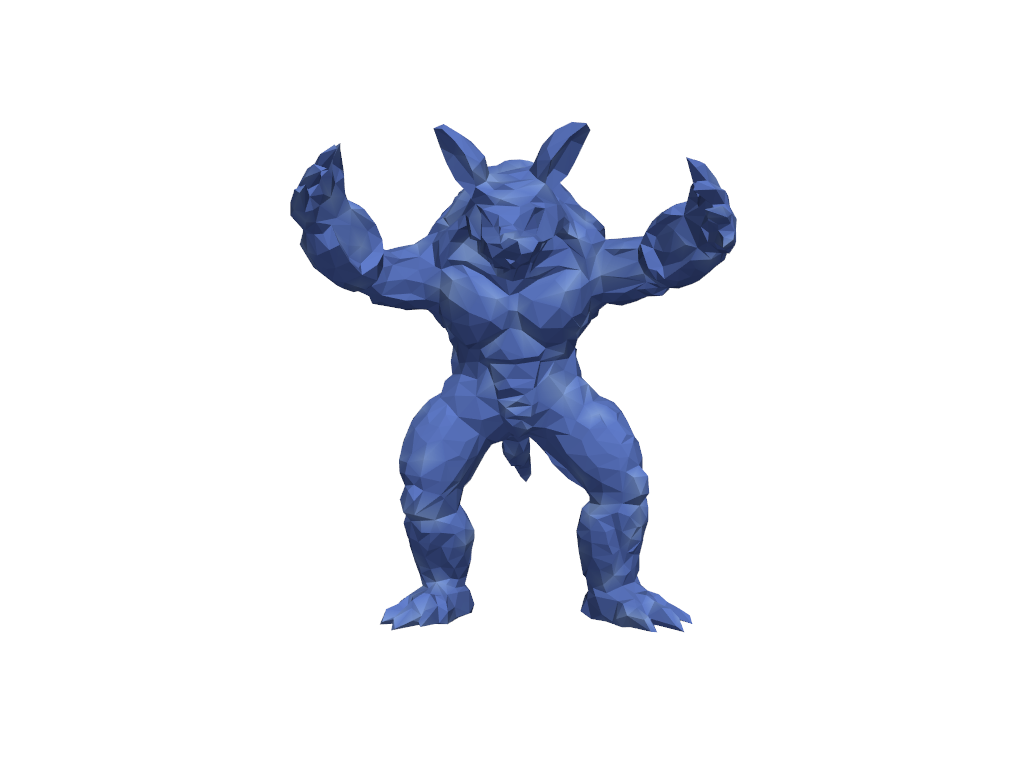} & \includegraphics[width=0.2\textwidth]{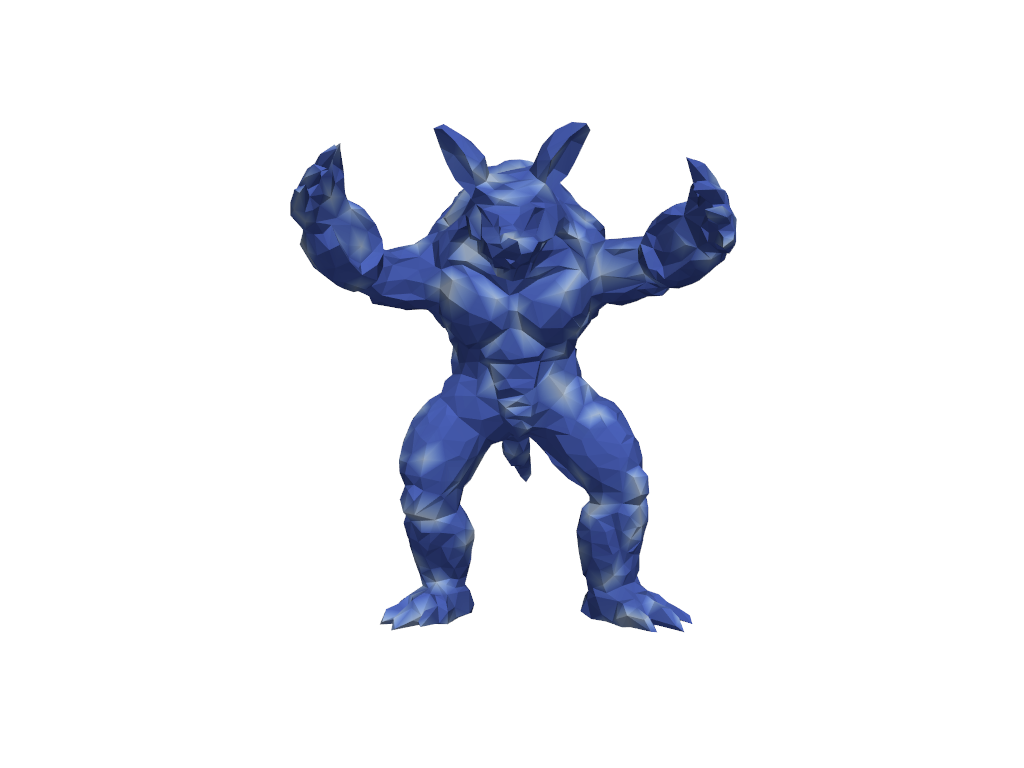} \\ \midrule
 $190$ &     \includegraphics[width=0.2\textwidth]{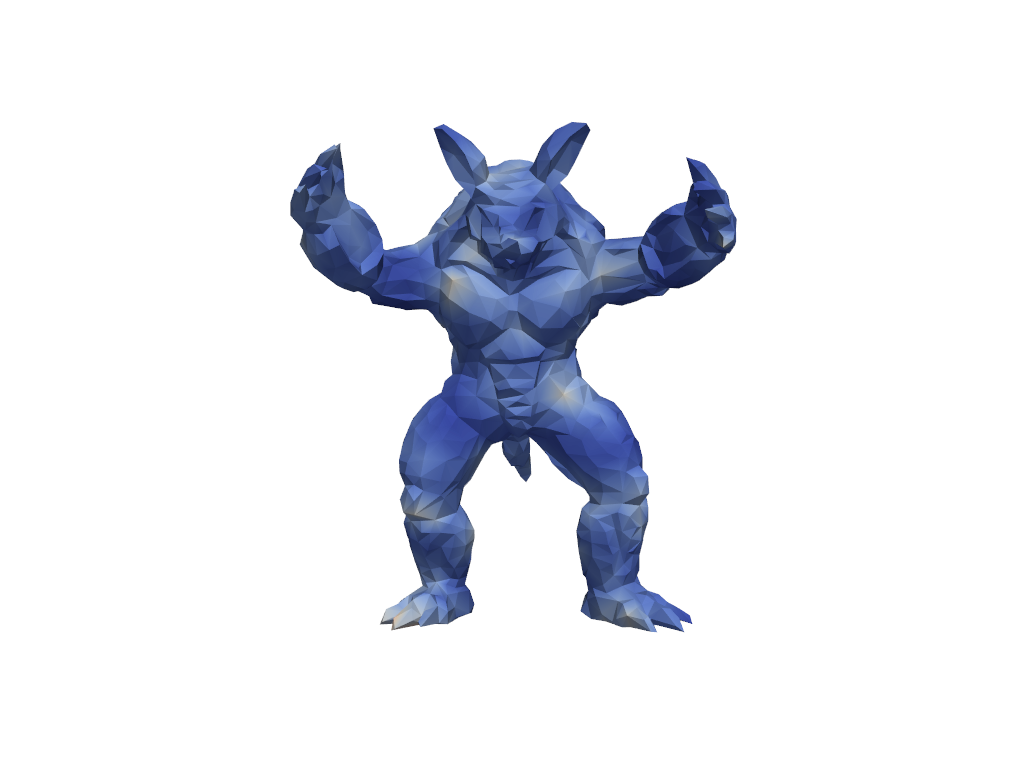}  &  \includegraphics[width=0.2\textwidth]{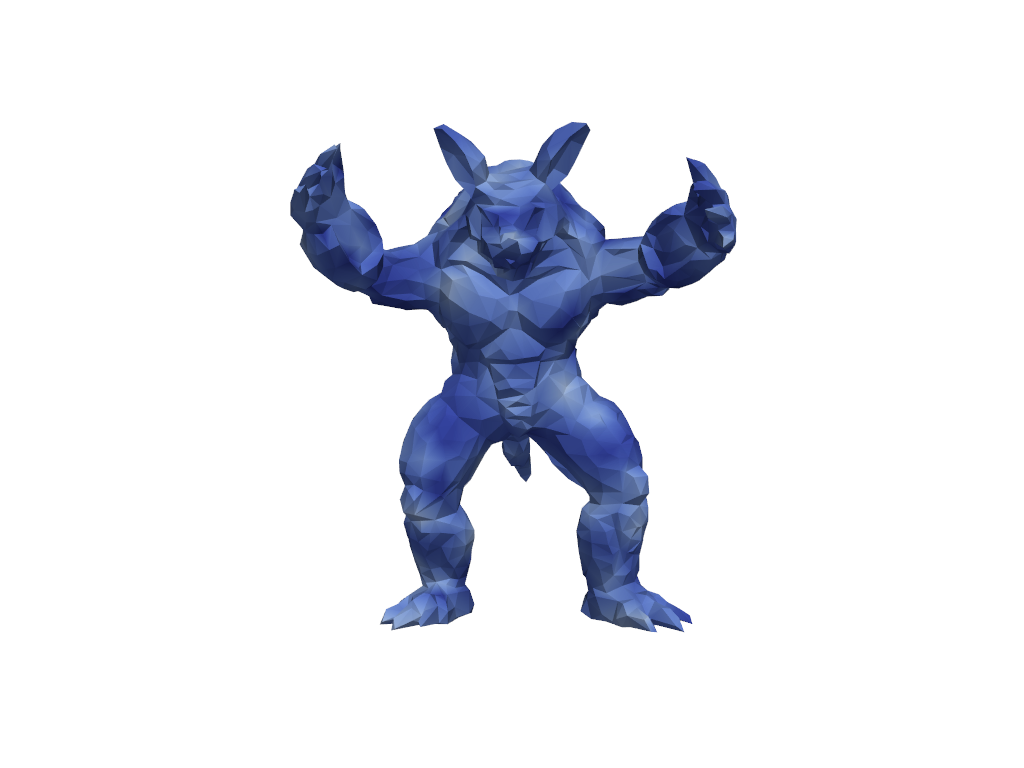}  & \includegraphics[width=0.2\textwidth]{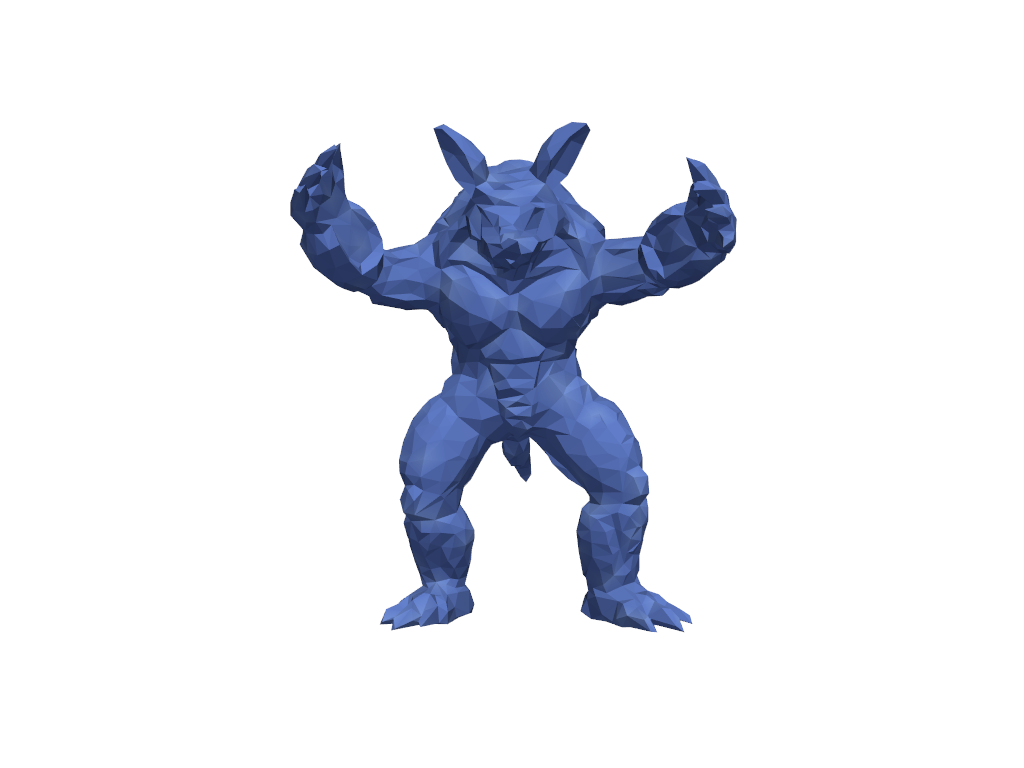} & \includegraphics[width=0.2\textwidth]{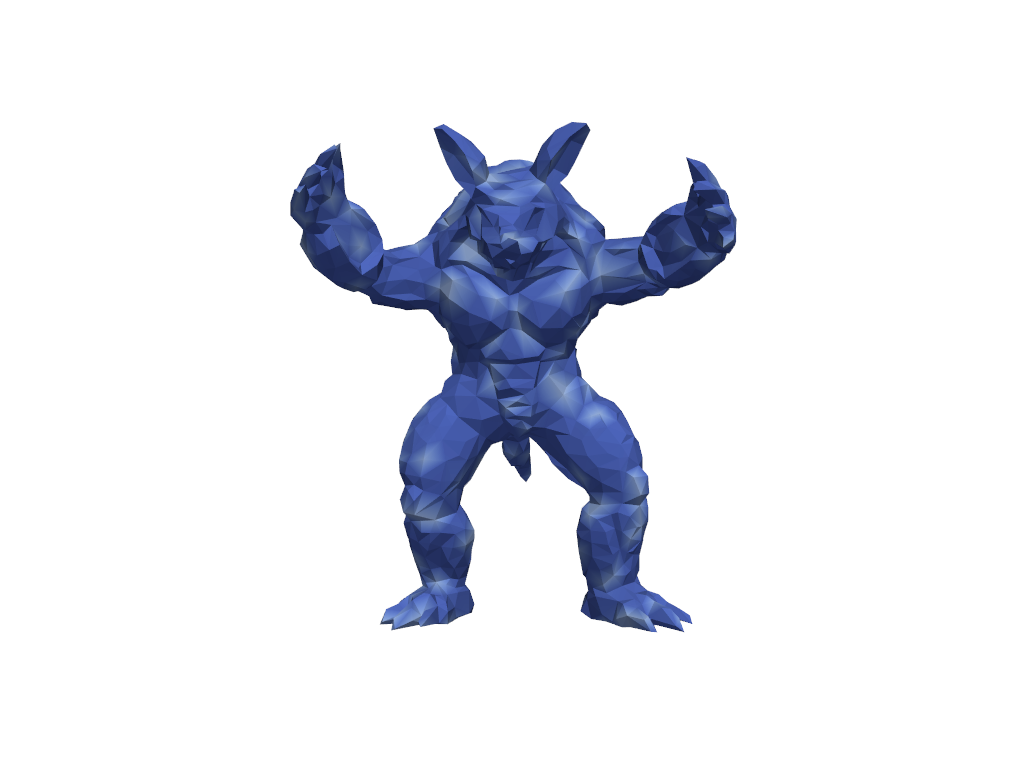}\\ \bottomrule
    \end{tabular}
    \caption{Qualitative comparison of model performance for the heat equation on the armadillo mesh. Our \ours{} model remains faithful to the ground truth during each step of the rollout, whereas the EMAN model over smooths and the Hermes model under smooths.}
    \label{tab:diagnostic}
\end{table*}

\subsection{MeshPDE Multiscale Smoothness Analysis} \label{sec:cf}

Since the Rayleigh quotient is a $1-$hop metric, this section performs additional comparisons on MeshPDE with a multiscale smoothness metric and finds that our $1-$hop smoothness tendencies also hold more generally for the gauge equivariant models we study. In particular, we define smoothness according to the $2-$point correlation function. Let $\delta : \mathbb{R}^3 \to \mathbb{R}$ be a function that maps a point $\mathbf{x}$ on a mesh to the scalar solution $u(\mathbf{x})$ (or approximation thereof) to the PDE at that point. The smoothness is then defined by the $2-$point correlation function $\xi$ given in \cref{eq:cor}:
\begin{equation}
    \xi (r; \delta) = \mathbb{E}\left[\delta(\mathbf{x})\delta(\mathbf{x} +r)\right]. \label{eq:cor}
\end{equation}

Intuitively, if node features are similar at a distance of $r$ apart, the correlation will be high. This allows us to study smoothness beyond $1-$hop neighbors by considering larger $r$. \cref{fig:kk} shows an example correlation function for Hermes at a given time step. We note that this characterization of smoothness is common in the weak gravitational lensing literature for point-cloud datasets \citep{schneider2006weak} and are easily computed with the \texttt{TreeCorr} library \citep{jarvis2004skewness}. 

Let $\delta_{ij}$ be the scalar field for the ground truth on a mesh $\mathcal{M}_i$ at time step $j$ and $\widehat{\delta_{ij}}$ be the approximation thereof. We define our smoothness error by 
\begin{equation*}
\text{err}_{\rm smooth}(\widehat{\delta_{ij}}) = \frac{1}{r_{\rm bins}}\frac{1}{\mathbf{T_{\rm max}}}\frac{1}{N}\underset{i=1}{\overset{N}{{\sum}}}\overset{\mathbf{T_{\rm max}}}{\underset{j=1}{\sum}}\overset{r_{\rm bins}}{{\underset{k=1}{\sum}}}| \xi (r_k ; \delta_{ij}) - \xi (r_k ; \widehat{\delta_{ij}})|.   
\end{equation*}
We note that the correlation function is related to the Fourier space power spectrum $P(k)$ by 
\begin{equation}
    \xi(r) = \frac{1}{2\pi^2}\int k^2 P(k) \frac{\sin (kr)}{kr}dk \label{eq:ps}. 
\end{equation}
Thus, \cref{eq:ps} informs us that our metric for smoothness as a function of $r$ is related to traditional energy spectrum errors \citep[e.g.,][]{wang2020incorporating}. We leave a more systematic comparison between the measures as an opportunity for future work.

As seen in \cref{tab:hope}, the more expressive attention and message passing based models are much better at capturing the underlying smoothness. The CNN model diverges for the heat and wave datasets, but performs reasonably well on Cahn-Hilliard. This is mirrored by our results in \cref{tab:IRE} in the main text for the Gauge Equivariant models.

\begin{figure}[!htb]
    \centering
    \includegraphics[width=0.5\linewidth]{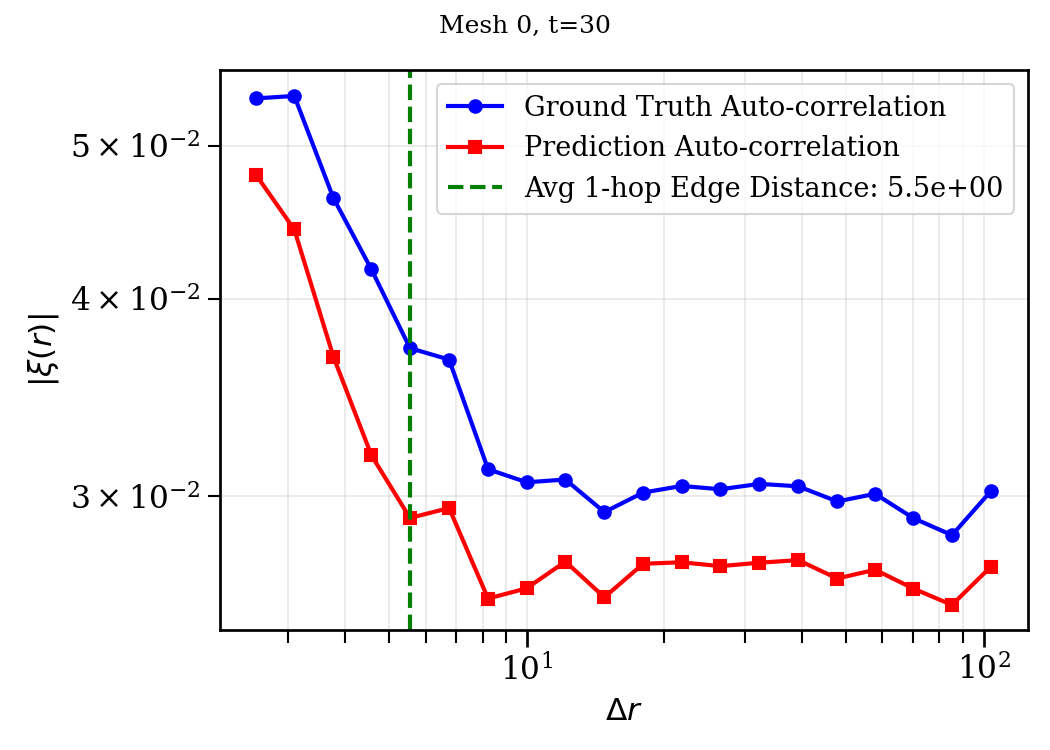}
    \caption{Smoothness for a Hermes model as measured by the $2-$point correlation function. The plot indicates undersmoothing in each radial bin.}
    \label{fig:kk}
\end{figure}

\begin{table}[!htb]
    \centering
    \tiny
    \begin{tabular}{l|c}
        \toprule
        \multicolumn{2}{c}{\textbf{Heat ($\alpha = 1$)}} \\ \cmidrule(lr){1-2}
        \textbf{Model} & $\text{err}_{\rm smooth}$ $(\downarrow)$ \\
        \midrule
    \hyperlink{cite.dehaan2021gaugeequivariantmeshcnns}{GemCNN} & -- \\
        \hyperlink{cite.basu2022equivariant}{EMAN} & $\mathbf{4.04 \times 10^{-3}}$ \\
        \hyperlink{cite.park2023modelingdynamicsmeshesgauge}{Hermes} & $9.71 \times 10 ^{-3}$ \\
        \midrule
        \multicolumn{2}{c}{\textbf{Wave ($c = 1$)}} \\ \cmidrule(lr){1-2}
        \textbf{Model} & $\text{err}_{\rm smooth}$ $(\downarrow)$ \\
        \midrule
        \hyperlink{cite.dehaan2021gaugeequivariantmeshcnns}{GemCNN}  & -- \\
        \hyperlink{cite.basu2022equivariant}{EMAN} & $\mathbf{1.78 \times 10^{-3}}$ \\
\hyperlink{cite.park2023modelingdynamicsmeshesgauge}{Hermes}  & $1.38 \times 10^{-2}$ \\
        \midrule
        \multicolumn{2}{c}{\textbf{Cahn--Hilliard}} \\ \cmidrule(lr){1-2}
        \textbf{Model} & $\text{err}_{\rm smooth}$ $(\downarrow)$ \\
        \midrule
    \hyperlink{cite.dehaan2021gaugeequivariantmeshcnns}{GemCNN} & $1.89 \times 10^{-1}$ \\
        \hyperlink{cite.basu2022equivariant}{EMAN} & $4.59 \times 10^{-1}$ \\
    \hyperlink{cite.park2023modelingdynamicsmeshesgauge}{Hermes} & $\mathbf{9.61 \times 10^{-3}}$ \\
        \bottomrule
    \end{tabular}
    \caption{$\text{err}_{\rm smooth}$ for Gauge Equivariant models on the MeshPDE datasets. Dashes (--) indicate non-convergence. Best performing model is indicated with \textbf{bold text}.}
    \label{tab:hope}
\end{table}

\subsection{MeshPDE Local Smoothness Analysis} \label{sec:local}

This section contains analysis for local smoothness tendencies not captured by global smoothness metrics like the Rayleigh quotient. This is particularly important for the Cahn-Hilliard PDE, where the solution contains sharp discontinuities. To analyze local smoothness, we define smoothness at a given vertex as the Rayleigh quotient of the graph given by the $2$-hop neighborhood around that vertex. To quantify how well local smoothness of a surrogate model matches the target, we examine the distribution of Rayleigh quotients for each timestep. The evolving distributions during each timestep of the PDE are available as gifs on our GitHub codebase at \url{https://github.com/EdwardBerman/rayleigh_analysis/tree/main/rebuttal/localized_metric}. On Cahn-Hilliard, we see that \ours{} is able to capture the local smoothness tendencies on some trajectories, albeit less consistently than GemCNN. This explains the performance in \cref{tab:IRE}.

\subsection{MeshPDE Operator Learning Method Comparison} \label{sec:opppp}

We compare \ours{} with MeshGraphNet, a strong operator learning method baseline and show that we outperform on diffusive dynamics. This is shown in \cref{tab:IRE_opp}.

\begin{table*}[!htb]
\centering
\tiny
\begin{tabular}{lcc}
\toprule
\textbf{Metric} & \ours{} & \hyperlink{pfaff2021learningmeshbasedsimulationgraph}{MeshGraphNet} \\
\midrule
\multicolumn{3}{c}{\textbf{Heat} $(\alpha = 1)$} \\ \midrule
NRMSE $(\downarrow)$ & $\mathbf{51.9 \pm 3.6}$ & $108.1 \pm 3.2$\\
SMAPE $(\downarrow)$ & $\mathbf{79.7 \pm 5.6}$ & $144.1 \pm 6.3$ \\
RE $(\downarrow)$    & $\mathbf{9.1 \pm 7.4}$ & $53.4 \pm 1.6$ \\
\midrule
\multicolumn{3}{c}{\textbf{Wave} $(c = 1)$} \\ \midrule
NRMSE $(\downarrow)$ & $\mathbf{236.5 \pm 6.4}$ & $253.6 \pm 28.2$ \\
SMAPE $(\downarrow)$ & $385.2 \pm 1.2$ & $\mathbf{278.6 \pm 2.3}$ \\
RE $(\downarrow)$    & $93.5 \pm 25.4$ & $\mathbf{30.0 \pm 23.2}$\\
\midrule
\multicolumn{3}{c}{\textbf{Cahn-Hilliard}} \\ \midrule
NRMSE $(\downarrow)$ & $123.9 \pm 2.6$ & $\mathbf{122.0 \pm 1.5}$ \\
SMAPE $(\downarrow)$ & $\mathbf{167.3 \pm 10.6}$ & $186.1 \pm 2.9$ \\
RE $(\downarrow)$  & $18.9 \pm 10.4$ & $\mathbf{6.2 \pm 0.4}$ \\
\bottomrule
\end{tabular}
\caption{NRMSE, SMAPE, and RE averaged over all rollouts on all test meshes for the heat, wave, and Cahn-Hilliard equations. The best values are in bold. \ours{} outperforms MeshGraphNet on heat and is competitive on the other PDEs.}
\label{tab:IRE_opp}
\end{table*}

\subsection{MeshPDE Ablation Study} \label{sec:abl}

We provide ablations on each of the main components in \ours{} and show that are proposed relaxation mechanism is necessary for faithful modeling. Most importantly, \cref{tab:pde_encoder_comparison} shows that unitarity graph convolutions are necessary for strong performance on diffusive dynamics. While our method can tolerate early Taylor truncations in the unitary layers of the preserver network $P$, replacing the unitary layers with fully unconstrained GCN layers results in significantly worse performance. Similarly, in \cref{tab:zeropad_gcn_comparison} we show that the smoothness preserving zero pad operation is necessary for strong performance across the heat task. However, we see in \cref{tab:groupsort_relu_comparison} that our method is not overly sensitive to our usage of the GroupSort activation function. Finally, we note that our method was not overly sensitive to the width and depth of the decoding MLP used in our experiments. As seen in \cref{tab:depth} and \cref{tab:width}, while our results were best for an MLP with $3$ layers and a hidden width of size $64$, performance was still competitive with many baselines for a single readout layer or with a width of $128$.

\begin{table*}[!htb]
\centering
\tiny
\begin{tabular}{lccc}
\toprule
\textbf{Metric} & \textbf{Unitary GCN} ($T=10$) & \textbf{Unitary GCN} ($T=3$) & \textbf{Unconstrained GCN} \\
\midrule
\multicolumn{4}{c}{\textbf{Heat}} \\ \midrule
NRMSE $(\downarrow)$ & $\mathbf{51.9 \pm 3.6}$  & $70.1 \pm 6.3$           & $416.9 \pm 14.3$ \\
SMAPE $(\downarrow)$ & $79.7 \pm 5.6$           & $\mathbf{76.7 \pm 10.4}$ & $248.8 \pm 6.6$  \\
RE $(\downarrow)$    & $9.1 \pm 7.4$            & $\mathbf{9.0 \pm 0.6}$   & $31.5 \pm 8.6$   \\
\midrule
\multicolumn{4}{c}{\textbf{Wave}} \\ \midrule
NRMSE $(\downarrow)$ & $236.5 \pm 6.4$          & $242.2 \pm 4.1$          & $\mathbf{196.0 \pm 0.0}$ \\
SMAPE $(\downarrow)$ & $385.2 \pm 1.2$          & $\mathbf{299.1 \pm 5.9}$ & $363.1 \pm 3.1$  \\
RE $(\downarrow)$    & $93.5 \pm 25.4$          & $\mathbf{37.7 \pm 4.8}$  & $193.6 \pm 4.3$  \\
\midrule
\multicolumn{4}{c}{\textbf{Cahn-Hilliard}} \\ \midrule
NRMSE $(\downarrow)$ & $\mathbf{123.9 \pm 2.6}$ & $130.1 \pm 2.2$          & $175.0 \pm 0.9$  \\
SMAPE $(\downarrow)$ & $167.3 \pm 10.6$         & $\mathbf{155.9 \pm 7.2}$ & $187.0 \pm 2.8$  \\
RE $(\downarrow)$    & $18.9 \pm 10.4$          & $24.2 \pm 8.5$           & $\mathbf{12.0 \pm 4.6}$ \\
\bottomrule
\end{tabular}
\caption{Comparison for the smoothness-preserving component type on NRMSE, SMAPE, and RE averaged over all rollouts on all test meshes for the heat, wave, and Cahn-Hilliard equations. The best values are in bold.}
\label{tab:pde_encoder_comparison}
\end{table*}

\begin{table*}[!htb]
\centering
\tiny
\begin{tabular}{lcc}
\toprule
\textbf{Metric} & \textbf{Zero Pad} & \textbf{GCN} \\
\midrule
\multicolumn{3}{c}{\textbf{Heat}} \\ \midrule
NRMSE $(\downarrow)$ & $\mathbf{51.9 \pm 3.6}$   & $4{,}689.6 \pm 232.0$ \\
SMAPE $(\downarrow)$ & $\mathbf{79.7 \pm 5.6}$   & $367.7 \pm 0.6$       \\
RE $(\downarrow)$    & $\mathbf{9.1 \pm 7.4}$    & $153.9 \pm 13.4$      \\
\midrule
\multicolumn{3}{c}{\textbf{Wave}} \\ \midrule
NRMSE $(\downarrow)$ & $\mathbf{236.5 \pm 6.4}$  & $12{,}532.0 \pm 242.2$ \\
SMAPE $(\downarrow)$ & $385.2 \pm 1.2$  & $\mathbf{383.9 \pm 0.6}$        \\
RE $(\downarrow)$    & $93.5 \pm 25.4$           & $\mathbf{11.8 \pm 2.8}$ \\
\midrule
\multicolumn{3}{c}{\textbf{Cahn-Hilliard}} \\ \midrule
NRMSE $(\downarrow)$ & $\mathbf{123.9 \pm 2.6}$  & $228.5 \pm 1.0$        \\
SMAPE $(\downarrow)$ & $\mathbf{167.3 \pm 10.6}$ & $250.8 \pm 1.3$        \\
RE $(\downarrow)$    & $18.9 \pm 10.4$           & $\mathbf{9.0 \pm 4.3}$ \\
\bottomrule
\end{tabular}
\caption{Comparison of Zero Pad and GCN mapping strategies for increasing the latent dimension on NRMSE, SMAPE, and RE averaged over all rollouts on all test meshes for the heat, wave, and Cahn-Hilliard equations. The best values are in bold.}
\label{tab:zeropad_gcn_comparison}
\end{table*}

\begin{table*}[!htb]
\centering
\tiny
\begin{tabular}{lcc}
\toprule
\textbf{Metric} & \textbf{GroupSort} & \textbf{ReLU} \\
\midrule
\multicolumn{3}{c}{\textbf{Heat}} \\ \midrule
NRMSE $(\downarrow)$ & $\mathbf{51.9 \pm 3.6}$   & $87.7 \pm 10.5$  \\
SMAPE $(\downarrow)$ & $\mathbf{79.7 \pm 5.6}$   & $82.9 \pm 3.6$   \\
RE $(\downarrow)$    & $9.1 \pm 7.4$             & $\mathbf{9.0 \pm 10.5}$ \\
\midrule
\multicolumn{3}{c}{\textbf{Wave}} \\ \midrule
NRMSE $(\downarrow)$ & $\mathbf{236.5 \pm 6.4}$  & $825.7 \pm 12.3$ \\
SMAPE $(\downarrow)$ & $385.2 \pm 1.2$  & $\mathbf{340.2 \pm 4.3}$  \\
RE $(\downarrow)$    & $93.5 \pm 25.4$           & $\mathbf{22.4 \pm 4.0}$ \\
\midrule
\multicolumn{3}{c}{\textbf{Cahn-Hilliard}} \\ \midrule
NRMSE $(\downarrow)$ & $\mathbf{123.9 \pm 2.6}$  & $144.7 \pm 4.2$  \\
SMAPE $(\downarrow)$ & $167.3 \pm 10.6$ & $\mathbf{151.2 \pm 4.5}$  \\
RE $(\downarrow)$    & $\mathbf{18.9 \pm 10.4}$  & $30.0 \pm 8.0$   \\
\bottomrule
\end{tabular}
\caption{Comparison of GroupSort and ReLU activation functions on NRMSE, SMAPE, and RE averaged over all rollouts on all test meshes for the heat, wave, and Cahn-Hilliard equations. The best values are in bold.}
\label{tab:groupsort_relu_comparison}
\end{table*}

\begin{table*}[!htb]
\centering
\tiny
\begin{tabular}{lcc}
\toprule
\textbf{Metric} & \textbf{MLP (Depth = 1)} & \textbf{MLP (Depth = 3)} \\
\midrule
\multicolumn{3}{c}{\textbf{Heat}} \\ \midrule
NRMSE $(\downarrow)$ & $99.9 \pm 13.1$          & $\mathbf{51.9 \pm 3.6}$   \\
SMAPE $(\downarrow)$ & $181.6 \pm 4.0$          & $\mathbf{79.7 \pm 5.6}$   \\
RE $(\downarrow)$    & $21.2 \pm 7.1$           & $\mathbf{9.1 \pm 7.4}$    \\
\midrule
\multicolumn{3}{c}{\textbf{Wave}} \\ \midrule
NRMSE $(\downarrow)$ & $-$                      & $236.5 \pm 6.4$           \\
SMAPE $(\downarrow)$ & $\mathbf{377.6 \pm 4.1}$ & $385.2 \pm 1.2$           \\
RE $(\downarrow)$    & $109.6 \pm 24.8$         & $\mathbf{93.5 \pm 25.4}$  \\
\midrule
\multicolumn{3}{c}{\textbf{Cahn-Hilliard}} \\ \midrule
NRMSE $(\downarrow)$ & $141.9 \pm 1.9$          & $\mathbf{123.9 \pm 2.6}$  \\
SMAPE $(\downarrow)$ & $233.7 \pm 7.8$          & $\mathbf{167.3 \pm 10.6}$ \\
RE $(\downarrow)$    & $25.5 \pm 5.2$           & $\mathbf{18.9 \pm 10.4}$  \\
\bottomrule
\end{tabular}
\caption{Effect of MLP depth on NRMSE, SMAPE, and RE averaged over all rollouts on all test meshes for the heat, wave, and Cahn-Hilliard equations. The best values are in bold.}
\label{tab:depth}
\end{table*}

\begin{table*}[!htb]
\centering
\tiny
\begin{tabular}{lcc}
\toprule
\textbf{Metric} & \textbf{MLP (Width = 64)} & \textbf{MLP (Width = 128)} \\
\midrule
\multicolumn{3}{c}{\textbf{Heat}} \\ \midrule
NRMSE $(\downarrow)$ & $\mathbf{51.9 \pm 3.6}$   & $88.2 \pm 6.8$              \\
SMAPE $(\downarrow)$ & $\mathbf{79.7 \pm 5.6}$   & $136.3 \pm 4.7$             \\
RE $(\downarrow)$    & $\mathbf{9.1 \pm 7.4}$    & $10.7 \pm 2.3$              \\
\midrule
\multicolumn{3}{c}{\textbf{Wave}} \\ \midrule
NRMSE $(\downarrow)$ & $\mathbf{236.5 \pm 6.4}$  & $1{,}433.0 \pm 35.0$        \\
SMAPE $(\downarrow)$ & $385.2 \pm 1.2$           & $\mathbf{343.4 \pm 7.7}$    \\
RE $(\downarrow)$    & $\mathbf{93.5 \pm 25.4}$  & $176.0 \pm 4.2$             \\
\midrule
\multicolumn{3}{c}{\textbf{Cahn-Hilliard}} \\ \midrule
NRMSE $(\downarrow)$ & $\mathbf{123.9 \pm 2.6}$  & $170.6 \pm 1.0$             \\
SMAPE $(\downarrow)$ & $167.3 \pm 10.6$          & $\mathbf{155.4 \pm 1.2}$    \\
RE $(\downarrow)$    & $\mathbf{18.9 \pm 10.4}$  & $47.3 \pm 3.0$              \\
\bottomrule
\end{tabular}
\caption{Effect of MLP width on NRMSE, SMAPE, and RE averaged over all rollouts on all test meshes for the heat, wave, and Cahn-Hilliard equations. The best values are in bold.}
\label{tab:width}
\end{table*}

\clearpage 

\subsection{MeshPDE Sample Initializations} \label{sec:SampleInit}

\begin{table}[!htb]
    \centering
    \begin{tabular}{ccc}
  Mesh &  Heat $T=0$ & Wave $T=0$  \\ \toprule
  Armadillo &  \includegraphics[width=0.2\textwidth]{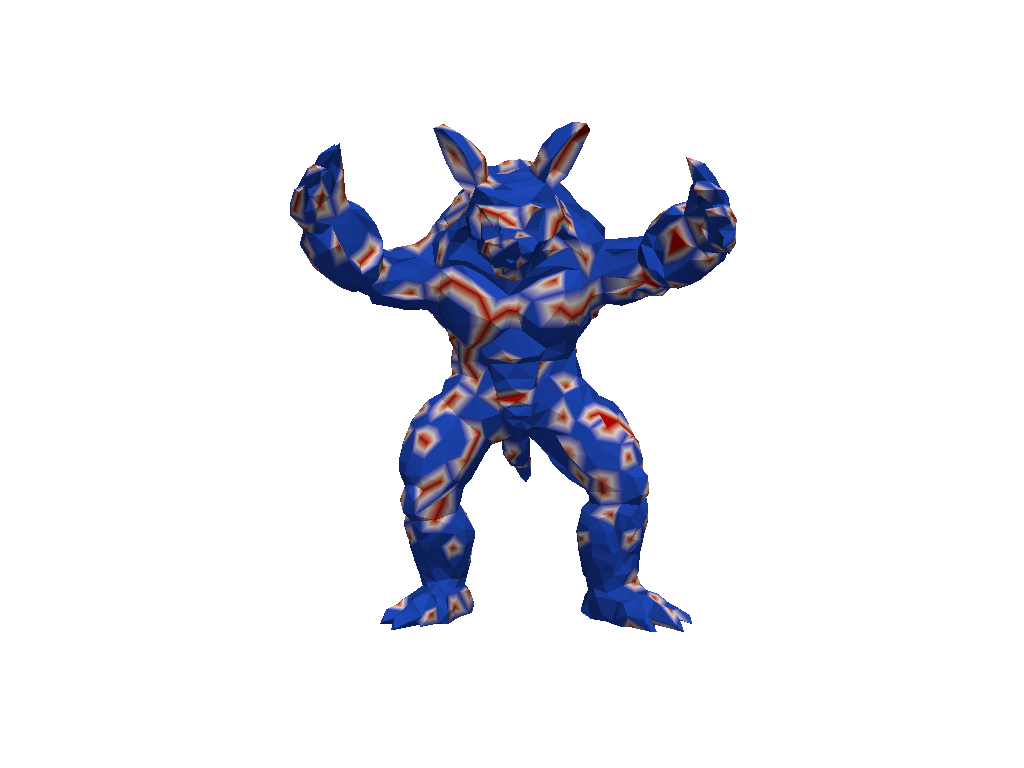} & \includegraphics[width=0.2\textwidth]{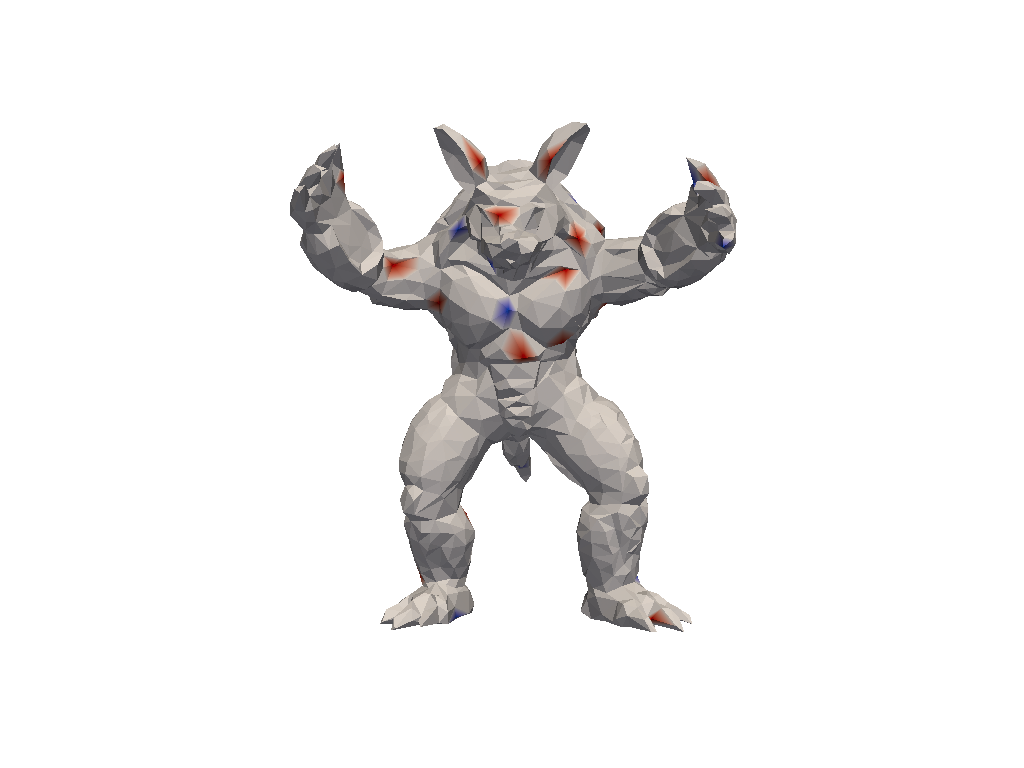}  \\ \midrule
  Bunny & \includegraphics[width=0.2\textwidth]{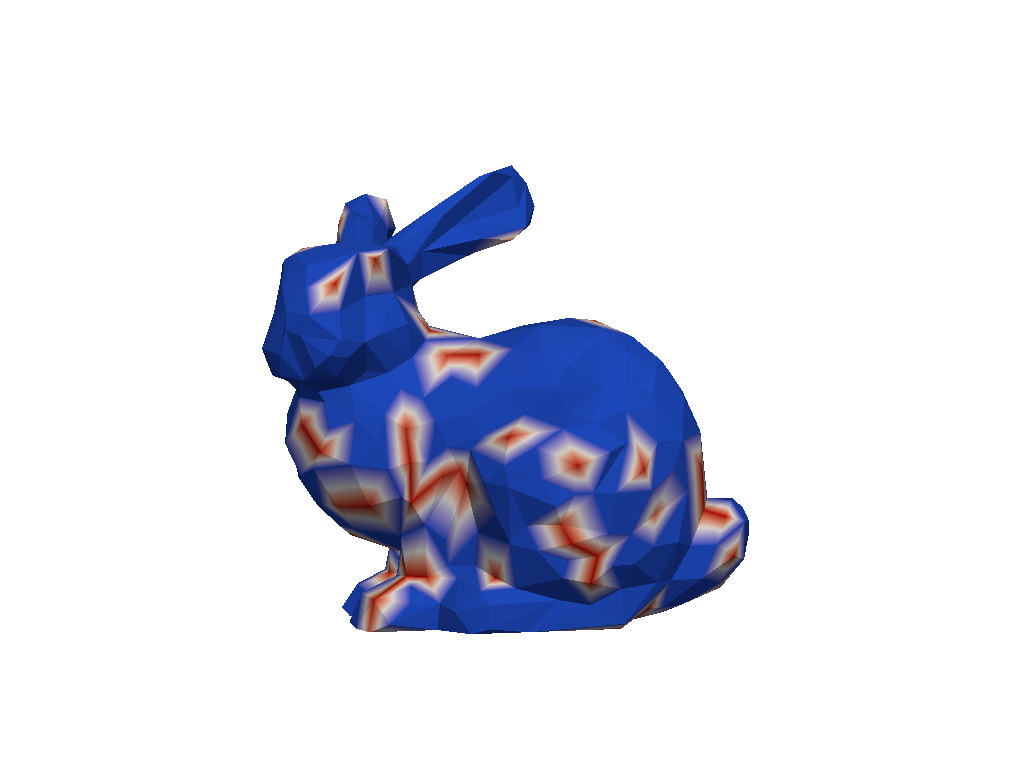} & \includegraphics[width=0.2\textwidth]{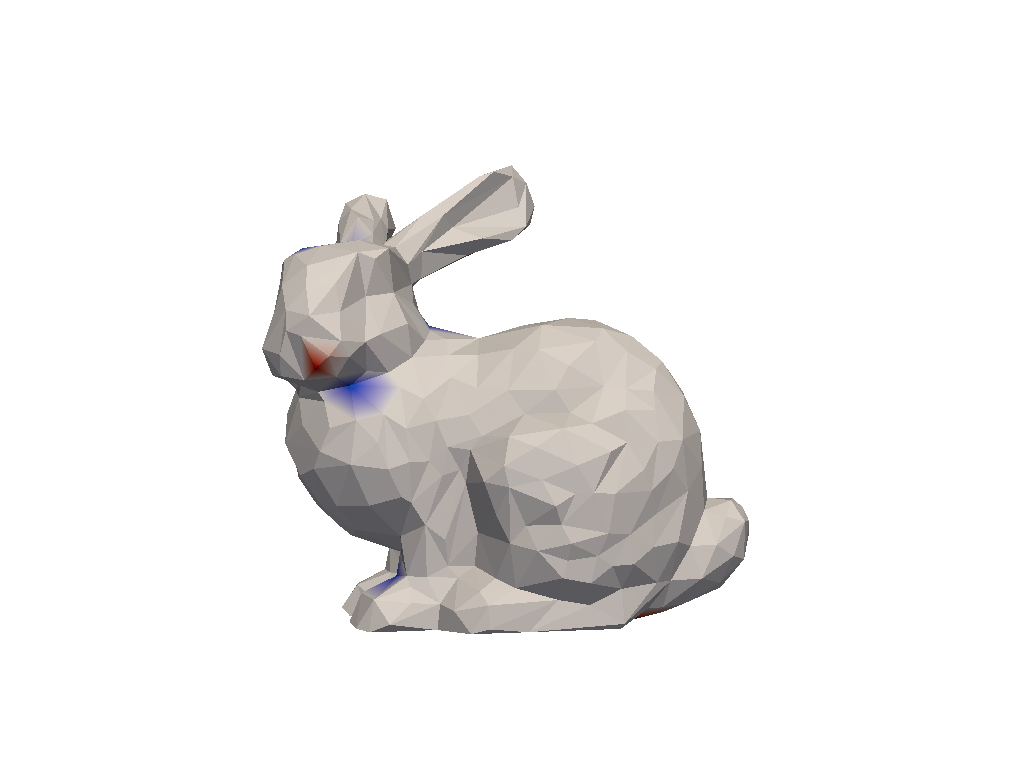}  \\ \midrule
  Lucy & \includegraphics[width=0.2\textwidth]{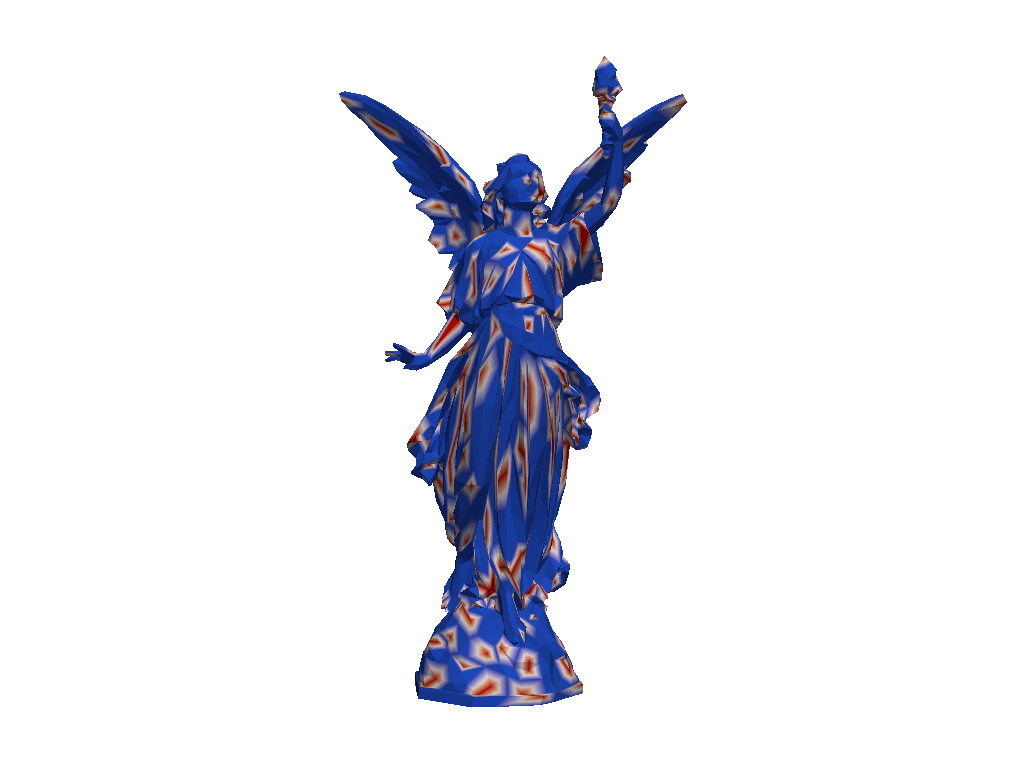} & \includegraphics[width=0.2\textwidth]{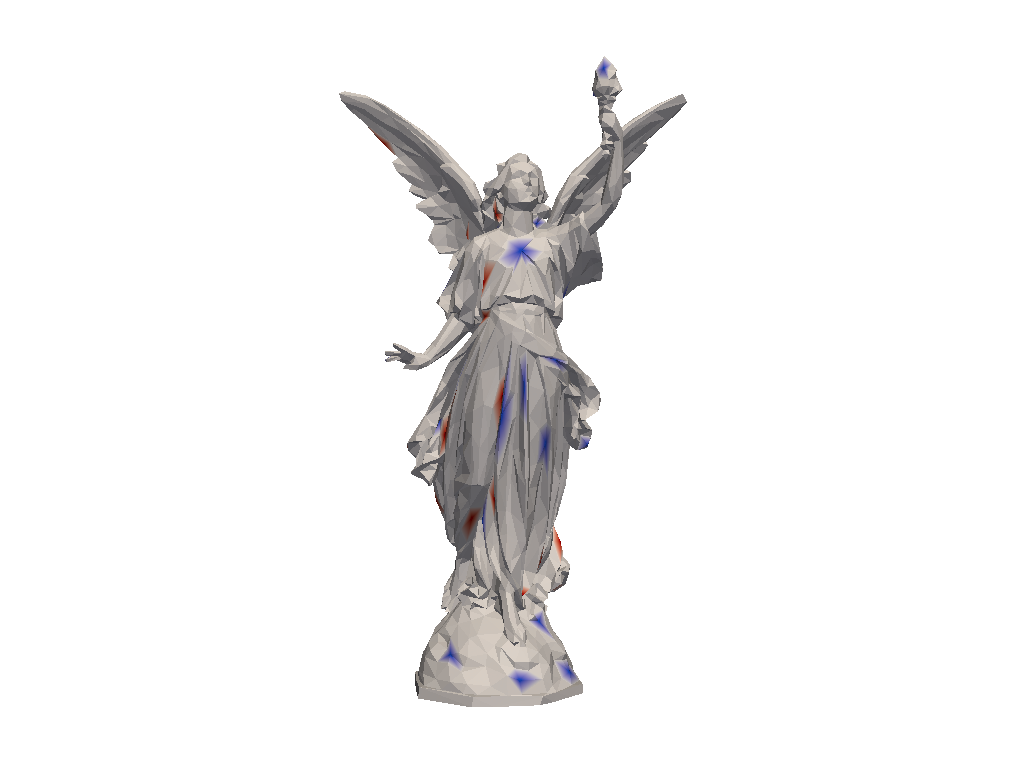} \\ \midrule
  Sphere & \includegraphics[width=0.2\textwidth]{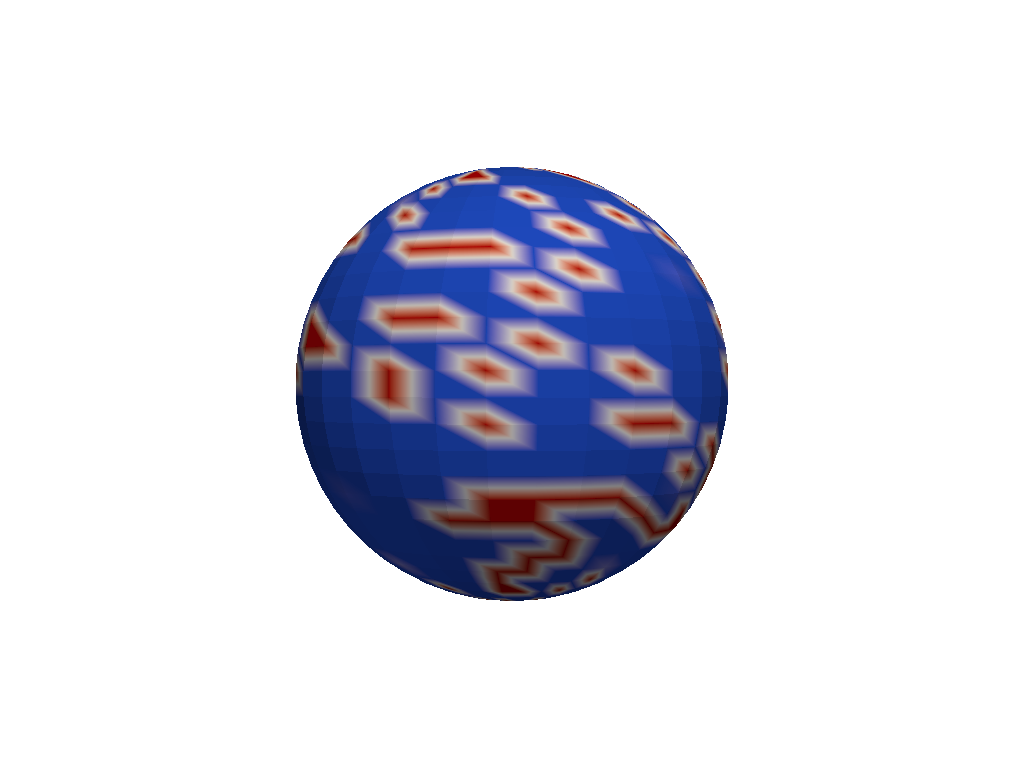} & \includegraphics[width=0.2\textwidth]{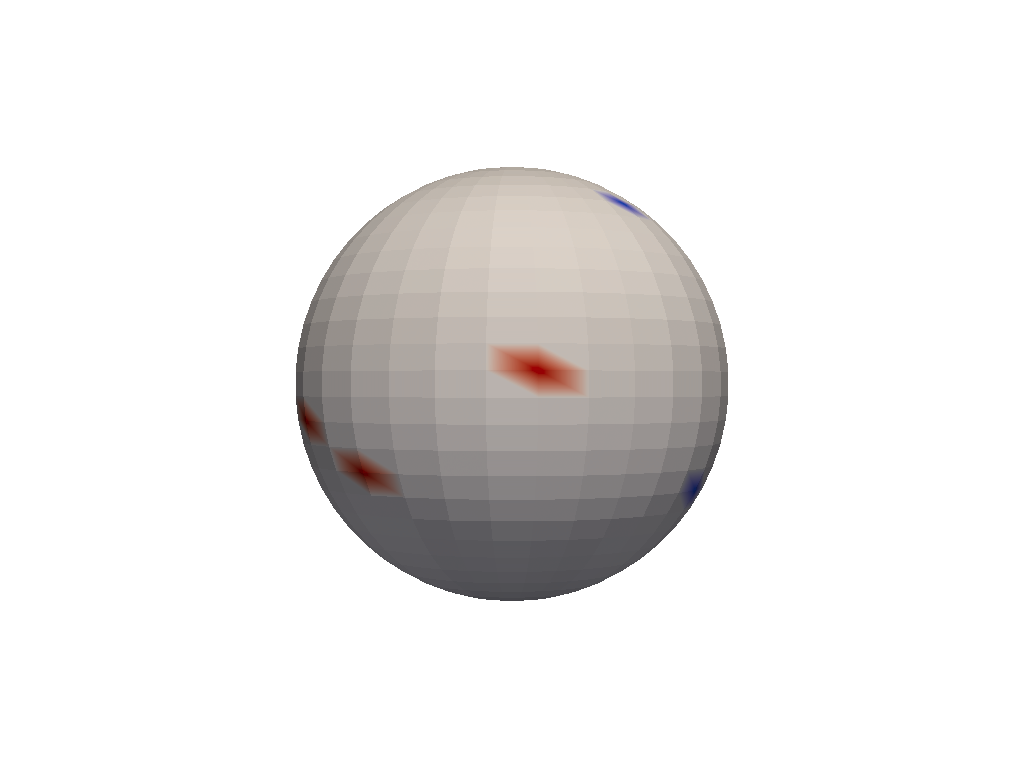} \\ \midrule
  Spider & \includegraphics[width=0.2\textwidth]{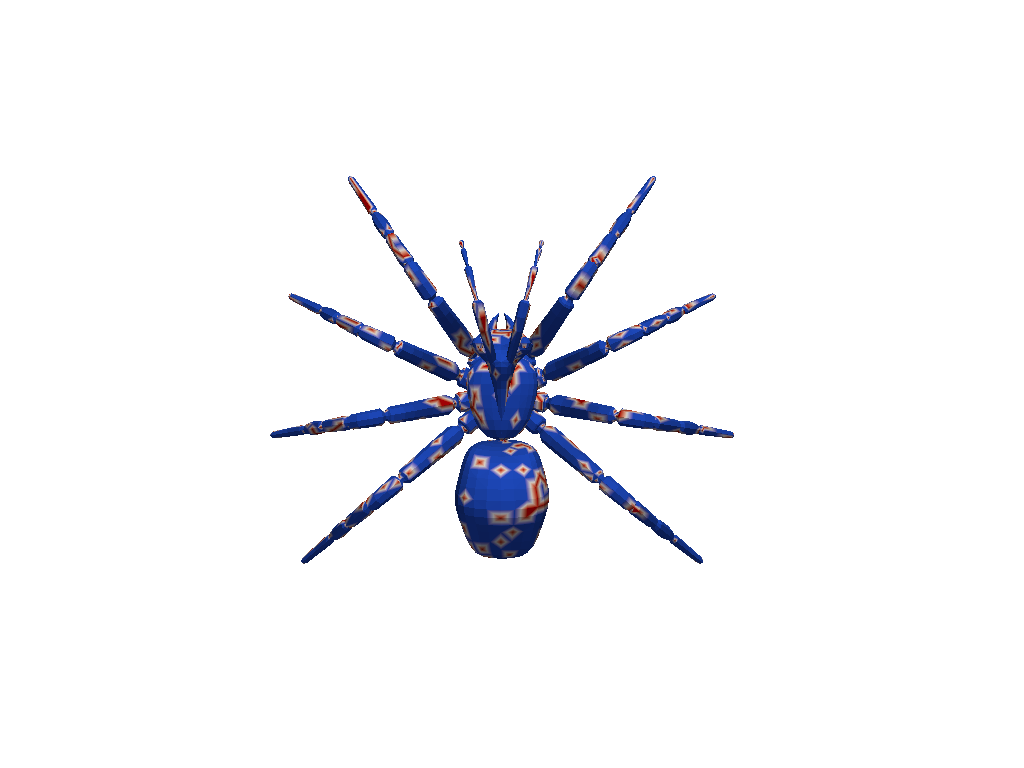} & \includegraphics[width=0.2\textwidth]{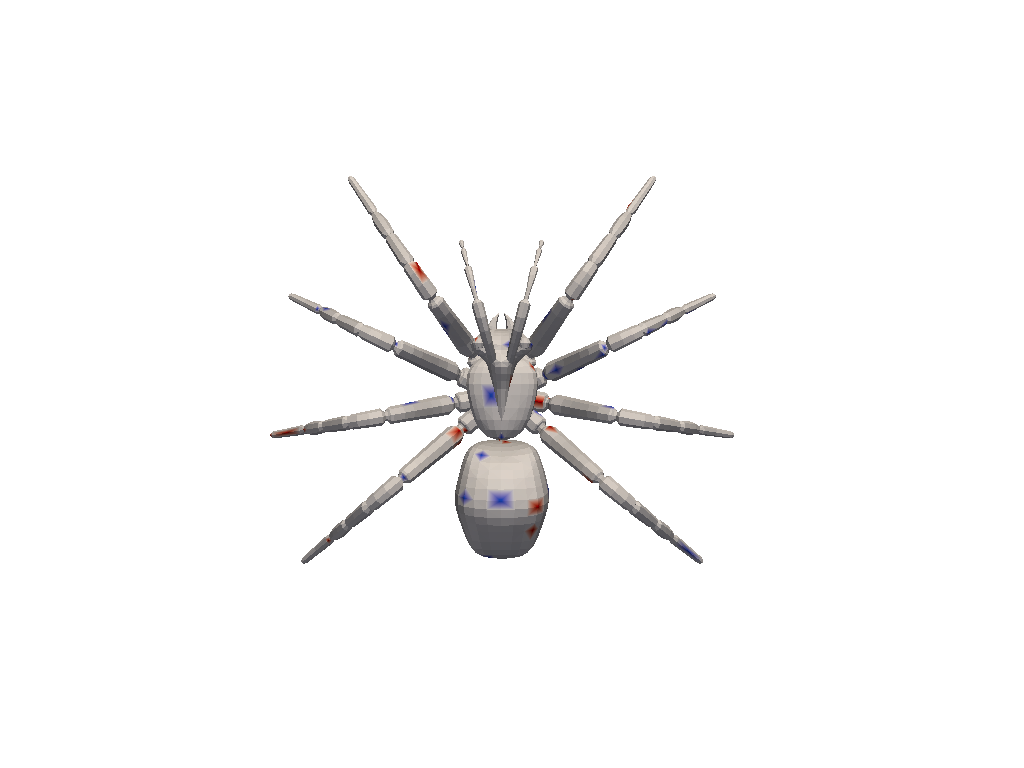} \\ \midrule
  Urn & \includegraphics[width=0.2\textwidth]{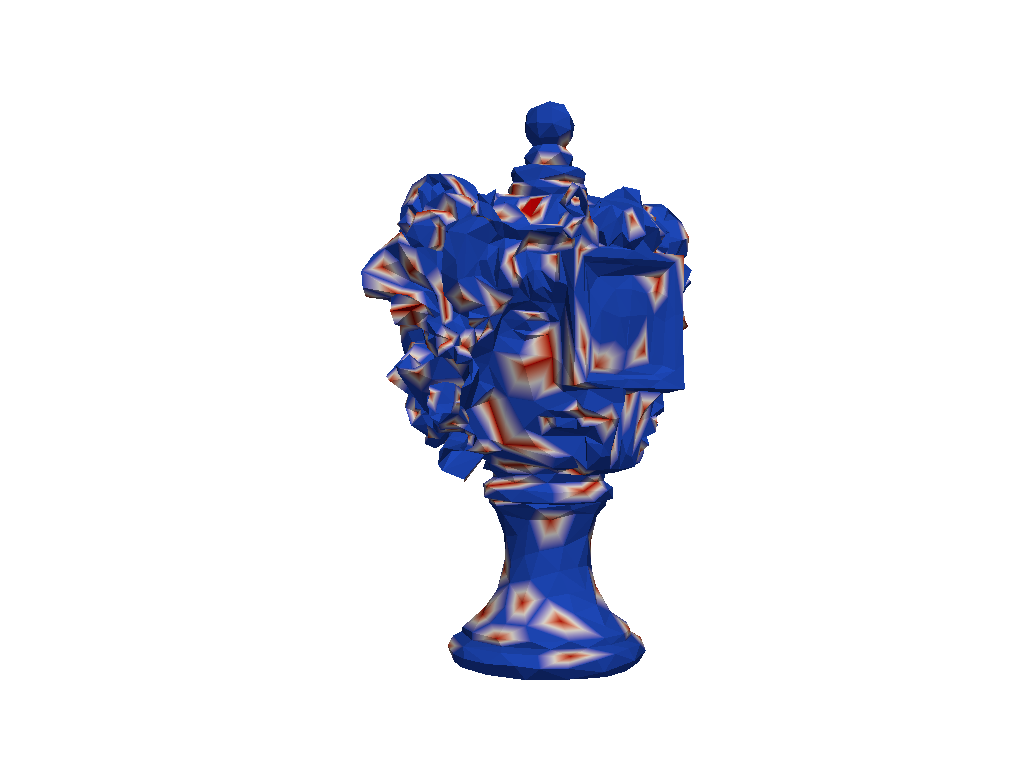} & \includegraphics[width=0.2\textwidth]{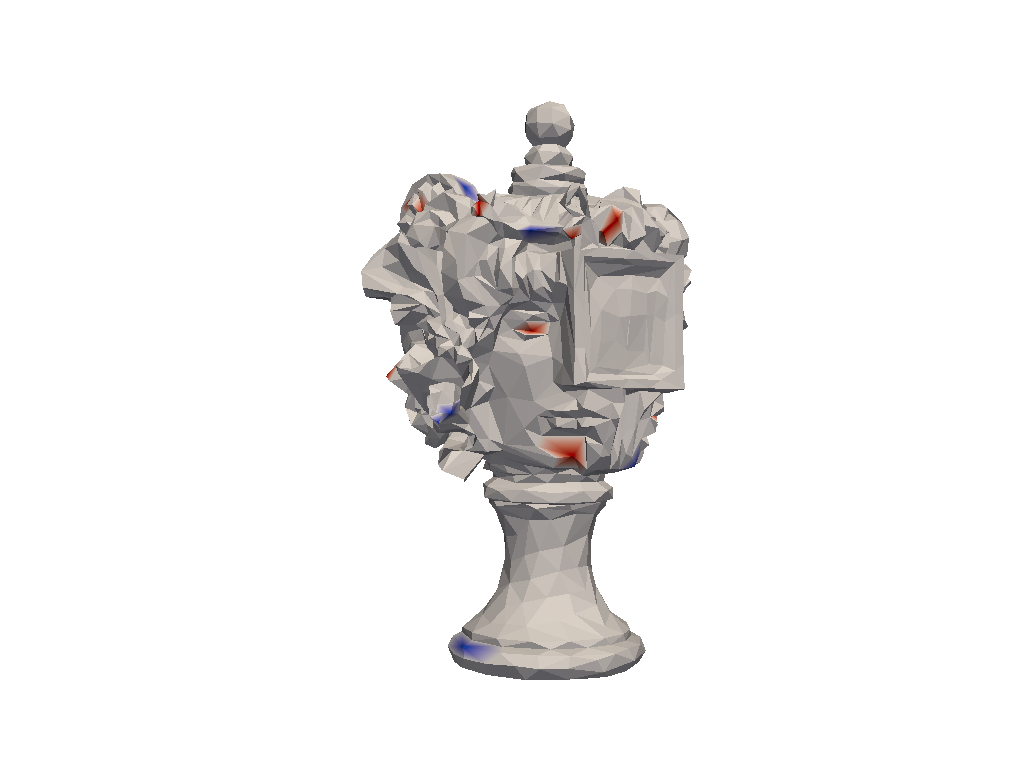} \\ \midrule
  Woman & \includegraphics[width=0.2\textwidth]{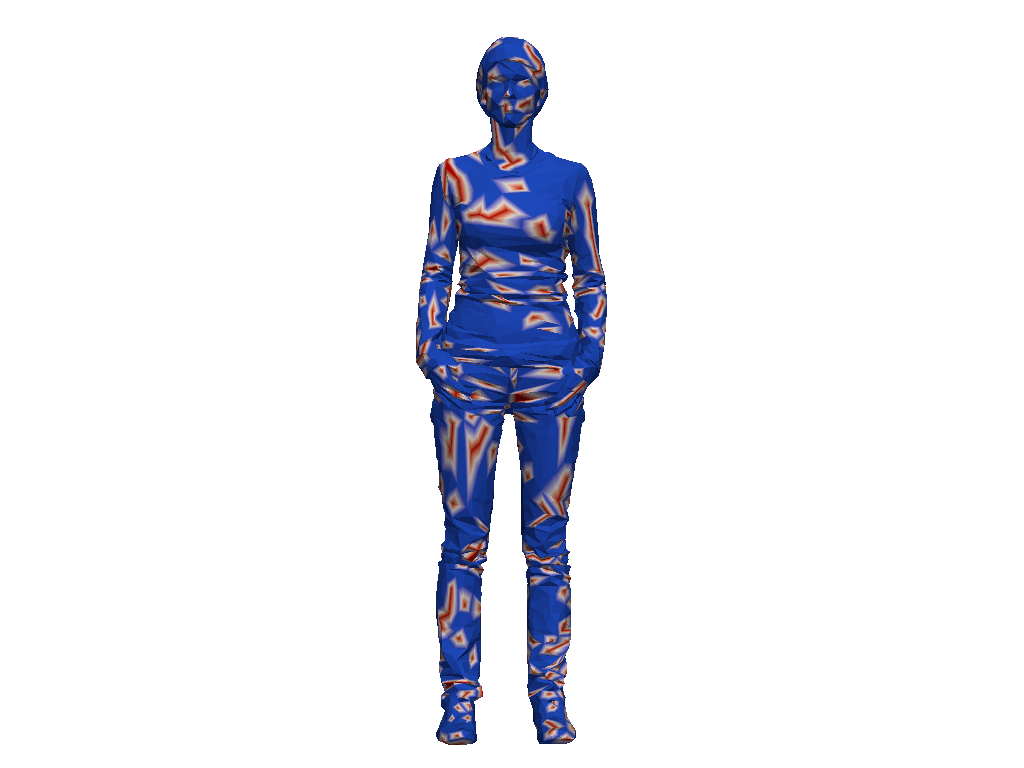} & \includegraphics[width=0.2\textwidth]{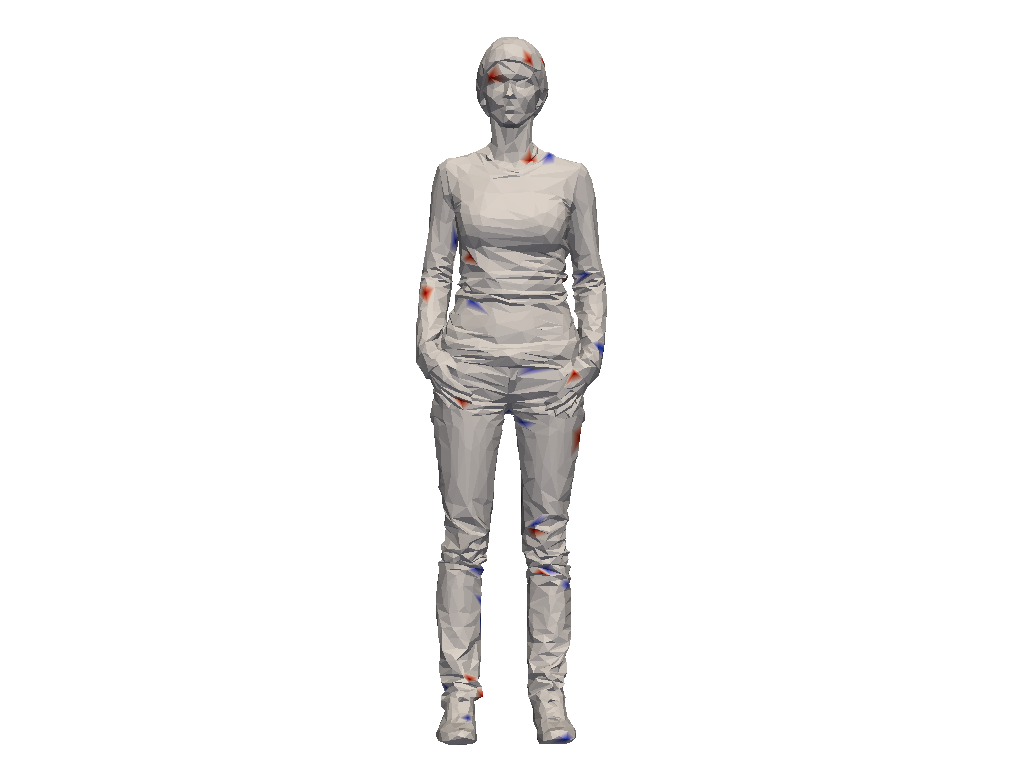} \\ \bottomrule
    \end{tabular}
    \caption{Sample initializations for the heat and wave equations on the \pv{} meshes.}
    \label{tab:sample}
\end{table}

\begin{table}[!htb]
    \centering
    \begin{tabular}{cc}
  Mesh &  Cahn-Hilliard $T=0$   \\ \toprule
  Bunny &  \includegraphics[width=0.2\textwidth]{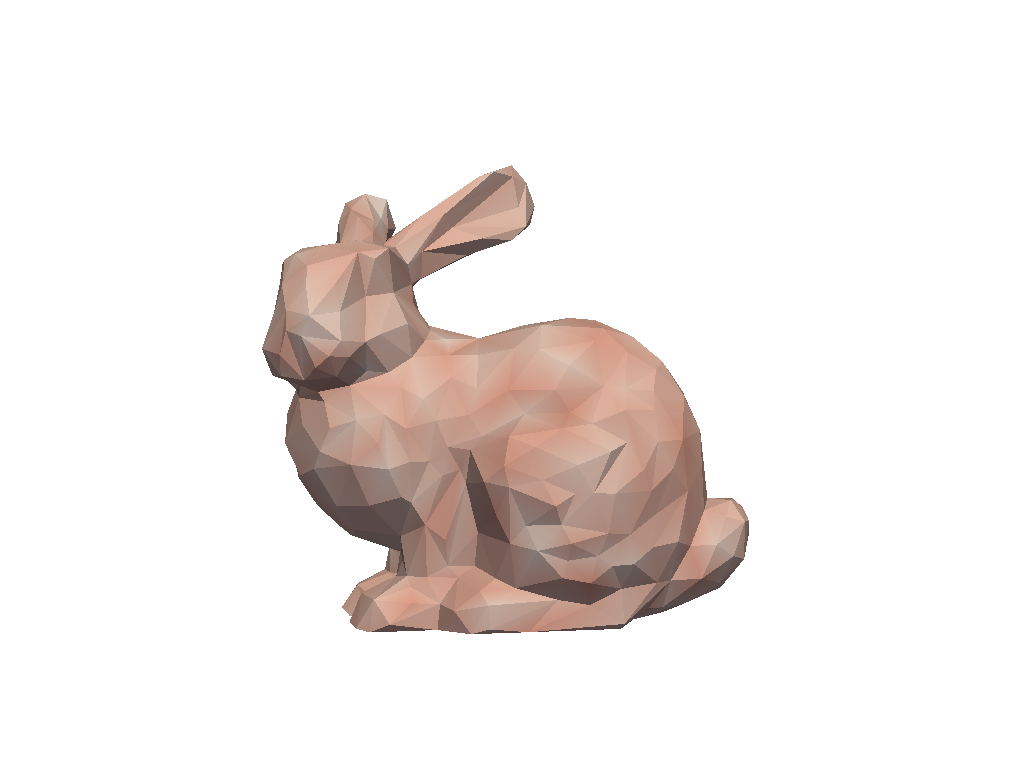}   \\ \midrule
  Ellipsoid & \includegraphics[width=0.2\textwidth]{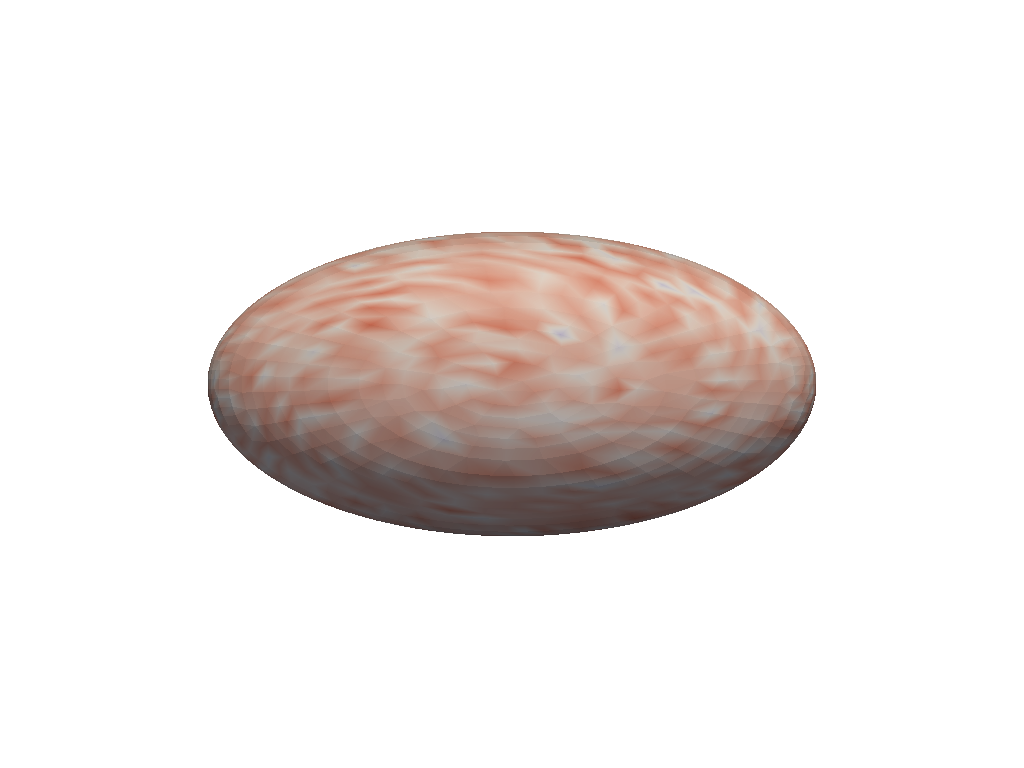}  \\ \midrule
  Sphere & \includegraphics[width=0.2\textwidth]{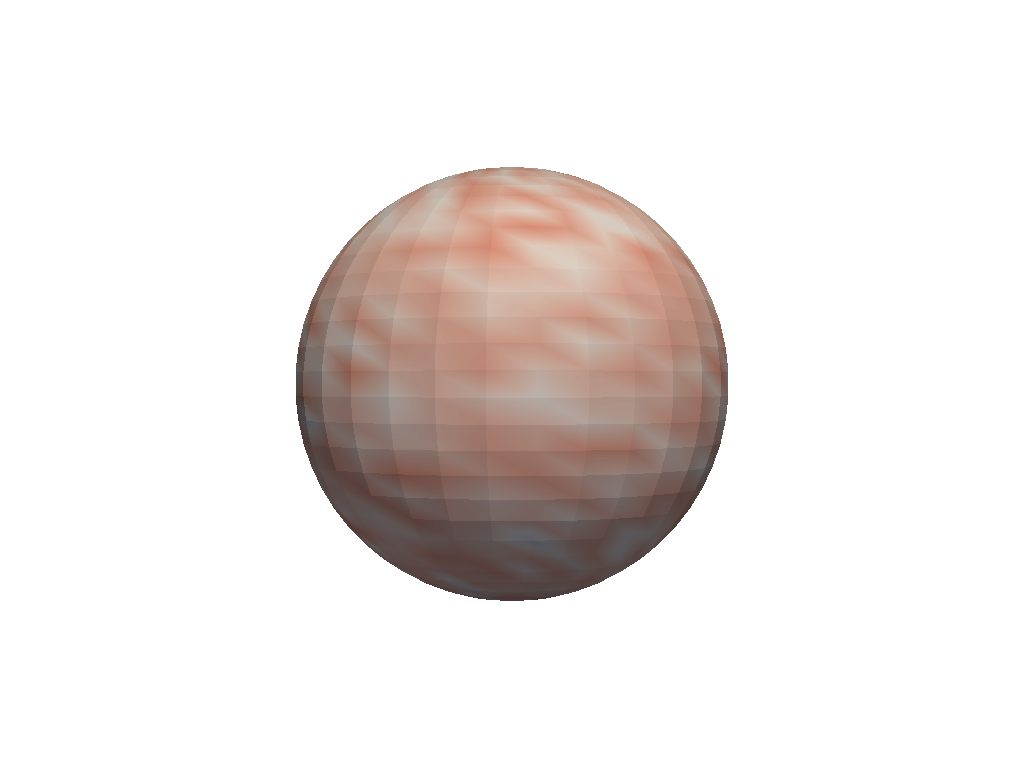}  \\ \midrule
 Super Toroid  & \includegraphics[width=0.2\textwidth]{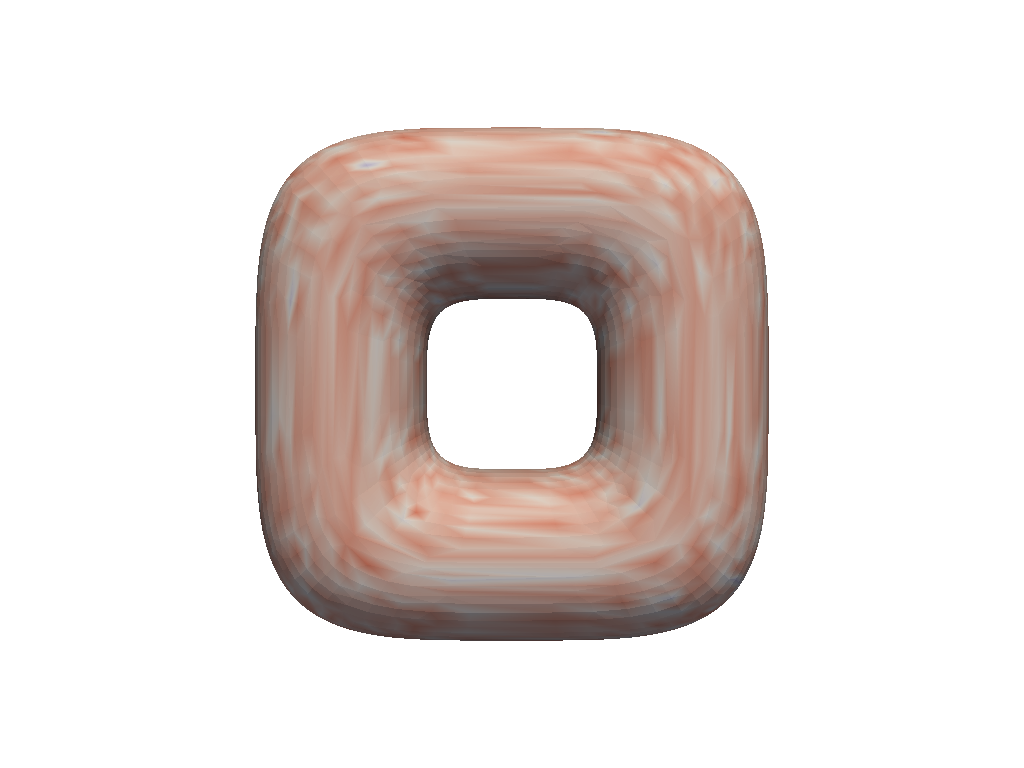} \\ \bottomrule
    \end{tabular}
    \caption{Sample initializations for the Cahn-Hilliard equation on the \pv{} meshes.}
    \label{tab:ch_sample}
\end{table}

\clearpage

\section{WeatherBench$\mathbf{2}$ Further Details and Results} \label{sec:wb_deets}

We lay out the relevant training, evaluation, and dataset details for WB$2$.

\subsection{WB2 Training Details and Problem Statement}  \label{sec:train_ps}

Our WB$2$ problem statement is as follows. Given a dataset $\mathcal{D} = \{X_i\}^N_{i = 1}$ of historical weather data, the task of weather forecasting is to predict future weather conditions $X_T \in \mathbb{R}^{V \times H \times W}$ given initial conditions $\{X_i\}^K_{i=1}, X_i \in \mathbb{R}^{V \times H \times W}$ , where $T$ is the target lead time, $K$ is the number of input time steps to the model, $V$ is the number of atmospheric variables, and $H \times W$ is the spatial resolution of the data, which depends on how densely we grid the globe. We follow the same training and hyper parameter optimization strategy as in \cref{sec:gegnn}. The only difference is that we use the $3$ previous time steps as input instead of $5$. All models are given a consistent compute budget of $8$ hours on an NVIDIA H$200$ GPU for up to $100$ epochs.

\subsection{WB2 Evaluation Details}

Here we give precise definitions of the evaluation and metrics omitted in the main text. We begin by establishing some notation common to the subsections, and consistent with the notation used in \cite{rasp2024weatherbench}.

Let $f$ denote the forecast, $o$ the ground-truth observation, and $c$ the climatology. Let $t \in \{1,\ldots,T\}$ denote the verification time, $l \in \{1,\ldots,L\}$ the lead time, $i \in \{1,\ldots,I\}$ the latitude index, and $j \in \{1,\ldots,J\}$ the longitude index. Forecasts are indexed as $f_{t,l,i,j}$, while observations and climatology are indexed by absolute time as $o_{t,i,j}$ and $c_{t,i,j}$. 

\subsubsection{Latitude weighting}

In an equiangular latitude-longitude grid, grid cells at the poles have a much smaller area compared to grid cells at the equator. Weighting all cells equally in the computation of RMSE and ACC would result in an inordinate bias towards the polar regions. As a result both metrics are latitude-weighted with weights computed as follows: 
\begin{equation*}
w(i) = \frac{\sin \theta^u_i - \sin \theta ^l_i}{\frac{1}{I} \sum^I_i (\sin \theta^u_t - \sin \theta ^l_i)},     
\end{equation*}
where $\theta^u_i$ and $\theta^l_i$ indicate upper and lower latitude bounds, respectively.

\subsubsection{Climatology}

The climatology $c$ is a function of the day of year and time of day, it is computed by taking the mean of ERA5 data from 1990 to 2019 (inclusive) for each grid point. A sliding window of 61 days is used around each day of year and time of day combination with weights linearly decaying to zero from the center. For notational consistency, we also define the lead-time–indexed climatology $c_{t,l,i,j} := c_{t+l,i,j},$ corresponding to the climatology at the forecast valid time.

\subsubsection{Root mean squared error (RMSE)}

Following the WB$2$ convention, our work measures error in terms of RMSE. For each variable and level pair, the RMSE at lead time $l$ is defined as: $$\RMSE_l = \sqrt{\frac{1}{TIJ} \sum^T_t\sum^I_i \sum^J_j w(i) (f_{t, l, i, j} - o_{t, i, j})^2}.$$ This choice is important for temperature forecasting, as we are invariant to choice of unit (e.g., temperature in terms of Kelvin and Celsius will have the same RMSE). Moreover, the change in scale over time is less dramatic as it was for the \pv{} meshes, where we considered NRMSE.

\subsubsection{Anomaly correlation coefficient (ACC)}

The ACC is computed as the Pearson correlation coefficient of the anomalies with respect to the climatology $c$. Denote the differences between forecast and climatology and between observation and climatology by $$f'_{t, l, i, j} = f_{t, l, i, j} - c_{t, l, i, j}; \quad o'_{t, i, j} = o_{t, i, j} - c_{t, i, j}.$$ The ACC at lead time $l$ is then defined as $$\ACC_l = \frac{1}{T} \sum^T_t \frac{\sum^I_i \sum^J_j w(i) f'_{t, l, i, j} \, o'_{t, i, j}}{\sqrt{\sum^I_i\sum^J_j w(i) {f'_{t, l, i, j}}^2 \sum^I_i \sum^J_j w(i) {o'_{t, i, j}}^2}}.$$ACC ranges from $1$, indicating perfect correlation, to $-1$, indicating perfect anti-correlation. The ECMWF states that when the ACC value falls below 0.6, it is considered that the positioning of synoptic scale features ceases to have value for forecasting purposes.

\subsection{WB2 Earth Mesh Discretization} \label{sec:earth_dis}

We construct a spherical mesh of the Earth by directly projecting the latitude–longitude grid points onto the unit sphere, and define mesh connectivity according to the original grid neighborhood structure. In order to obtain triangular faces, we further subdivide each cell into two triangles. The resulting mesh has 29040 nodes, 57600 faces, and 86640 edges. We note that this mesh construction is simpler than those used in other graph-based models, such as GraphCast \cite{lam2023graphcastlearningskillfulmediumrange}, which employs a subdivided icosahedron as the underlying mesh. However, our approach has the advantage that it operates directly on the native latitude–longitude grid and therefore does not require interpolation or regridding of the ERA5 data. Although more elaborate mesh constructions is likely to improve performance in real weather forecasting applications, our focus is on methodological experimentation rather than optimized weather prediction, and we therefore leave mesh optimization as an opportunity for future work.

\subsection{WB2 Temperature and Geopotential Extended Results} \label{sec:wb_geo_extended_results}

We report the RMSE and ACC curves for all lead times. We see in \cref{fig:wb2_g} that \ours{} stays valid for the longest lead time in terms of ACC, and is rivaled only by GemCNN in terms of RMSE for Geopotential. The results of \cref{fig:wb2} and \cref{fig:wb2_g} are also given as a table in \cref{tab:weather_results}.

\begin{figure}[!htb]
    \centering
    \includegraphics[width=0.75\linewidth]{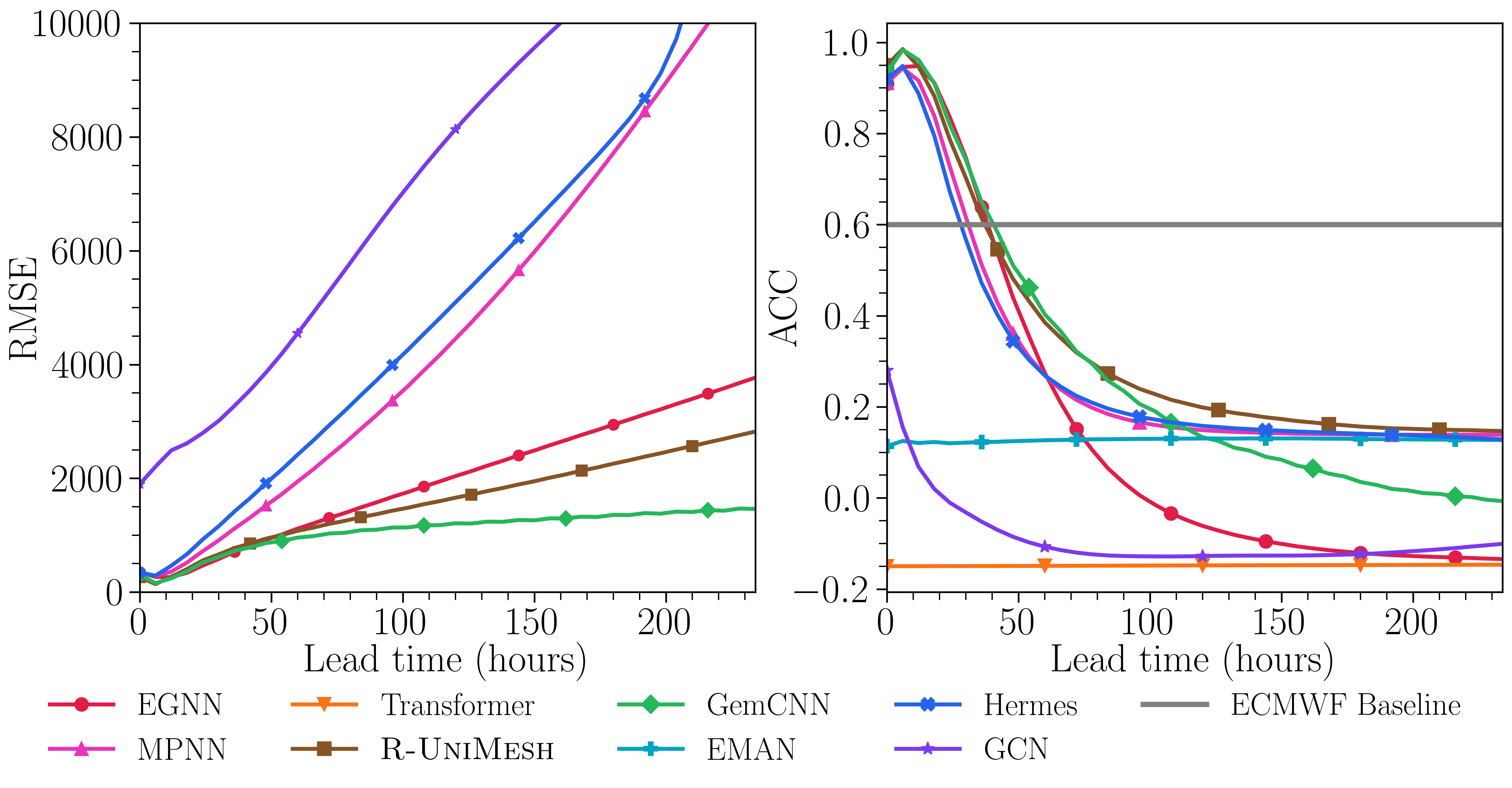}
    \caption{RMSE and ACC as a function of lead time for all models geopotential prediction. \ours{} has a competitive RMSE, especially at early lead time. \ours{} also maintains viability for lead times of roughly $2$ days according to the ECMWF baseline.}  
    \label{fig:wb2_g}
\end{figure}

\begin{table}[htbp]
\centering
\tiny
\resizebox{\textwidth}{!}{%
\begin{tabular}{llrrrrrrrrr}
\toprule
Model & Metric & 6h & 12h & 18h & 24h & 30h & 36h & 42h & 48h & 240h \\
\midrule
\multirow{4}{*}{\hyperlink{cite.satorras2022enequivariantgraphneural}{EGNN}} & Z500 ACC & 0.909 & 0.946 & 0.948 & 0.912 & 0.835 & 0.747 & 0.639 & 0.538 & -0.134 \\
 & Z500 RMSE & 347.83 & 266.96 & 261.09 & 340.43 & 467.68 & 581.20 & 706.01 & 812.70 & 3769.31 \\
 & T850 ACC & 0.840 & 0.923 & 0.894 & 0.817 & 0.725 & 0.637 & 0.541 & 0.456 & 0.093 \\
 & T850 RMSE & 1.94 & 1.32 & 1.54 & 2.02 & 2.46 & 2.81 & 3.17 & 3.47 & 6.71 \\
\midrule
\multirow{4}{*}{ \hyperlink{cite.basu2022equivariant}{EMAN}} & Z500 ACC & 0.114 & 0.125 & 0.121 & 0.123 & 0.120 & 0.122 & 0.123 & 0.123 & 0.127 \\
 & Z500 RMSE & 39530.01 & 52028.53 & 65037.68 & 89210.70 & 103179.47 & 119529.81 & 139314.85 & 156786.23 & 4898180.46 \\
 & T850 ACC & -0.016 & -0.011 & -0.025 & -0.029 & -0.029 & -0.028 & -0.028 & -0.028 & -0.007 \\
 & T850 RMSE & 19.87 & 26.34 & 29.95 & 33.25 & 35.42 & 37.73 & 39.95 & 43.95 & 1097293305.30 \\
\midrule
\multirow{4}{*}{ \hyperlink{cite.kipf2017semisupervised}{GCN}} & Z500 ACC & 0.280 & 0.155 & 0.068 & 0.020 & -0.011 & -0.032 & -0.052 & -0.070 & -0.101 \\
 & Z500 RMSE & 1897.27 & 2205.72 & 2491.37 & 2616.58 & 2798.23 & 3009.40 & 3272.80 & 3552.62 & 13147.85 \\
 & T850 ACC & 0.287 & 0.170 & 0.087 & 0.030 & -0.006 & -0.025 & -0.033 & -0.031 & 0.093 \\
 & T850 RMSE & 6.15 & 7.02 & 8.29 & 10.90 & 21.75 & 61.01 & 173.60 & 466.71 & 930146229837458.25 \\
\midrule
\multirow{4}{*}{\hyperlink{cite.dehaan2021gaugeequivariantmeshcnns}{GemCNN}} & Z500 ACC & 0.934 & 0.984 & 0.962 & 0.910 & 0.818 & 0.742 & 0.650 & 0.584 & -0.007 \\
 & Z500 RMSE & 304.11 & 156.00 & 242.86 & 364.36 & 521.75 & 615.11 & 724.60 & 785.92 & 1461.45 \\
 & T850 ACC & 0.879 & 0.995 & 0.901 & 0.776 & 0.670 & 0.581 & 0.477 & 0.392 & -0.025 \\
 & T850 RMSE & 1.71 & 0.34 & 1.51 & 2.30 & 2.77 & 3.10 & 3.49 & 3.80 & 6.05 \\
\midrule
\multirow{4}{*}{\hyperlink{cite.park2023modelingdynamicsmeshesgauge}{Hermes}} & Z500 ACC & 0.919 & 0.948 & 0.888 & 0.795 & 0.670 & 0.568 & 0.472 & 0.402 & 0.129 \\
 & Z500 RMSE & 339.23 & 289.93 & 463.62 & 667.90 & 927.73 & 1157.17 & 1416.43 & 1653.81 & 25644.40 \\
 & T850 ACC & 0.842 & 0.906 & 0.756 & 0.586 & 0.447 & 0.340 & 0.252 & 0.187 & -0.025 \\
 & T850 RMSE & 1.96 & 1.50 & 2.55 & 3.59 & 4.45 & 5.20 & 5.99 & 6.76 & 293.82 \\
\midrule
\multirow{4}{*}{\hyperlink{cite.g2017}{MPNN}} & Z500 ACC & 0.911 & 0.943 & 0.916 & 0.839 & 0.726 & 0.617 & 0.512 & 0.429 & 0.139 \\
 & Z500 RMSE & 349.58 & 283.77 & 360.45 & 519.00 & 721.06 & 908.69 & 1116.36 & 1310.80 & 11191.46 \\
 & T850 ACC & 0.846 & 0.933 & 0.898 & 0.817 & 0.724 & 0.638 & 0.547 & 0.469 & 0.043 \\
 & T850 RMSE & 1.90 & 1.23 & 1.52 & 2.03 & 2.48 & 2.82 & 3.17 & 3.44 & 5.99 \\
\midrule
\multirow{4}{*}{\hyperlink{cite.janny2023eagle}{Transformer}} & Z500 ACC & -0.150 & -0.150 & -0.150 & -0.150 & -0.150 & -0.150 & -0.149 & -0.149 & -0.147 \\
 & Z500 RMSE & 55368.55 & 55368.50 & 55368.45 & 55520.22 & 55520.17 & 55520.12 & 55558.25 & 55558.18 & 55570.73 \\
 & T850 ACC & 0.884 & 0.884 & 0.884 & 0.553 & 0.553 & 0.553 & 0.304 & 0.304 & -0.068 \\
 & T850 RMSE & 1.62 & 1.62 & 1.62 & 3.13 & 3.13 & 3.13 & 3.96 & 3.96 & 7.37 \\
\midrule
\multirow{4}{*}{\ours{} \textbf{(Ours)}} & Z500 ACC & 0.950 & 0.985 & 0.947 & 0.882 & 0.786 & 0.703 & 0.614 & 0.546 & 0.147 \\
 & Z500 RMSE & 260.31 & 140.60 & 271.28 & 405.01 & 556.83 & 662.19 & 773.67 & 853.87 & 2821.03 \\
 & T850 ACC & 0.888 & 0.964 & 0.889 & 0.774 & 0.679 & 0.598 & 0.504 & 0.428 & 0.134 \\
 & T850 RMSE & 1.63 & 0.91 & 1.63 & 2.37 & 2.85 & 3.20 & 3.65 & 4.04 & 11.68 \\
\bottomrule
\end{tabular}%
}
\caption{Weather Forecasting Results, ACC and RMSE at Different Lead Times for Temperature and Geopotential.}
\label{tab:weather_results}
\end{table}

\subsection{WB2 Smoothness Extended Results} \label{sec:wb_smooth_extended_results}

We provide smoothness errors for all models on WB$2$ temperature and geopotential datasets. As seen in \cref{tab:wb_smooth}, \ours{} is competitive across both temperature and geopotential, and is within statistical significance of the best performing EGNN model on geopotential.

\begin{table}[!htb]
    \centering
    \tiny 
    \begin{tabular}{ccc}
      Model &  RE Temperature $(\downarrow)$ & RE Geopotential $(\downarrow)$ \\ \toprule
      \hyperlink{cite.kipf2017semisupervised}{GCN} & $6.8 \cdot 10^{-2} \pm 1.7 \cdot 10^{-2}$   & $1.2 \cdot 10^{-3} \pm  1.0 \cdot 10^{-3}$ \\ 
      \hyperlink{cite.g2017}{MPNN} & $3.0 \cdot 10^{-3} \pm 2.7 \cdot 10^{-3}$ & $1.3 \cdot 10^{-3} \pm 8.5 \cdot 10^{-4}$\\
      \hyperlink{cite.park2023modelingdynamicsmeshesgauge}{Hermes} & $3.1 \cdot 10^{-1} \pm 7.5 \cdot 10^{-1}$ & $1.0 \cdot 10^{-1} \pm 3.6 \cdot 10^{-1}$ \\ 
      \hyperlink{cite.dehaan2021gaugeequivariantmeshcnns}{GemCNN}  & $\mathbf{8.9 \cdot 10^{-4} \pm 7.0 \cdot 10^{-4}}$ & $9.4 \cdot 10^{-4} \pm 6.1 \cdot 10^{-4}$ \\
      \hyperlink{cite.basu2022equivariant}{EMAN} & $7.7 \cdot 10^0 \pm 5.2 \cdot 10^0$ & $4.1 \pm 2.6 \cdot 10^{-2}$ \\ 
      \hyperlink{cite.satorras2022enequivariantgraphneural}{EGNN} & $1.1 \cdot 10^{-3} \pm 7.3 \cdot 10^{-4}$ & $\mathbf{8.3 \cdot 10^{-4} \pm 1.0 \cdot 10^{-3}}$ \\ 
      \hyperlink{cite.janny2023eagle}{Transformer}  & $9.4 \cdot 10^{-4} \pm 8.2 \cdot 10^{-4}$ & $2.0 \cdot 10^{-3} \pm 1.1 \cdot 10^{-3}$ \\
            \ours{} (Ours) & $2.2 \cdot 10^{-3} \pm 1.6 \cdot 10^{-3}$ & $9.8 \cdot 10^{-4} \pm 7.2 \cdot 10^{-4}$ \\ 
\bottomrule
    \end{tabular}
    \caption{Rayleigh error for all models for all initializations on WB$2$ temperature and geopotential. Best performing model is indicated with \textbf{bold text}. Errors scaled up by $\times 40$.}
    \label{tab:wb_smooth}
\end{table}

\end{document}